\documentclass[11pt]{article}

\usepackage{microtype,fullpage}
\usepackage{natbib}
\usepackage{graphicx}
\usepackage{subfigure}
\usepackage[table]{xcolor}
\usepackage{multirow}
\usepackage{pifont}
\usepackage[font=small]{caption}
\usepackage{placeins}
\usepackage{floatpag}
\usepackage{etoolbox}
\usepackage{changepage}
\RequirePackage{algorithm}
\RequirePackage{algorithmic}
\usepackage{booktabs} 
\usepackage{amsmath,amssymb,amsthm,amsfonts,enumitem,color}
\usepackage{hyperref}

\theoremstyle{plain}
\newtheorem{theorem}{Theorem}[section]
\newtheorem{lemma}[theorem]{Lemma}

\newtheorem{assumption}{Assumption}

\title{Instabilities of Offline RL with Pre-Trained Neural Representation}

\newcommand{\expect}{\mathbb{E}}

\newcommand{\states}{\mathcal{S}}
\newcommand{\trans}{P}
\newcommand{\actions}{\mathcal{A}}
\newcommand{\mdp}{M}

\date{}
\author{Ruosong Wang \\ Carnegie Mellon University \\ \texttt{ruosongw@andrew.cmu.edu}
\and Yifan Wu \\ Carnegie Mellon University \\
\texttt{yw4@andrew.cmu.edu}
\vspace*{0.1cm}
\and Ruslan Salakhutdinov   \\ Carnegie Mellon University \\ \texttt{rsalakhu@cs.cmu.edu}
\and Sham M. Kakade \\ University of Washington, Seattle\\ \& Microsoft Research \\ \texttt{sham@cs.washington.edu}}
\begin{document}

\maketitle
\begin{abstract}

In offline reinforcement learning (RL), we seek to utilize offline data to evaluate (or learn) policies in scenarios where the data are collected from a distribution that substantially differs from that of the target policy to be evaluated. Recent theoretical advances have shown that such sample-efficient offline RL is indeed possible provided certain strong representational conditions hold, else there are lower bounds exhibiting exponential error amplification (in the problem horizon) unless the data collection distribution has only a mild distribution shift relative to the target policy. This work studies these issues from an empirical perspective to gauge how stable offline RL methods are. In particular, our methodology explores these ideas when using features from pre-trained neural networks, in the hope that these representations are powerful enough to permit sample efficient offline RL. Through extensive experiments on a range of tasks, we see that substantial error amplification does occur even when using such pre-trained representations (trained on the same task itself); we find offline RL is stable only under extremely mild distribution shift. The implications of these results, both from a theoretical and an empirical perspective, are that successful offline RL (where we seek to go beyond the low distribution shift regime)  requires 
substantially stronger conditions beyond those which suffice for successful supervised learning.

\end{abstract}
\section{Introduction}

Offline reinforcement learning (RL) seeks to utilize offline data to
alleviate the sample complexity burden in challenging sequential
decision making settings where sample-efficiency is
crucial~\citep{mandel2014offline,gottesman2018evaluating,Wang2018SupervisedRL,8904645};
it is seeing much recent interest due to the large amounts of offline
data already available in numerous scientific and engineering domains.
The goal is to efficiently evaluate (or learn) policies, in scenarios
where the data are collected from a distribution that (potentially)
substantially differs from that of the target policy to be evaluated.
Broadly, an important question here is to better understand the practical challenges we
face in offline RL problems and how to address them.

Let us start by considering when we expect offline RL to be successful
from a theoretical perspective~\citep{munos2003error,
szepesvari2005finite, antos2008learning, munos2008finite,tosatto2017boosted, chen2019information, duan2020minimax}.
For the purpose of evaluating a given target policy,
\citet{duan2020minimax} showed that
under a (somewhat stringent) policy completeness assumption with
regards to a linear feature mapping\footnote{A linear feature
  mapping is said to be complete if Bellman backup of a linear
  function remains in the span of the given features. See Assumption~\ref{assumption:completeness} for a formal definition. } along with data \emph{coverage}
assumption, then Fitted-Q iteration (FQI)~\citep{gordon1999approximate}
--- a classical offline Bellman backup based method --- can provably
evaluate a policy with low sample complexity (in the dimension of the
feature mapping).
While the coverage assumptions here are mild,
the representational conditions for such settings to be successful are more concerning; they go well
beyond simple \emph{realizability} assumptions, which only requires the
representation to be able to approximate the state-value function of the given target policy.

Recent theoretical
advances~\citep{wang2020statistical} show that without such a strong
representation condition,
there are lower bounds exhibiting exponential error amplification (in
the problem horizon) unless the data collection distribution has only a mild distribution shift
relative to the target policy.\footnote{We discuss these issues in more depth in
Section~\ref{sec:fqi}, where we give a characterization of FQI in the discounted setting}
It is worthwhile to emphasize that this ``low distribution condition''
is a problematic restriction, since in offline RL, we seek to utilize 
diverse data collection distributions. 
As an intuitive example to contrast the issue of distribution shift
vs. coverage, consider offline RL for spatial navigation
tasks (e.g. \cite{youtube}): coverage in our offline dataset would
seek that our dataset has example transitions from a diverse set of
spatial locations, while a low distribution shift condition would seek that our
dataset closely resembles that of the target policy itself for which
we desire to evaluate.

From a practical point of view, it is natural to ask to what
extent these worst-case characterizations are reflective of the scenarios
that arise in practical applications because, in fact, modern deep learning
methods often produce representations that are extremely effective, say for
transfer learning (computer vision~\cite{yosinski2014transferable}
and NLP~\cite{peters2018deep,devlin2018bert,radford2018improving} have both witnessed remarkable successes
using pre-trained features on downstreams tasks of interest). Furthermore, there are number of
offline RL methods with promising performance on certain benchmark tasks~\citep{Laroche2019,fujimoto2019off,jaques2020way,kumar2019stabilizing,agarwal2020striving,wu2020behavior,kidambi2020morel,
  DBLP:conf/icml/RossB12}. There are (at least) two
reasons which support further empirical investigations over these
current works: (i) the extent to which these data collection
distributions are diverse has not been carefully
controlled\footnote{The data collection in many benchmarks tasks are
  often taken from the data obtain when training an \emph{online}
  policy, say with deep Q-learning or policy gradient methods.} and (ii) the hyperparameter
tuning in these approaches are done in an interactive manner tuned on
how the policy actually behaves in the world as opposed to being tuned
on the offline data itself (thus limiting the scope of these methods).

In this work we provide a careful empirical investigation to further
understand how sensitive offline RL methods are to distribution shift.
Along this line of inquiry, 
One specific question 
to answer is to what extent we should be concerned about the error amplification effects as
suggested by worst-case theoretical considerations.

\paragraph{Our Contributions.} We study these questions on a range of standard tasks ($6$ tasks from the OpenAI gym benchmark suite),
using offline datasets with features from pre-trained neural networks trained on the task itself. Our
offline datasets are a mixture of trajectories from the target policy itself,
along the data from other policies (random or lower performance
policies). Note that this is favorable setting in that we would not expect
realistic offline datasets to have a large number of
trajectories from the target policy itself.

The motivation for using pre-trained features are both conceptual and technical. First, we may
hope that such features are powerful enough to permit sample-efficient offline
RL because they were learned in an online manner on the task itself.
Also, practically, while we are not able to verify if certain theoretical
assumptions hold, we may optimistically hope that such pre-trained
features will perform well under distribution shift (indeed, as discussed
earlier, using pre-trained features has had remarkable successes in
other domains). Second, using pre-trained
features allows us to decouple practical representational learning questions
from the offline RL question, where we can focus on offline RL
with a given representation. 
We provide further discussion on our methodologies in Section~\ref{sec:further}. 
We also utilize random Fourier features~\citep{rahimi2007random}
as a point of comparison.

\begin{figure}
\centering
\includegraphics[width=0.5\textwidth]{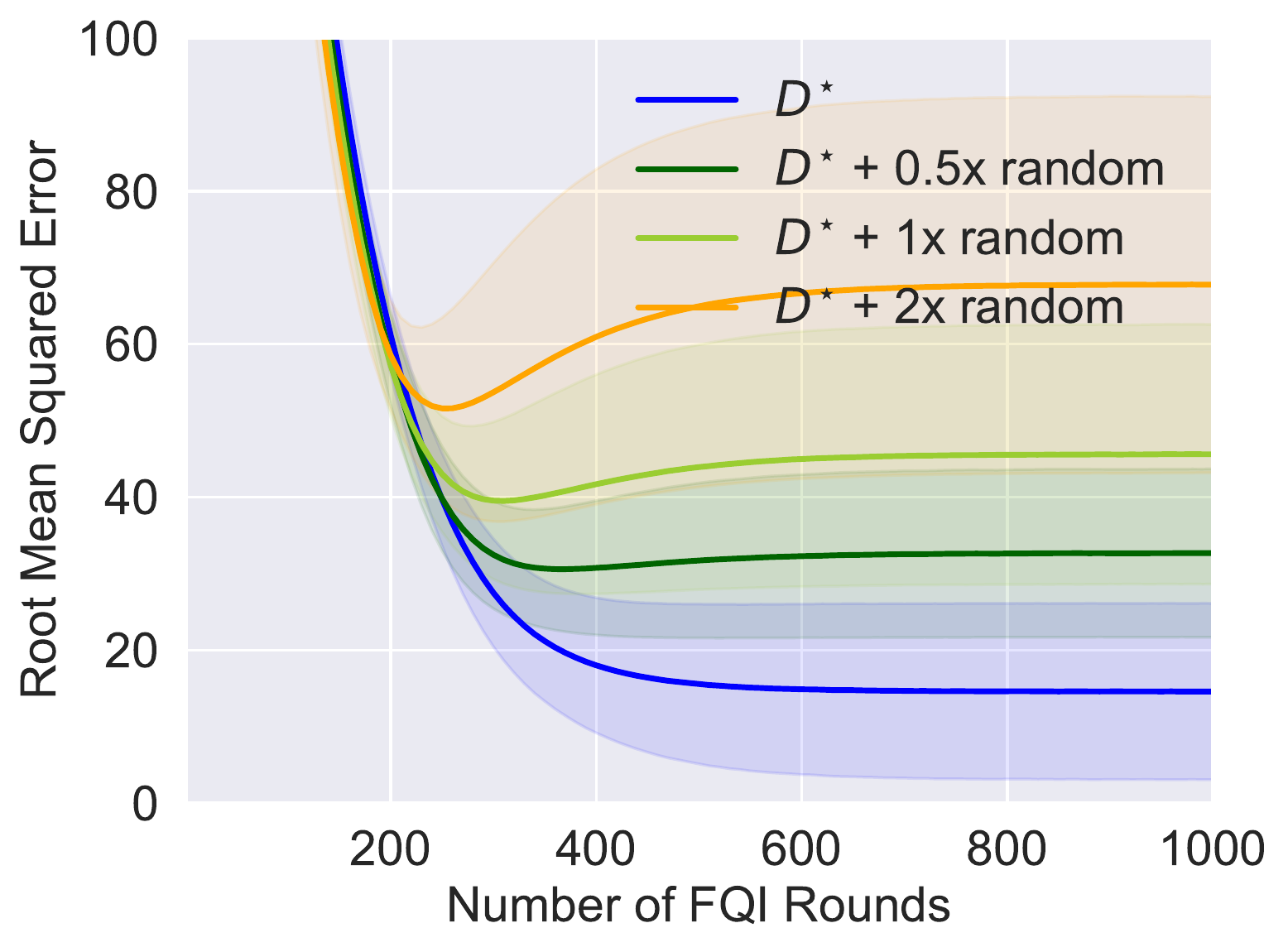}
\caption{We show the performance of FQI on Walker-2d
  v2 when applied to policy evaluation. Here the $x$-axis is the
  number of rounds we run FQI, and the $y$-axis is the square root of the mean squared
  error of the predicted values (smaller is better).
The blue line corresponds to performance when the dataset is generated by the target policy itself with $1$ million samples, and other lines correspond to the performance when adding more data induced by random trajectories. E.g., the orange line corresponds to the case where we add 2x more data (i.e., $2$ million extra samples) induced by random trajectories. 
As shown here, with more data induced by random trajectories added, the performance of FQI degrades. }
\label{fig:example}
\end{figure}

The main conclusion of this work, through extensive experiments on a
number of tasks, is that: \emph{we do in fact observe substantial error
amplification}, even when using pre-trained representations, even we
tune hyper-parameters, regardless of what the distribution was
shifted to; furthermore, this amplification even occurs under relatively mild distribution shift.
As an example, Figure~\ref{fig:example} shows the performance of
FQI on Walker-2d v2 when our offline dataset has $1$ millions samples generated by the target policy itself, with additional samples from random policies. 

These experiments also complement the recent hardness results
in~\citet{wang2020statistical} showing the issue of error
amplification is a real practical concern. From a practical point
of view, our experiments demonstrate that the definition of a good
representation is more subtle than in supervised learning.
These results also
raise a number of concerns about empirical practices employed in a
number of benchmarks, and they also have a number of implications for
moving forward (see
Section~\ref{sec:discussion} with regards to these
two points).


Finally, it is worth emphasizing that our findings are not suggesting that
offline RL is not possible. Nor does it suggest that there are
no offline RL successes, as there have been some successes in
realistic domains (e.g.~\citep{mandel2014offline,youtube}).
Instead, our emphasis is that the conditions for success in offline
RL, both from a theoretical and an empirical perspective, are
substantially stronger than those in supervised learning settings.

\section{Related Work}\label{sec:related}

\paragraph{Theoretical Understanding.}
Offline RL is closely related to the theory of Approximate Dynamic Programming~\citep{bertsekas1995neuro}.
Existing theoretical work~\citep{munos2003error, szepesvari2005finite, antos2008learning, munos2008finite, tosatto2017boosted, duan2020minimax} usually makes strong representation conditions. 
In offline RL, the most natural assumption would be realizability, which only assumes the value function of the policy to be evaluated lies in the function class, and existing theoretical work usually make assumptions stronger than realizability.
For example, \citet{szepesvari2005finite, duan2020minimax, wang2020statistical} assume (approximate) closedness under Bellman updates, which is much stronger than realizability. 
Polynomial sample complexity results are also obtained under the realizability assumption, albeit under coverage conditions~\cite{xie2020batch} or stringent distribution shift conditions~\citep{wang2020statistical}.
Technically, our characterization of FQI (in Section~\ref{sec:fqi}) is similar to the characterization of LSPE by~\citet{wang2020statistical}, although we work in the more practical discounted case while~\citet{wang2020statistical} work in the finite-horizon setting.

Error amplification induced by distribution shift is a known issue in the theoretical analysis of RL algorithms. 
See~\citep{gordon1995stable, gordon1996stable, munos1999barycentric, ormoneit2002kernel, kakade2003sample, zanette2019limiting} for discussion on this topic.
Recently, \citet{wang2020statistical} show that in the finite-horizon setting, without a strong
representation condition,
there are lower bounds exhibiting exponential error amplification unless the data collection distribution has only a mild distribution shift
relative to the target policy.
Such lower bound was later generalized to the discounted setting by~\citet{amortila2020variant}.
Similar hardness results are also obtained by~\citet{zanette2020exponential}, showing that offline RL could be exponentially harder than online RL. 

\paragraph{Empirical Work.} Error amplification in offline RL has been observed in empirical work~\citep{fujimoto2019off, kumar2019stabilizing} and was called ``extrapolation error'' in these work. 
For example, it has been observed in~\citep{fujimoto2019off} that DDPG~\citep{lillicrap2015continuous} trained on the replay buffer of online RL methods performs significantly worse than the behavioral agent. 
Compared to previous empirical study on the error amplification issue, in this work, we use pre-trained features which allow us to decouple practical representational learning questions from the offline RL question, where we can focus on offline RL with a given representation.
We also carefully control the data collection distributions, with different styles of shifted distributions (those induced by random trajectories or induced by lower performance policies) and different levels of noise. 

To mitigate the issue of error amplification, prior empirical work usually constrains the learned policy to be closer to the behavioral policy~\citep{fujimoto2019off, kumar2019stabilizing, wu2020behavior, jaques2020way, nachum2019algaedice, peng2019advantage,siegel2020keep, kumar2020conservative, yu2021combo} and utilizes uncertainty quantification~\citep{agarwal2019striving, yu2020mopo, kidambi2020morel, rafailov2020offline}. 
We refer interested readers to the survey by~\citet{levine2020offline} for recent developments on this topic. 
\section{Background}\label{sec:pre}
\paragraph{Discounted Markov Decision Process.} Let $\mdp =\left(\states, \actions, \trans ,R, \gamma, \mu_{\mathrm{init}}\right)$ be a \emph{discounted Markov Decision Process} (DMDP, or MDP for short)
where $\states$ is the state space, 
$\actions$ is the action space, 
$\trans: \states \times \actions \rightarrow \Delta\left(\states\right)$ is the transition operator, 
$R : \states \times \actions \rightarrow \Delta\left( \mathbb{R} \right)$ is the reward distribution,
$\gamma < 1$ is the discount factor
and $\mu_{\mathrm{init}} \in \Delta(\states)$ is the initial state distritbution.
A policy $\pi: \states \to \Delta\left(\actions\right)$ chooses an action $a$ based on the state $s$.
The policy $\pi$ induces a trajectory $s_0,a_0,r_0,s_1,a_1,r_1,\ldots$,
where $s_0 \sim \mu_{\mathrm{init}}$, $a_0 \sim \pi(s_0)$, $r_0 \sim R(s_0,a_0)$, $s_1 \sim \trans(s_0,a_0)$, $a_1 \sim \pi(s_1)$, etc.
For the theoretical analysis, we assume $r_h \in [0, 1]$.

\paragraph{Value Function.}
Given a policy $\pi$ and $(s,a) \in \states \times
\actions$, define
\[Q^\pi(s,a) = \expect\left[\sum_{h = 0}^{\infty} \gamma^h r_{h'}\mid s_0 =s, a_0 =
  a, \pi\right]\] and \[V^\pi(s)=\expect\left[\sum_{h = 0}^{\infty}\gamma^h r_{h}\mid s_0=s,
  \pi\right].\]
For a policy $\pi$, we define $V^{\pi} = \expect_{s_0 \sim \mu_{\mathrm{init}}} [V^{\pi}(s_0)]$ to be the expected
value of $\pi$ from the initial state distribution $\mu_{\mathrm{init}}$.

\paragraph{Offline Reinforcement Learning.}
This paper is concerned with the offline RL setting. 
In this setting, the agent does not have direct access to the MDP and instead is given access to a data distribution $\mu \in \Delta(\states \times \actions)$.
The inputs of the agent is a datasets $D$,  consisting of i.i.d. samples of the form $(s, a, r, s') \in \states
\times \actions \times \mathbb{R} \times \states$, where $(s,
a)\sim\mu$, $r \sim r(s, a)$, $s' \sim P(s, a)$.
We primarily focus on the \emph{offline policy evaluation} problem: given a target policy $\pi : \states \to \Delta\left(\actions\right)$, the goal is to output an
accurate estimate of the value of $\pi$ (i.e., $V^{\pi}$) approximately, using the collected
dataset $D$, with as few samples as possible.

\paragraph{Linear Function Approximation.} 
In this paper, we focus on offline RL with linear function approximation. 
When applying linear function approximation schemes, the agent is given a feature extractor $\phi : \states \times \actions \to \mathbb{R}^d$ which can either be hand-crafted or a pre-trained neural network, which transforms a state-action pair to a $d$-dimensional embedding, and it is commonly assumed that $Q$-functions can be predicted by linear functions of the features.
For our theoretical analysis, we assume $\|\phi(s, a)\|_2 \le 1$ for all $(s, a) \in \states \times \actions$. 
For the theoretical analysis, we are interested in the offline policy evaluation problem, under the following assumption.
\begin{assumption}[Realizability] \label{assmp:realizability}
For the policy $\pi :
\states \to \Delta(\actions)$ to evaluated, there exists $\theta^* \in \mathbb{R}^d$ with $\|\theta^*\|_2 \le \sqrt{d }/ (1 - \gamma)$\footnote{Without loss of generality, we assume that we work in a coordinate system such that $\|\theta^*\|_2 \le \sqrt{d} / (1 - \gamma)$ and $\|\phi(s, a)\|_2 \le 1$.
 }, such that for all $(s, a) \in \states \times \actions$, $Q^{\pi}(s, a) = \left(\theta^*\right)^{\top}\phi(s, a)$. Here $\pi$ is the target policy to be evaluated. 
\end{assumption}

\paragraph{Notation.} 
For a vector $x \in \mathbb{R}^d$ and a positive semidefinite matrix $A \in \mathbb{R}^{d \times d}$, we use $\|x\|_2$ to denote its $\ell_2$ norm, $\|x\|_A$ to denote $\sqrt{x^{\top}Ax}$, $\|A\|_2$ to denote its operator norm, and $\sigma_{\min}(A)$ to denote its smallest eigenvalue. 
For two positive semidefinite matrices $A$ and $B$, we write $A \succeq B$ if and only if $A - B$ is positive semidefinite.
\section{An Analysis of Fitted Q-Iteration in the Discounted Setting}\label{sec:fqi}
In order to illustrate the error amplification issue and discuss conditions that permit sample-efficient offline RL, 
in this section, we analyze Fitted Q-Iteration (FQI)~\citep{gordon1999approximate} when applied to the offline policy evaluation problem under the realizability assumption.
Here we focus on FQI since it is the prototype of many practical algorithms.
For example, when DQN~\citep{mnih2015human} is run on off-policy data, and the target network is updated slowly, it can be viewed as an analog of FQI, with neural networks being the function approximator.
We give a description of FQI in Algorithm~\ref{algo:upper}. 
We also perform experiments on temporal difference methods in our experiments (Section~\ref{sec:exp}).
For simplicity, we assume a deterministic target policy $\pi$.
\begin{algorithm}[t]
	\caption{Fitted Q-Iteration (FQI)}
	\label{algo:main}
	\begin{algorithmic}[1]
	\STATE \textbf{Input:} policy $\pi$ to be evaluated, number of samples $N$, regularization parameter $\lambda > 0$, number of rounds $T$
	\label{line:sample}
	\STATE Take samples $(s_i, a_i) \sim \mu$, $r_i \sim r(s_i, a_i)$ and $\overline{s}_i \sim \trans(s_i, a_i)$ for each $i \in [N]$
	\STATE $\hat{\Lambda} = \frac{1}{N}\sum_{i \in [N]} \phi(s_i, a_i)\phi(s_i, a_i)^{\top} + \lambda I$
	\STATE $Q_0(\cdot, \cdot) = 0$ and $V_0(\cdot) = 0$
	\FOR{$t = 1, 2, \ldots, T$}
	\STATE $\hat{\theta}_t =\hat{\Lambda}^{-1} (\frac{1}{N} \sum_{i = 1}^{N} \phi(s_i, a_i) \cdot(r_i + \gamma \hat{V}_{t - 1}(\overline{s}_i))) $
	\STATE $\hat{Q}_t(\cdot, \cdot) = \phi(\cdot, \cdot)^{\top} \hat{\theta}_t$ and $\hat{V}_t(\cdot) = \hat{Q}_t(\cdot, \pi(\cdot))$
	\ENDFOR
	\STATE \textbf{return} $\hat{Q}_T(\cdot, \cdot)$
	\end{algorithmic}
	\label{algo:upper}
\end{algorithm}
%

We remark that the issue of error amplification discussed here is similar to that in~\citet{wang2020statistical}, which shows that  if one just assumes realizability, geometric error amplification is inherent in offline RL in the finite-horizon setting. 
Here we focus on the discounted case which exhibit some subtle differences (see, e.g., ~\citet{amortila2020variant}).

\paragraph{Notation.} We define $\Phi$ to be a $N \times d$ matrix, whose $i$-th row is $\phi(s_i, a_i)$, and define $\overline{\Phi}$ to be another $N \times d$ matrix whose $i$-th row is $\phi(\overline{s}_i, \pi(\overline{s}_i))$ (see Algorithm~\ref{algo:upper} for the definition of $\overline{s}_i$). 
For each $i \in [N]$, define $\zeta_i= r_i + V(\overline{s}_i) - Q(s_i, a_i)$.
Clearly, $\expect[\zeta_i] = 0$.
We use $\zeta \in \mathbb{R}^N$ to denote a vector whose $i$-th entry is $\zeta_i$.
We use $\Lambda = \expect_{(s, a) \sim \mu}\expect[\phi(s, a) \phi(s, a)^{\top}]$ to denote the feature covariance matrix of the data distribution, and use
\[
\overline{\Lambda} = \expect_{(s, a) \sim \mu, s' \sim P(s, a)} \expect[\phi(s', \pi(s')) \phi(s', \pi(s'))^{\top}]
\]
 to denote the  feature covariance matrix of the
one-step lookahead distribution induced by $\mu$ and $\pi$. 
We also use $\Lambda_{\mathrm{init}} = \expect_{s \sim \mu_{\mathrm{init}}} [\phi(s, \pi(s))\phi(s, \pi(s))^{\top}]$ to denote the feature covariance matrix induced by the initial state distribution.


Now we present a general lemma that characterizes the estimation error of Algorithm 1 by an equality. 
Later, we apply this general lemma to special cases.
%
%
\begin{lemma}\label{lem:upper_main}
Under Assumption~\ref{assmp:realizability},
\[
\theta_T - \theta^* 
= \sum_{t = 1}^T \left(\gamma L\right)^{t - 1}  (\frac{\gamma}{N} \hat{\Lambda}^{-1}\Phi^{\top} \zeta - \lambda \hat{\Lambda}^{-1}\theta^*)  
+ \left(\gamma L\right)^T  \theta^*
\]
where $L = \hat{\Lambda}^{-1}\Phi^{\top}\overline{\Phi} / N$. 
\end{lemma}

By Lemma~\ref{lem:upper_main}, to achieve a bounded error, the matrix $L = \hat{\Lambda}^{-1}\Phi^{\top}\overline{\Phi} / N$ should satisfy certain non-expansive properties.
Otherwise, the estimation error grows exponentially as $t$ increases, and geometric error amplification occurs. 
Now we discuss two cases when geometric error amplification does not
occur, in which case the estimation error can be bounded with a
polynomial number of samples.  

\paragraph{Policy Completeness.} 
The policy completeness assumption~\citep{szepesvari2005finite, duan2020minimax} assumes the feature mapping is complete under bellman updates.

\begin{assumption}[Policy Completeness]\label{assumption:completeness}
For any $\theta \in \mathbb{R}^d$, there exists $\theta' \in \mathbb{R}^d$, such that for any $(s, a) \in \states \times \actions$, \[\phi(s, a)^{\top}\theta' = \expect_{r \sim R(s, a), s' \sim P(s, a)}[r + \gamma \phi(s', \pi(s'))^{\top} \theta].\]
\end{assumption}

Now we show that under Assumption~\ref{assumption:completeness}, FQI achieves bounded error with polynomial number of samples. 
\begin{lemma}\label{lem:fqi_completeness}
Suppose $N \ge \mathrm{poly}(d, 1 / \varepsilon, 1 / (1 - \gamma), 1 / \sigma_{\min}(\Lambda))$, by taking $T \ge C \log\left(d / (\varepsilon(1 - \gamma))\right) / (1 - \gamma)$ for some constant $C > 0$, we have \[|\hat{Q}_T(s, a) - Q^{\pi}(s, a)| \le \varepsilon\] for all $(s, a) \in \states \times \actions$. 
\end{lemma}

\begin{proof}[Proof Sketch]
By Lemma~\ref{lem:upper_main}, it suffices to show the non-expansiveness of $L =  \hat{\Lambda}^{-1}\Phi^{\top}\overline{\Phi} / N$.
For intuition, let us consider the case where $N \to
\infty$ and $\lambda \to 0$.
Let $\Phi_{\mathrm{all}} \in \mathbb{R}^{|\states| |\actions| \times d}$ denote a matrix whose row indexed by $(s, a) \in \states \times \actions$ is $\phi(s, a) \in \mathbb{R}^d$.
Let $D^{\mu} \in \mathbb{R}^{|\states| |\actions| \times |\states| |\actions|}$ denote a diagonal matrix whose diagonal entry indexed by $(s, a)$ is $\mu(s, a)$. 
We use $P^{\pi} \in \mathbb{R}^{|\states| |\actions| \times |\states| |\actions|}$ to denote a matrix where
\[
P^{\pi}((s, a), (s', a')) = \begin{cases}
P(s' \mid s, a) & a' = \pi(s)\\
0 & \text{otherwise} 
\end{cases}.\]
When $N \to \infty$ and $\lambda \to 0$, we have $\hat{\Lambda} = \Phi_{\mathrm{all}}^{\top}D^{\mu}\Phi_{\mathrm{all}}$ and $\frac{\Phi^{\top}\overline{\Phi}}{N} = \Phi_{\mathrm{all}}^{\top} D^{\mu} P^{\pi} \Phi_{\mathrm{all}}$.
By Assumption~\ref{assumption:completeness}, for any $x \in \mathbb{R}^d$, there exists $x'$ such that $P^{\pi}\Phi_{\mathrm{all}} x = \Phi_{\mathrm{all}} x'$.
Thus, 
\[
\Phi_{\mathrm{all}}  \hat{\Lambda}^{-1}\frac{\Phi^{\top}\overline{\Phi}}{N} x
= \Phi_{\mathrm{all}} (\Phi_{\mathrm{all}}^{\top}D^{\mu}\Phi_{\mathrm{all}} )^{-1} \Phi_{\mathrm{all}}^{\top} D^{\mu} P^{\pi} \Phi_{\mathrm{all}} x = \Phi_{\mathrm{all}} x' = P^{\pi} \Phi_{\mathrm{all}} x.
\]
Therefore, the
magnitude of entries in $\Phi x$ will not be amplified after applying
$\hat{\Lambda}^{-1}\Phi^{\top}\overline{\Phi} / N$ onto $x$ since $\|P^{\pi} \Phi_{\mathrm{all}} x\|_{\infty} \le \|\Phi_{\mathrm{all}} x\|_{\infty}$.

In the formal proof, we combine the above with standard concentration arguments to obtain a finite-sample rate. 
\end{proof}

We remark that variants of Lemma~\ref{lem:fqi_completeness} have been known in the literature (see, e.g.,~\citep{duan2020minimax} for a similar analysis in the finite-horizon setting). Here we present Lemma~\ref{lem:fqi_completeness} mainly to illustrate the versatility of Lemma~\ref{lem:upper_main}.  

\paragraph{Low Distribution Shift.}
Now we focus on the case where the distribution shift between the data distributions and the distribution induced by the target policy is low. 
Here, our low distribution shift condition is similar to that in~\citet{wang2020statistical}, though we focus on the discounted case while \citet{wang2020statistical} focus on the finite-horizon case. 

To measure the distribution shift, our main assumption is as follows.

\begin{assumption}[Low Distribution Shift]\label{assumption:low_distribution_shift}
There exists $C_{\mathrm{policy}} < 1 / \gamma^2$ such that $\overline{\Lambda} \preceq C_{\mathrm{policy}} \Lambda$ and $\Lambda_{\mathrm{init}} \preceq C_{\mathrm{init}} \Lambda$
\end{assumption}

Assumption~\ref{assumption:low_distribution_shift} assumes that the data distribution dominates both the one-step lookahead distribution and the initial state distribution, in a spectral sense.
For example, when the data distribution $\mu$ is induced by the target policy $\pi$ itself, and $\pi$ induces the same data distribution for all timesteps $t \ge 0$, Assumption~\ref{assumption:low_distribution_shift} holds with $C_{\mathrm{policy}} = C_{\mathrm{eval}} = 1$.
In general,  $C_{\mathrm{policy}}$ characterizes the difference between the data distribution and the one-step lookahead distribution induced by the data distribution and the target policy.

Now we show that under Assumption~\ref{assumption:low_distribution_shift}, FQI achieves bounded error with polynomial number of samples.
The proof can be found in the appendix. 

\begin{lemma}\label{lem:fqi_distribution_shift}
Suppose $N \ge \mathrm{poly}(d, 1 / \varepsilon, 1 / (1 - C_{\mathrm{policy}}^{1/2}\gamma), 1 / C_{\mathrm{init}})$.
By taking $T \ge C \log (d \cdot C_{\mathrm{init}} / (\varepsilon(1 -C_{\mathrm{policy}}^{1/2} \gamma))) / (1 - C_{\mathrm{policy}}^{1/2}\gamma)$ for some $C > 0$, under Assumption~\ref{assmp:realizability} and~\ref{assumption:low_distribution_shift}, we have
$\expect_{s \sim \mu_{\mathrm{init}}}[(V^{\pi}(s) - \hat{V}_T(s))^2] \le \varepsilon$. 
\end{lemma}

\subsection{Simulation Results}
We now provide simulation results on a synethic environment to better illustrate the issue of error amplification and the tightness of our characterization of FQI in Lemma~\ref{lem:upper_main}. 

\begin{figure*}[h]
	\centering
	\subfigure[$N  = 100$]{
	\includegraphics[width=0.48\textwidth]{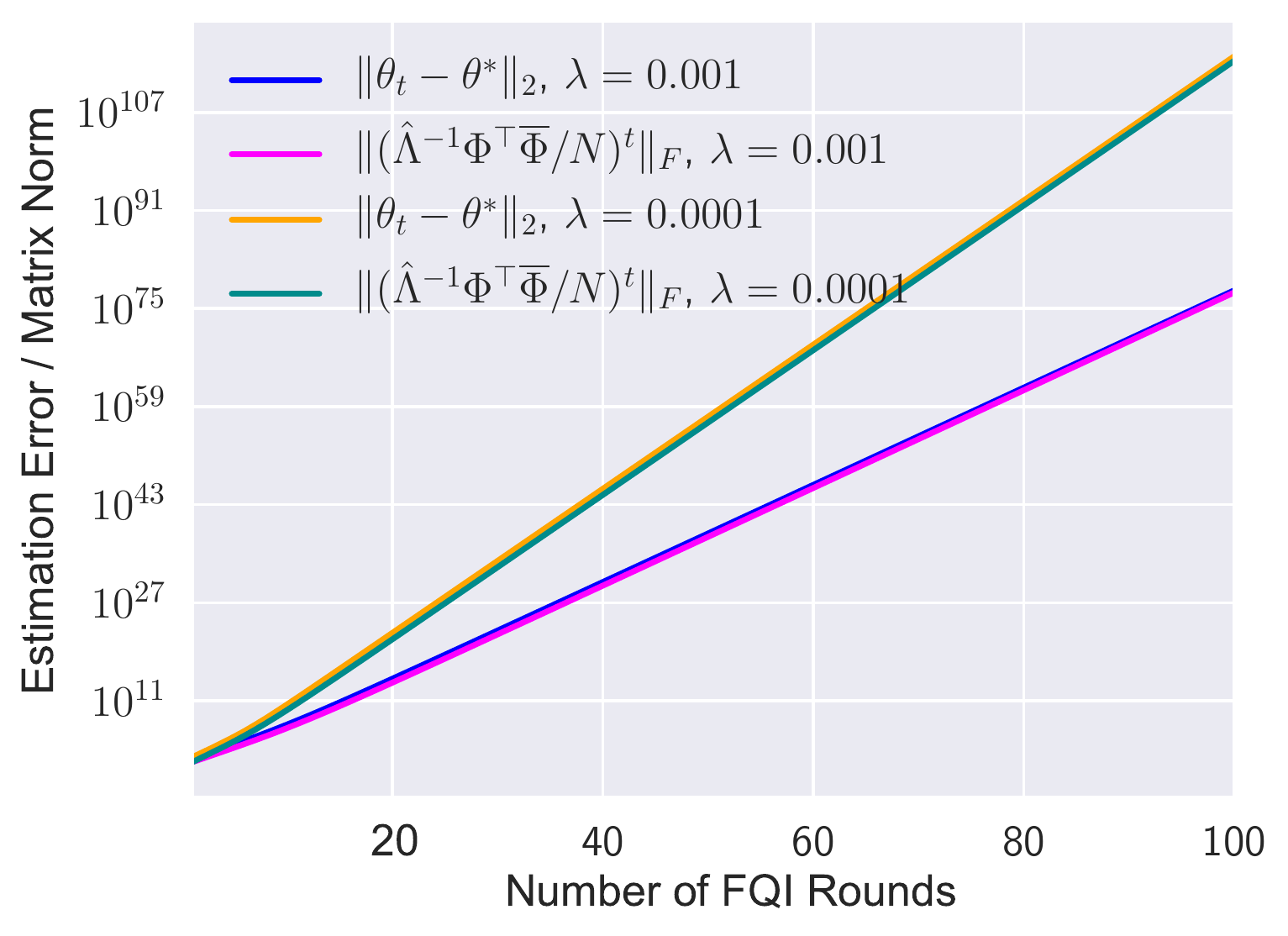}
	}
	\subfigure[$N = 200$]{
	\includegraphics[width=0.48\textwidth]{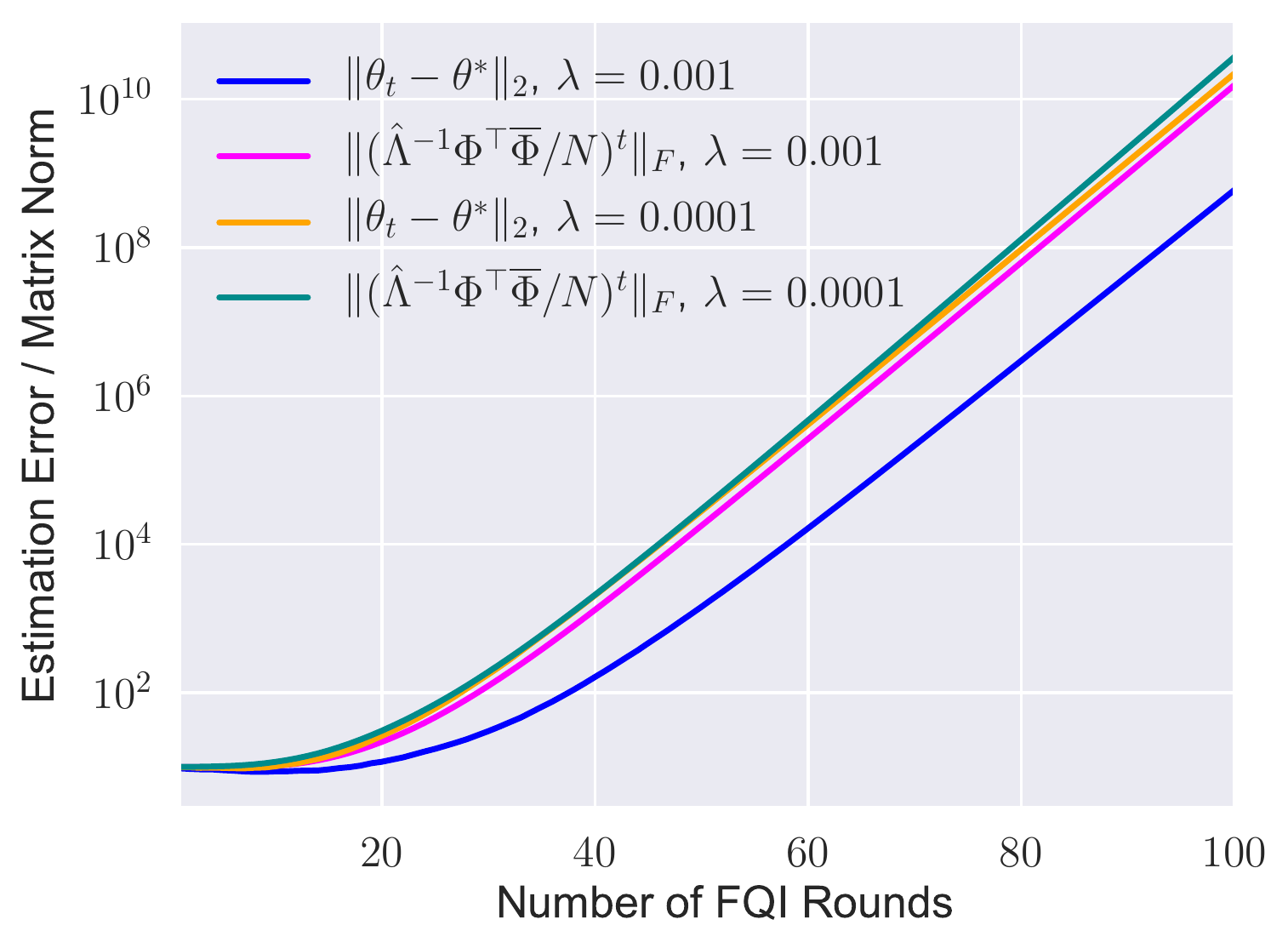}
	}
	\caption{In this figure we report the results of the simulation when $N = 100$ and $N = 200$. The $x$-axis is the number of rounds $t$ we run FQI, while the $y$-axis is the estimation error $\|\theta_t - \theta^*\|_2$ or the the Frobenius norm of $(\hat{\Lambda}^{-1}\Phi^{\top}\overline{\Phi})^t$, taking average over $100$ repetitions. }
	\label{fig:simulation}
\end{figure*}

\paragraph{Simulation Settings.}
In our construction, the number of data points is $|D| = N$, where $N = 100$ or $N = 200$.
The feature dimension is fixed to be $d = 100$ and the discount factor $\gamma = 0.99$. 
We draw $\theta^*$ from $\mathcal{N}(0, I_d)$. 
The data distribution, the transition operator and the rewards are all deterministic in this environment. 
For each $(s_i, a_i, r_i, s_i') \in D$, $\phi(s_i, a_i)$ and $\phi(s_i', \pi(s_i'))$ are independently drawn from $\mathcal{N}(0, I_d)$, and $r_i = \phi(s, a)^{\top}\theta^* - \gamma \phi(s', \pi(s'))^{\top} \theta^*$ so that Assumption~\ref{assmp:realizability} holds. 
We then run FQI in Algorithm~\ref{algo:upper}, by setting $T = 100$ and $\lambda = 10^{-4}$ or $10^{-3}$.
In Figure~\ref{fig:simulation}, we plot the estimation error $\|\theta_t - \theta^*\|_2$ and the Frobenius norm of $(\hat{\Lambda}^{-1}\Phi^{\top}\overline{\Phi})^t$, for $t = 1, 2, \ldots, 100$. 
We repeat the experiment for $100$ times and report the mean estimation error and the mean Frobenius norm of $(\hat{\Lambda}^{-1}\Phi^{\top}\overline{\Phi})^t$.

We remark that our dataset $D$ has sufficient coverage over the feature space, both when $N = 100$ and $N = 200$.
This is because the feature covariance matrix has lower bounded eigenvalue with high probability in both cases, according to standard random matrix theory~\citep{chen2005condition}. 
%

\paragraph{Results.}
For deterministic environments, by Lemma~\ref{lem:upper_main}, the estimation error is dominated by $(\hat{\Lambda}^{-1}\Phi^{\top}\overline{\Phi})^t \theta^*$.
As shown in Figure~\ref{fig:simulation}, geometric error amplification does occur, and the norm of $(\hat{\Lambda}^{-1}\Phi^{\top}\overline{\Phi})^t$ grows exponentially as $t$ increases. 
Moreover, the norm of $(\hat{\Lambda}^{-1}\Phi^{\top}\overline{\Phi})^t$ has almost the same growth trend as $\|\theta_t - \theta^*\|_2$.
E.g., when $N = 200$, the estimation error $\|\theta_t - \theta^*\|_2$ grows exponentially, although much slower than the case when $N = 100$.
In that case, the norm of $(\hat{\Lambda}^{-1}\Phi^{\top}\overline{\Phi})^t$  also increases much slower than the case when $N = 200$.
Our simulation results show that the issue of error amplification could occur even in simple environments, and our theoretical result in Lemma~\ref{lem:upper_main} gives a tight characterization of the estimation error. 


\section{Experiments}\label{sec:exp}

%
The goal of our experimental evaluation is to understand whether offline RL methods are sensitive to distribution shift in practical tasks, given a good representation (features extracted from pre-trained neural networks or random features).
Our experiments are performed on a range of challenging tasks from the OpenAI gym benchmark suite~\citep{brockman2016openai}, including two environments with discrete action space (MountainCar-v0, CartPole-v0) and four environments with continuous action space (Ant-v2, HalfCheetah-v2, Hopper-v2, Walker2d-v2).
We also provide further discussion on our methodologies in Section~\ref{sec:further}.

\subsection{Experimental Methodology}
Our methodology proceeds according to the following steps:

\begin{enumerate}
\item We \textbf{decide on a (target) policy} to be evaluated, along with a
  good feature mapping for this policy.  
\item \textbf{Collect offline data} using trajectories that are a mixture of the target policy
along with another distribution. 
\item \textbf{Run offline RL methods} to evaluate the target policy
  using the feature mapping found in Step~1 and the offline data obtained in
  Step~2. 
\end{enumerate}

We now give a detailed description for each step.

\paragraph{Step 1: Determine the Target Policy. }
To find a policy to be evaluated together with a good representation, we run classical online RL methods.
For environments with discrete action space (MountainCar-v0, CartPole-v0), we run Deep Q-learning (DQN)~\citep{mnih2015human}, while for environments with continuous action space (Ant-v2, HalfCheetah-v2, Hopper-v2, Walker2d-v2), we run Twin Delayed Deep Deterministic policy gradient (TD3)~\citep{fujimoto2018addressing}.
The hyperparameters used can be found in Section~\ref{sec:exp_details}. 
The target policy is set to be the final policy output by DQN or TD3. 
We also set the feature mapping to be the output of the last hidden layer of the learned value function networks, extracted in the final stage of the online RL methods. 
Since the target policy is set to be the final policy output by the online RL methods, such feature mapping contains sufficient information to represent the value functions of the target policy.
We also perform experiments using random Fourier features~\citep{rahimi2007random}.

\paragraph{Step 2: Collect Offline Data. }
We consider two styles of shifted distributions: distributions induced by random policies and by lower performance policies. 
When the data collection policy is the same as the target policy, we
will see that offline methods achieve low estimation error, as expected.
In our experiments, we use the target policy to generate a dataset $D^{\star}$ with $1$ million samples. 
We then consider two types of datasets induced by shifted distributions: adding random trajectories into $D^{\star}$, and adding samples induced by lower performance policies into $D^{\star}$.
In both cases, the amount of data from the target policy remains unaltered (fixed to be $1$ million). 
For the first type of dataset, we add 0.5 million, 1 million, or 2 million samples from random trajectories into $D^{\star}$.
For the second type of dataset, we manually pick four lower performance policies $\pi_{\mathrm{sub}}^1, \pi_{\mathrm{sub}}^2, \pi_{\mathrm{sub}}^3, \pi_{\mathrm{sub}}^4$ with $V^{\pi_{\mathrm{sub}}^1} > V^{\pi_{\mathrm{sub}}^2} > V^{\pi_{\mathrm{sub}}^3} > V^{\pi_{\mathrm{sub}}^4}$, and use each of them to collect $1$ million samples.
We call these four datasets (each with $1$ million samples) $D_{\mathrm{sub}}^1, D_{\mathrm{sub}}^2, D_{\mathrm{sub}}^3, D_{\mathrm{sub}}^4$, and we run offline RL methods on $D^{\star} \cup D_{\mathrm{sub}}^i$ for each $i \in \{1, 2, 3, 4\}$.

\paragraph{Step 3: Run Offline RL Methods. }
With the collected offline data and the target policy (together with a good representation), we can now run offline RL methods to evaluate the (discounted) value of the target policy.
In our experiments, we run FQI (described in Section~\ref{sec:fqi}) and Least-Squares Temporal Difference\footnote{See the Section~\ref{sec:exp_details} for a description of LSTD.} (LSTD, a temporal difference offline RL method)~\citep{bradtke1996linear}. 
For both algorithms, the only hyperparameter is the regularization parameter $\lambda$ (cf. Algorithm~\ref{algo:upper}), which we choose from $\{10^{-1}, 10^{-2}, 10^{-3}, 10^{-4}, 10^{-8}\}$.
In our experiments, we report the performance of the best-performing
$\lambda$ (measured in terms of the square root of the mean squared
estimation error in the final stage of the algorithm, taking average
over all repetitions of the experiment); such favorable
hyperparameter tuning is clearly not possible in practice (unless we
have interactive access to the environment). See
Section~\ref{sec:results} for more discussion on hyperparameter
tuning. 

%

In our experiments, we repeat this whole process $5$ times.
For each FQI round, we report the square root of the mean squared evaluation error, taking
average over $100$ randomly chosen states. 
We also report the values ($V^{\pi}(s)$) of those randomly chosen
states in Table~\ref{table:value}.
We note that in our experiments, the randomness combines both from the feature generation process (representation uncertainty, Step 1) and the dataset (Step 2).
Even though we draw millions of samples in Step 2, the estimation of FQI could still have high variance.
Note that this is consistent with our theory in Lemma~\ref{lem:upper_main}, which shows that the variance can also be exponentially amplified without strong representation conditions and low distribution shift conditions.  
We provide more discussion regarding this point in Section~\ref{sec:details_3}. 

\begin{table*}
\centering
\begin{tabular}{lcccc}
\hline
Dataset & $D^{\star}$ & $D^{\star}$ + 0.5x random & $D^{\star}$ + 1x random & $D^{\star}$+ 2x random\\
\hline


Hopper-v2 & $2.18 \pm 1.14$  &  $9.38 \pm 3.84$  &  $13.18 \pm 2.77$  &  $16.86 \pm 2.84$ \\

Walker2d-v2 & $13.88 \pm 11.22$  &  $32.73 \pm 11.05$  &  $45.61 \pm 17.06$  &  $67.78 \pm 24.77$ \\
\hline
\end{tabular}
\caption{\emph{Performance of LSTD}. Performance of LSTDwith features from pre-trained neural networks and distributions induced by random policies.
Each number of is the square root of the mean squared error of the
estimation, taking average over $5$ repetitions, $\pm$ standard
deviation. 
}
\label{table:lstd}
\end{table*}
\begin{table}[h]
\centering
\begin{tabular}{ll}
\hline
Environment & Discounted Value\\
\hline
CartPole-v0 & $90.17 \pm 20.61$\\
Hopper-v2 & $321.42 \pm 30.26$\\
Walker2d-v2 &$336.64 \pm 49.80$\\
\hline
\end{tabular}
\caption{\emph{Values of Randomly Chosen States}. Mean value of the $100$ randomly chosen states (used for evaluating the estimations), $\pm$ standard deviation. }
\label{table:value}
\end{table}

\subsection{Results and Analysis}\label{sec:results}
Due to space limitations, we present experiment results on
Walker2d-v2, Hopper-v2 and CartPole-v0. Other experimental results are
provided in Section~\ref{sec:exp_results}.

\paragraph{Distributions Induced by Random Policies.}
We first present the performance of FQI with features from pre-trained neural networks and distributions induced by random policies.
The results are reported in Figure~\ref{fig:noise}.
Perhaps surprisingly, compared to the result on $D^{\star}$, adding
more data (from random trajectories) into the dataset generally hurts
the performance.  
With more data added into the dataset, the performance generally becomes worse. 
Thus, even with features from pre-trained neural networks, the performance of offline RL methods is still sensitive to data distribution.

\paragraph{Distributions Induced by Lower Performance Polices. }
Now we present the performance of FQI with features from pre-trained neural networks and datasets with samples from lower performance policies.
The results are reported in Figure~\ref{fig:suboptimal}.
Similar to Figure~\ref{fig:noise}, adding more data into the dataset could hurt performance, and the performance of FQI is sensitive to the quality of the policy used to generate samples. 
Moreover, the estimation error increases exponentially in some cases (see, e.g., the error curve of $D^{\star} \cup D_{\mathrm{sub}}^2$ in Figure~\ref{fig:exp_blow_up}), showing that geometric error amplification is not only a theoretical consideration, but could occur in practical tasks when given a good representation as well. 

\paragraph{Random Fourier Features.}
Now we present the performance of FQI with random Fourier features and distributions induced by random policies.
The results are reported in Figure~\ref{fig:rff}.
Here we tune the hyperparameters  of the random Fourier features so that FQI achieves reasonable performance on $D^{\star}$. 
Again, with more data from random trajectories added into the dataset, the performance generally becomes worse. 
This implies our observations above hold not only for features from pre-trained neural networks, but also for random features. 
On the other hand, it is known random features achieve reasonable performance in policy gradient methods~\citep{NIPS2017_7233} in the online setting. 
This suggests that the representation condition required by offline policy evaluation could be stronger than that of policy gradient methods in online setting.

\begin{table}
\centering
\begin{tabular}{l|cccc}
\hline
\multirow{ 2}{*}{Environment} & $\pi_{\mathrm{sub}}^1$ & $\pi_{\mathrm{sub}}^2$ & $\pi_{\mathrm{sub}}^3$ & $\pi_{\mathrm{sub}}^4$\\
& \multicolumn{4}{c}{RMSE / Gap between target policy and comparison policy } \\
\hline
Ant-v2 & \color{red} $>$1000 / 26.18  & \color{red} $>$1000 / 35.07  & \color{red} $>$1000 / 145.83  & \color{red} $>$1000 / 146.15 \\
CartPole-v0 & \color{red} 6.58 / 4.18  & \color{red} $>$1000 / 7.04  & \color{red} 9.16 / 8.10  & \color{red} 15.08 / 12.86 \\
HalfCheetah-v2 & \color{blue} 35.54 / 118.29  & \color{blue} 36.45 / 166.31  & \color{blue} 36.81 / 346.05  & \color{blue} 86.31 / 482.80 \\
Hopper-v2 & \color{red} 36.22 / -4.84  & \color{red} 35.54 / -4.37  & \color{red} 165.11 / 5.97  & \color{red} $>$1000 / 17.43 \\
MountainCar-v0 & \color{red} 1.46 / 0.74  & \color{red} 2.19 / 1.54  & \color{blue} 2.43 / 2.64  & \color{red} $>$1000 / 3.98 \\
Walker2d-v2 & \color{red} 121.35 / 5.28  & \color{red} $>$1000 / 34.48  & \color{red} 43.47 / 35.30  & \color{red} $>$1000 / 107.93 \\
\hline
\end{tabular}
\caption{\emph{Policy Comparison Results}. Best viewed in color. We seek to determine if FQI
  succeeds in comparing policies (with using features from pre-trained
  neural networks) with datasets induced by lower performing policies.
  Roughly, a red entry can be viewed as a failure (ideally, we would
  hope that the gap is at least a factor of $2$ larger than the RMSE).
For each entry, the first number is the root mean squared error of the estimation of the target policy of FQI, over $5$ repetitions and $100$ randomly chosen initial states. 
The second number is the average gap between the value of the target policy and that of the lower performing policy ($\pi_{\mathrm{sub}}^1$, $\pi_{\mathrm{sub}}^2$, $\pi_{\mathrm{sub}}^3$, or $\pi_{\mathrm{sub}}^4$), over $5$ repetitions and evaluated using $100$ trajectories. 
An entry is marked \textcolor{red}{red} if the root mean squared error is larger than
the average gap, and is marked \textcolor{blue}{blue} otherwise. We write
$>1000$ when the root mean squared error is larger than $1000$.
}
\label{table:pc}
\end{table}

\paragraph{Policy Comparison.}
We further study whether it is possible to compare the value of the target policy and that of the lower performing policies using FQI. 
In Table~\ref{table:pc}, we present the policy comparison results of FQI with features from pre-trained neural networks and datasets induced by lower performing policies. 
For each lower performing policy $\pi_{\mathrm{sub}}^i$ where $i \in \{1, 2, 3, 4\}$, we report the root mean squared error of FQI when evaluating the target policy using $D^{\star} \cup D_{\mathrm{sub}}^i$, and the average gap between the value of the target policy and that of $\pi_{\mathrm{sub}}^i$.
If the root mean squared error is less than the average gap, then we mark the corresponding entry to be green (meaning that FQI can distinguish between the target policy and the lower performing policy).
If the root mean squared error is larger than the average gap, then we mark the corresponding entry to be red (meaning that FQI cannot distinguish between the target policy and the lower performing policy).
From Table~\ref{table:pc}, it is clear that for most settings, FQI cannot distinguish between the target policy and the lower performing policy.

\paragraph{Sensitivity to Hyperparameters.}
In previous experiments, we tune the regularization parameter $\lambda$ and report the performance of the best-performing $\lambda$.
However, we remark that in practice, without access to online samples, hyperparameter tuning is hard in offline RL. 
Here we investigate how sensitive FQI is to different regularization parameters $\lambda$.
The results are reported in Figure~\ref{fig:rff}.
Here we fix the environment to be Walker2d-v2 and vary the number of additional samples from random trajectories and the regularization parameter $\lambda$.
As observed in experiments, the regularization parameter $\lambda$ significantly affects the performance of FQI, as long as there are random trajectories added into the dataset.

\paragraph{Performance of LSTD.} Finally, we present the performance of LSTD with features from pre-trained neural networks and distributions induced by random policies.
The results are reported in Table~\ref{table:lstd}. 
With more data from random trajectories added into the dataset, the performance of LSTD becomes worse. 
This means the sensitivity to distribution shift is not specific to FQI, but also holds for LSTD.

\begin{figure*}[!h]
\thisfloatpagestyle{empty}
\centering
\subfigure[Walker2d-v2]{
\includegraphics[width=0.30\textwidth]{fig/fig_Walker2d-v2_0}
}
\subfigure[Hopper-v2]{
\includegraphics[width=0.30\textwidth]{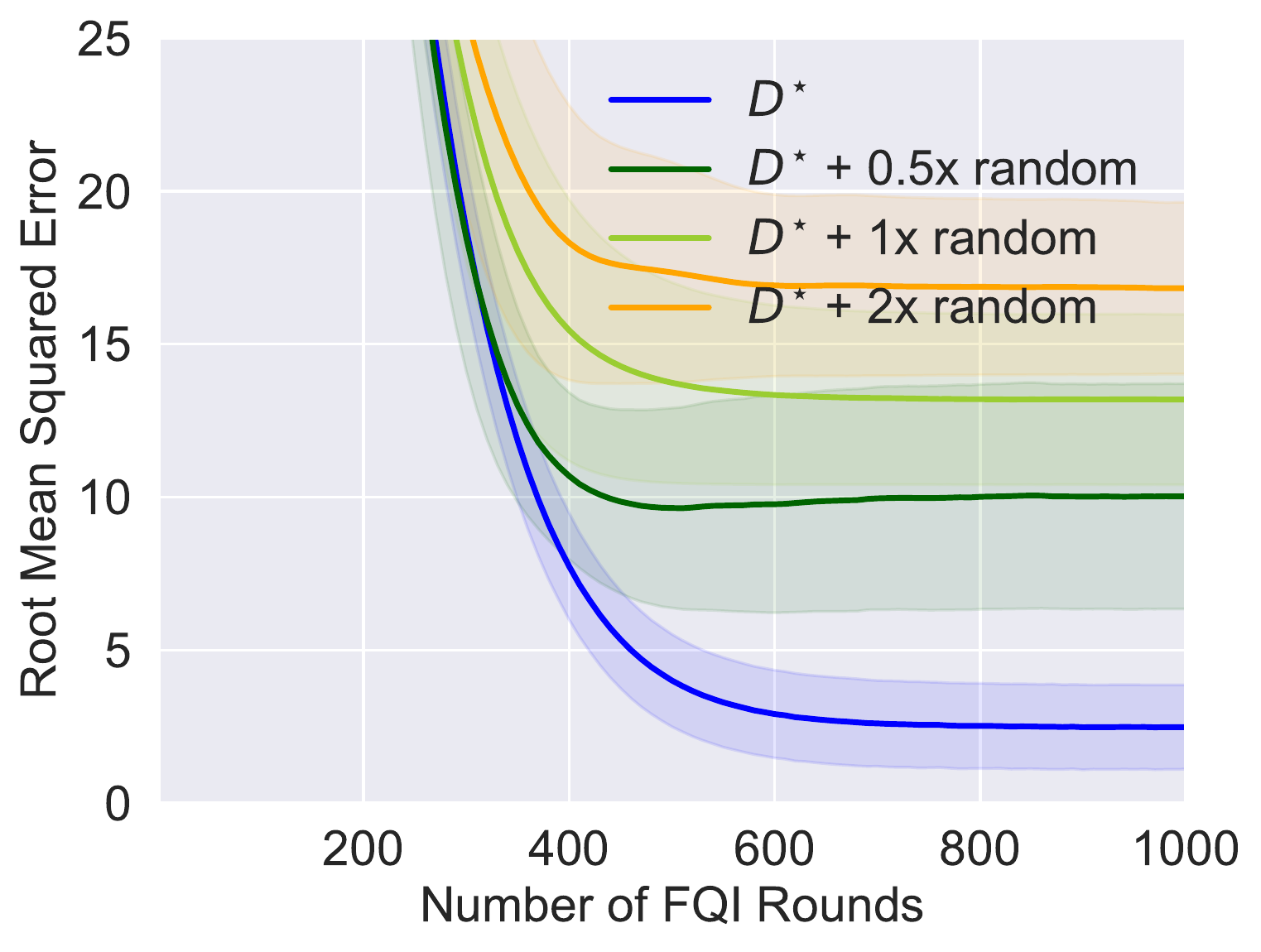}
}
\subfigure[CartPole-v0]{
\includegraphics[width=0.30\textwidth]{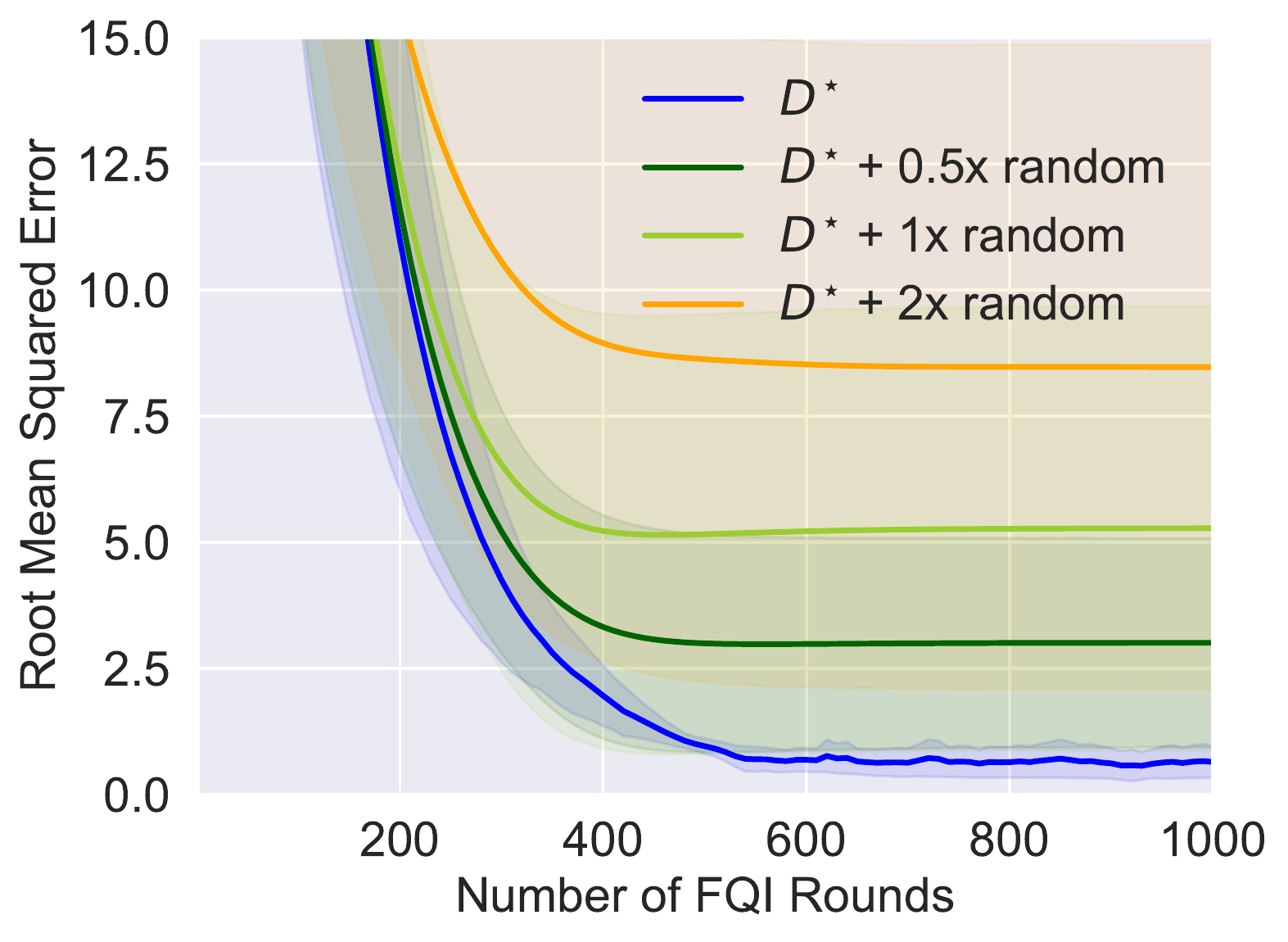}
}
\caption{Performance of FQI with features from pre-trained neural networks and datasets induced by random policies.
See Figure~\ref{fig:noise_add} in Section~\ref{sec:exp_results} for results on other environments. 
}
\label{fig:noise}

\subfigure[Walker2d-v2]{\label{fig:exp_blow_up}
\includegraphics[width=0.30\textwidth]{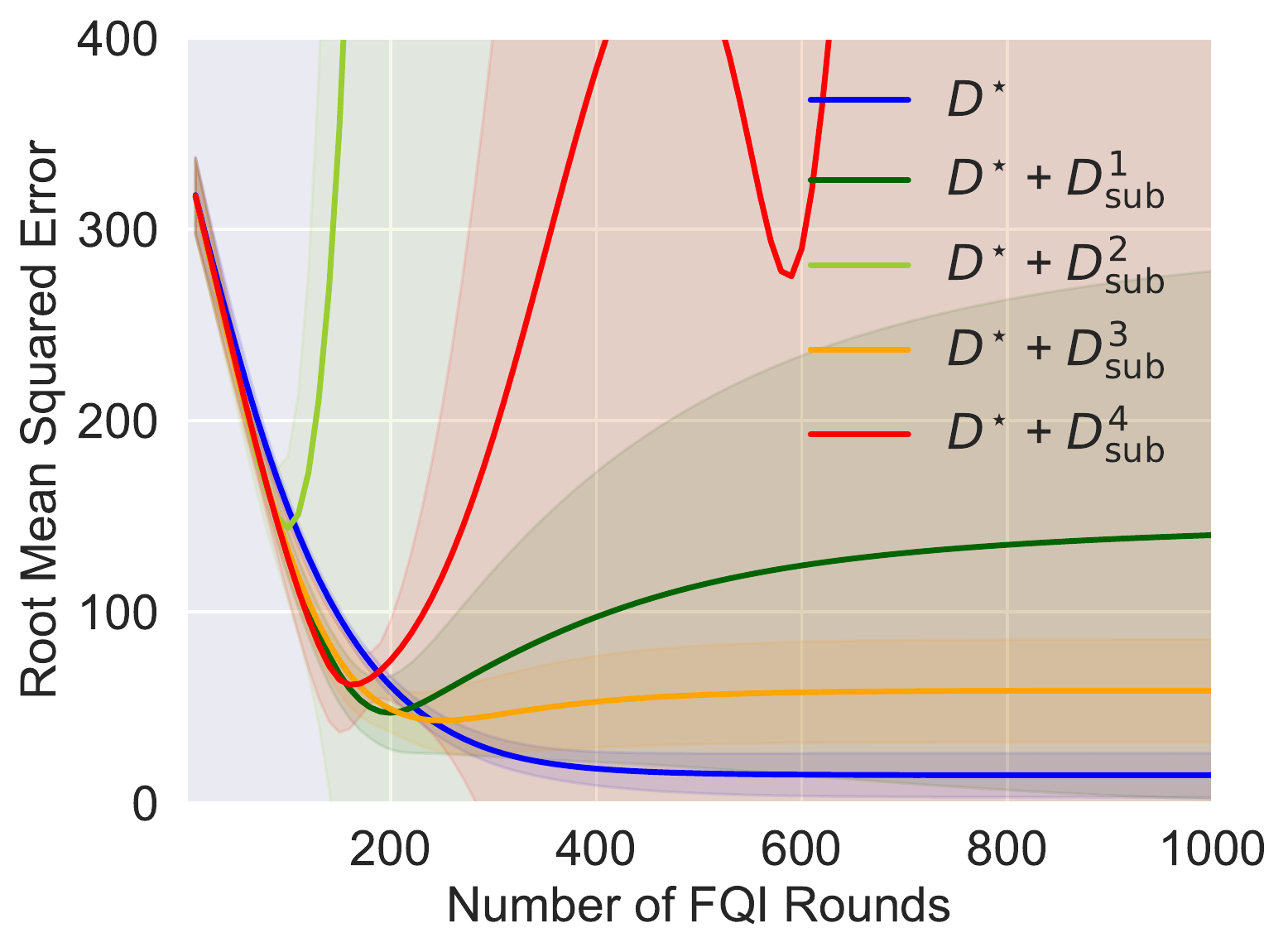}
}
\subfigure[Hopper-v2]{
\includegraphics[width=0.30\textwidth]{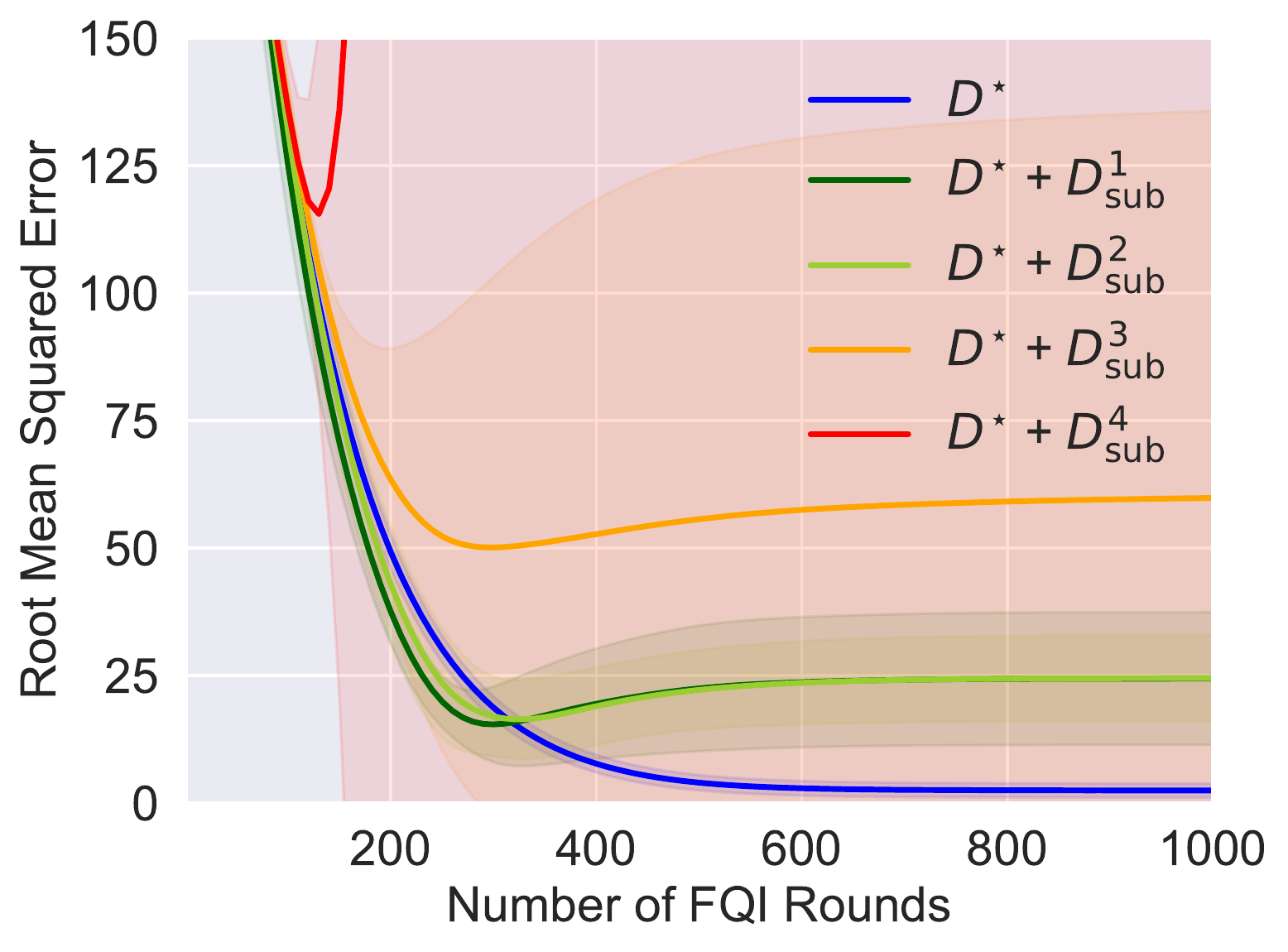}
}
\subfigure[CartPole-v0]{
\includegraphics[width=0.30\textwidth]{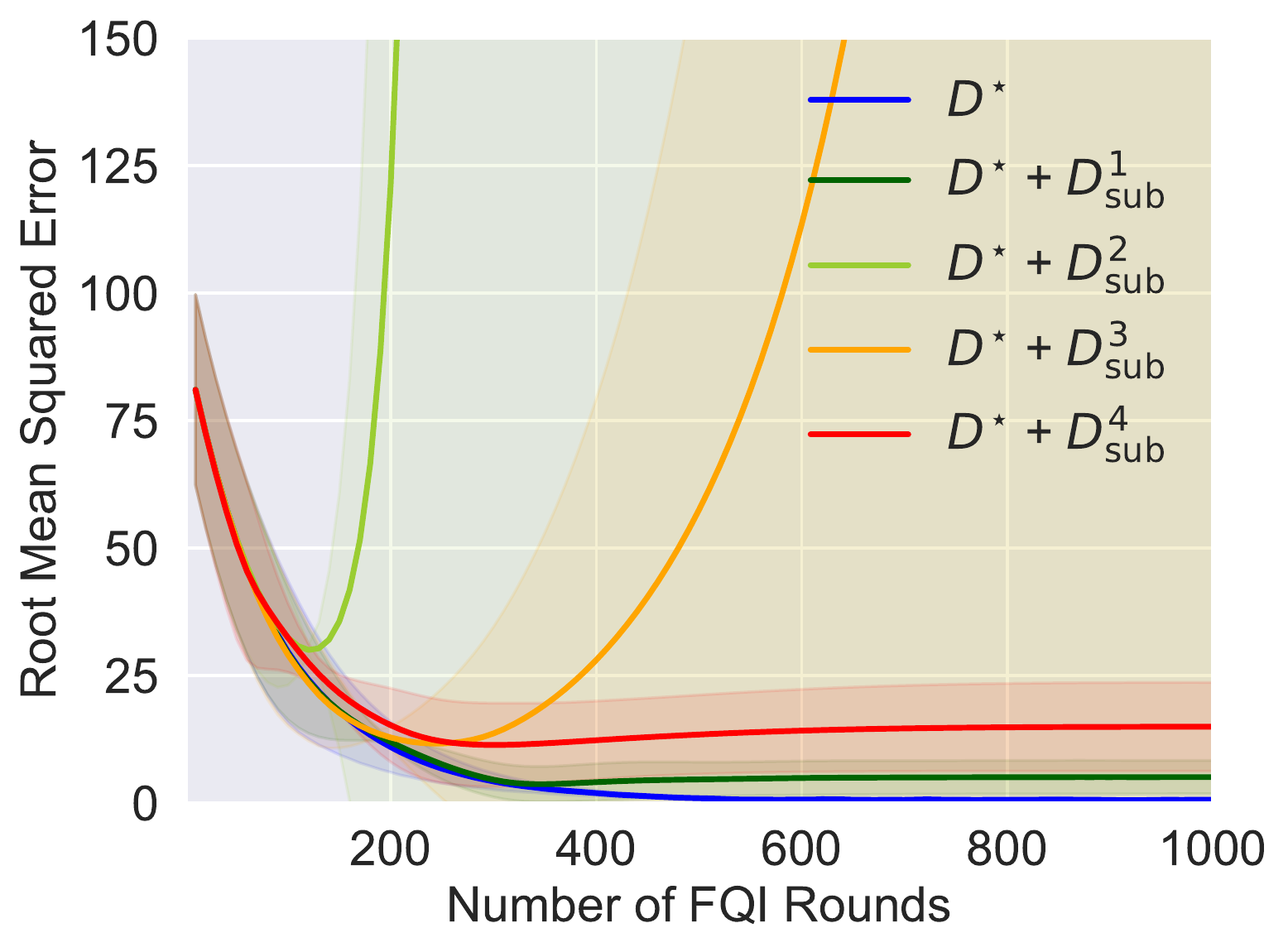}
}
\caption{Performance of FQI with features from pre-trained neural networks and datasets induced by lower performance policies.
See Figure~\ref{fig:suboptimal_add} in Section~\ref{sec:exp_results} for results on other environments. 
}
\label{fig:suboptimal}

\subfigure[Walker2d-v2]{
\includegraphics[width=0.30\textwidth]{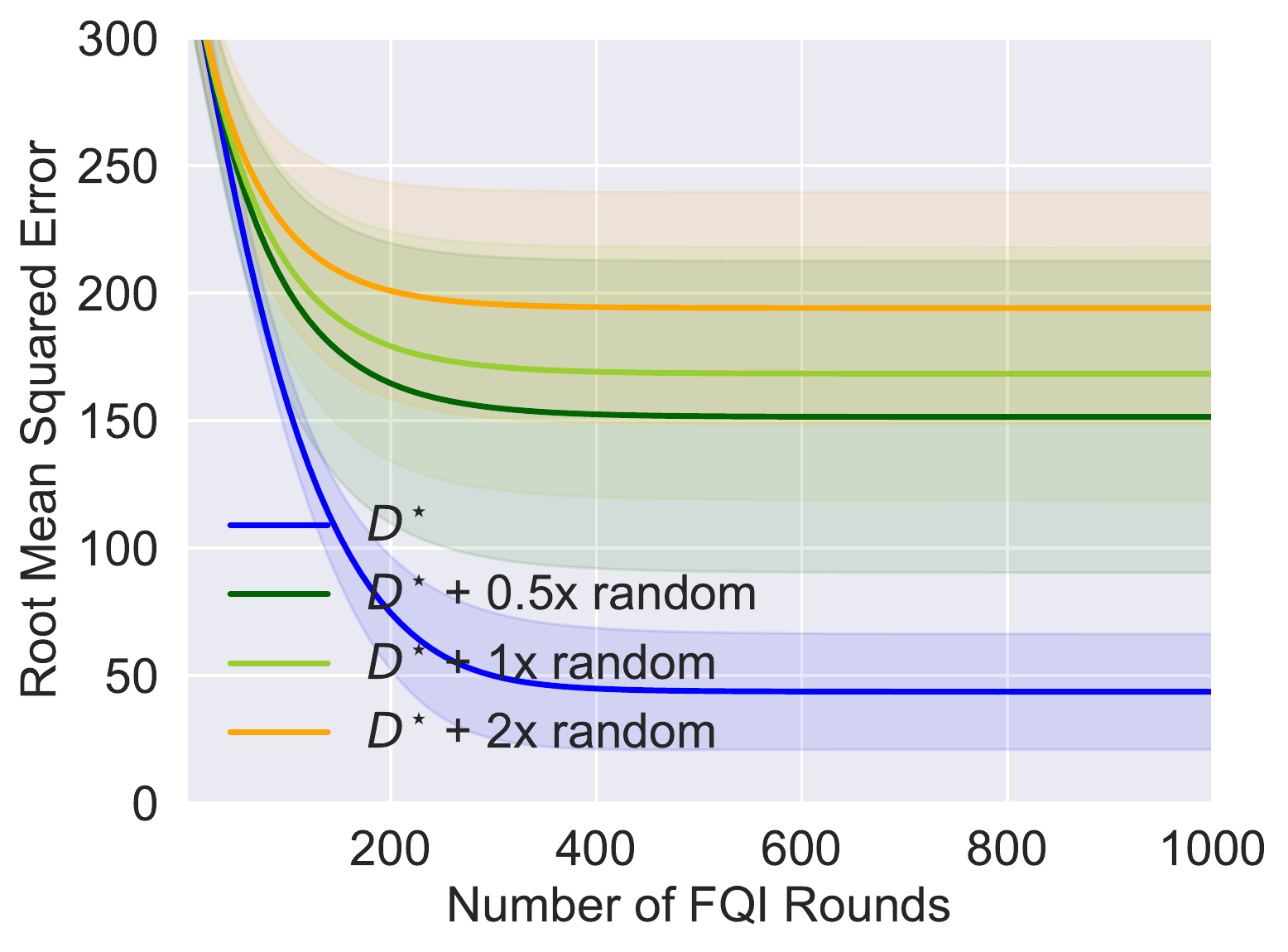}
}
\subfigure[Hopper-v2]{
\includegraphics[width=0.30\textwidth]{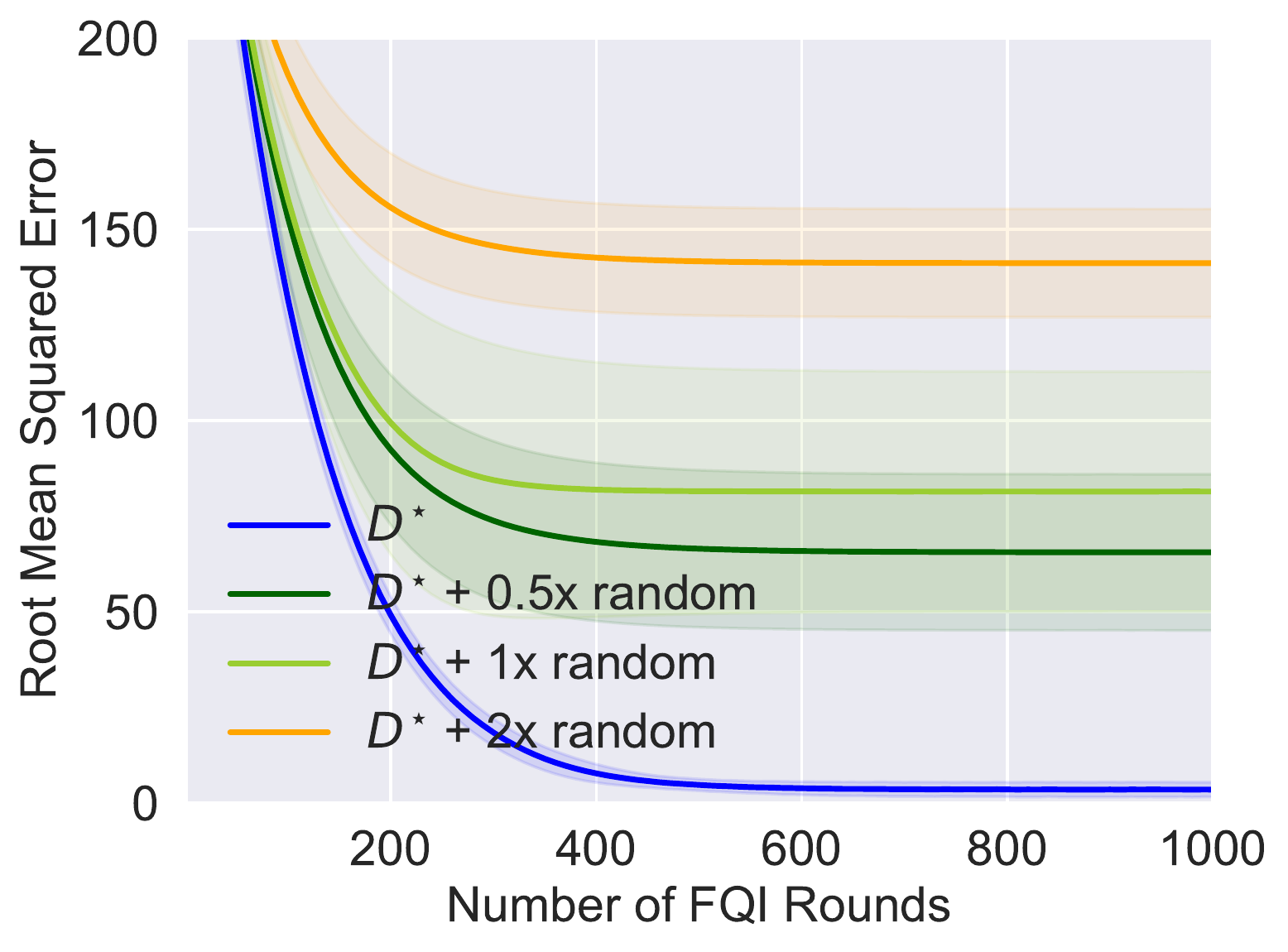}
}
\subfigure[CartPole-v0]{
\includegraphics[width=0.30\textwidth]{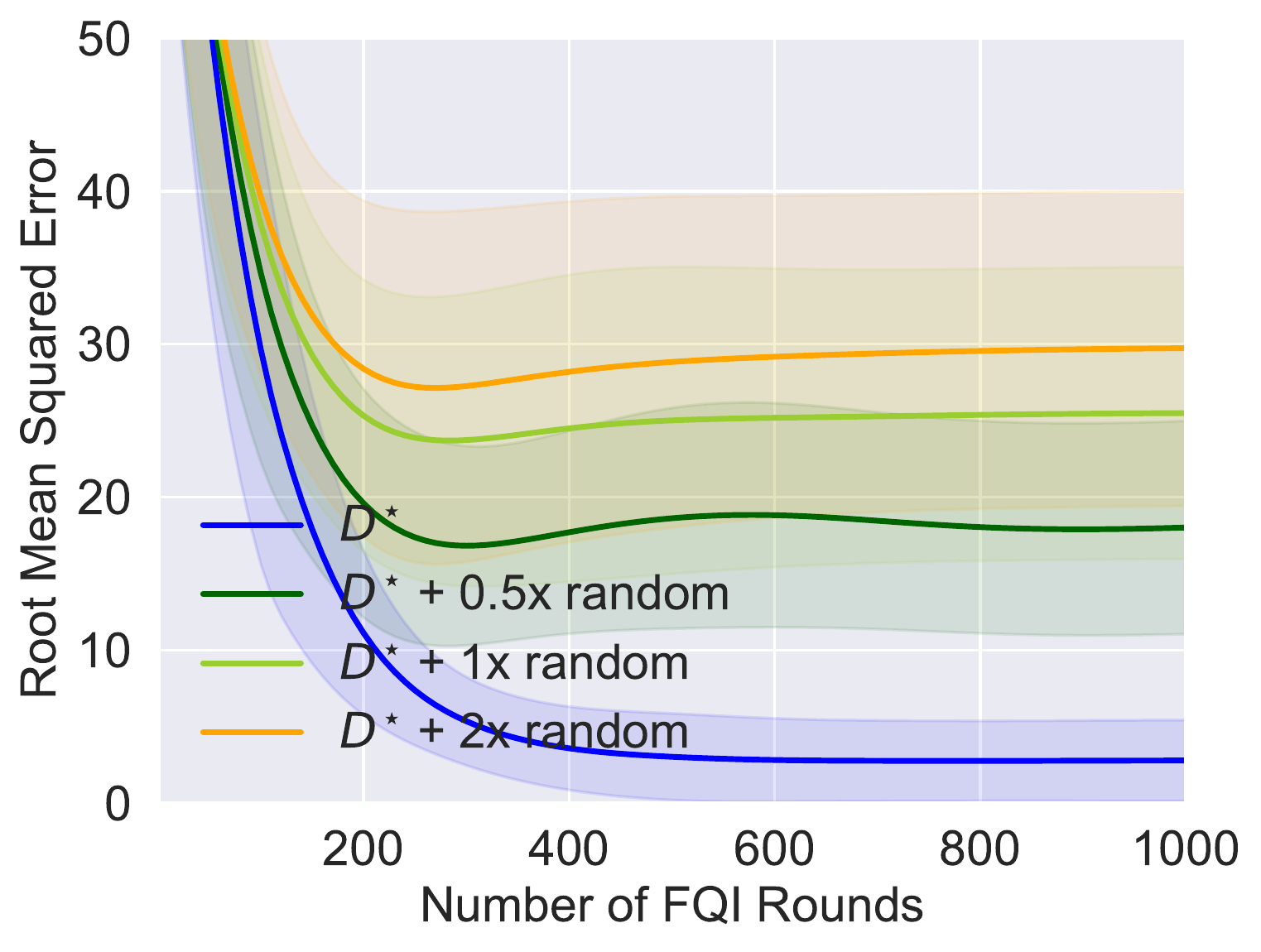}
}
\caption{Performance of FQI with random Fourier features and datasets induced by random policies.
See Figure~\ref{fig:rff_add} in Section~\ref{sec:exp_results} for results on other environments. 
}
\label{fig:rff}
\subfigure[$D^{\star}$ ]{
\includegraphics[width=0.30\textwidth]{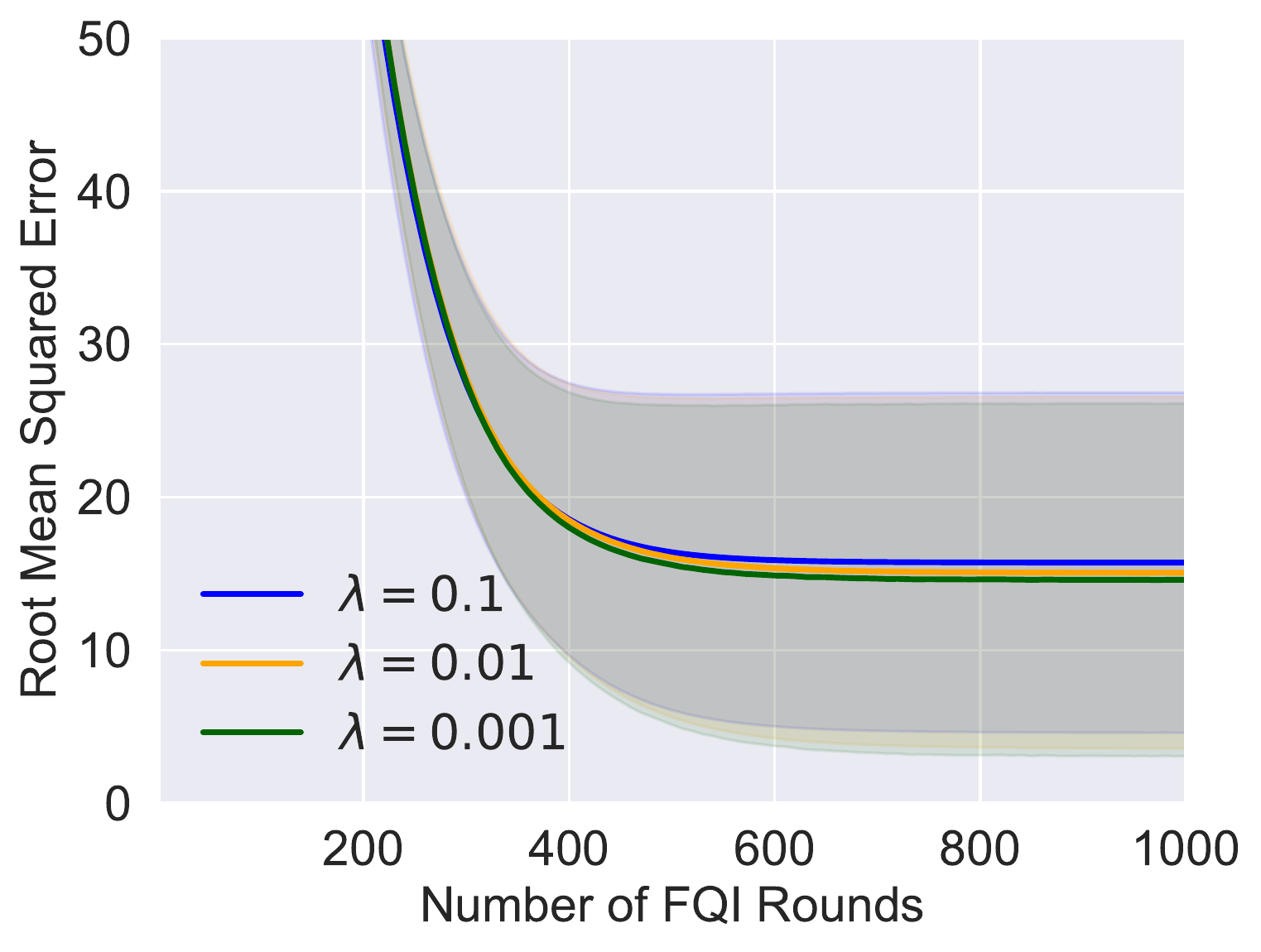}
}
\subfigure[$D^{\star}$ + 1x random]{
\includegraphics[width=0.30\textwidth]{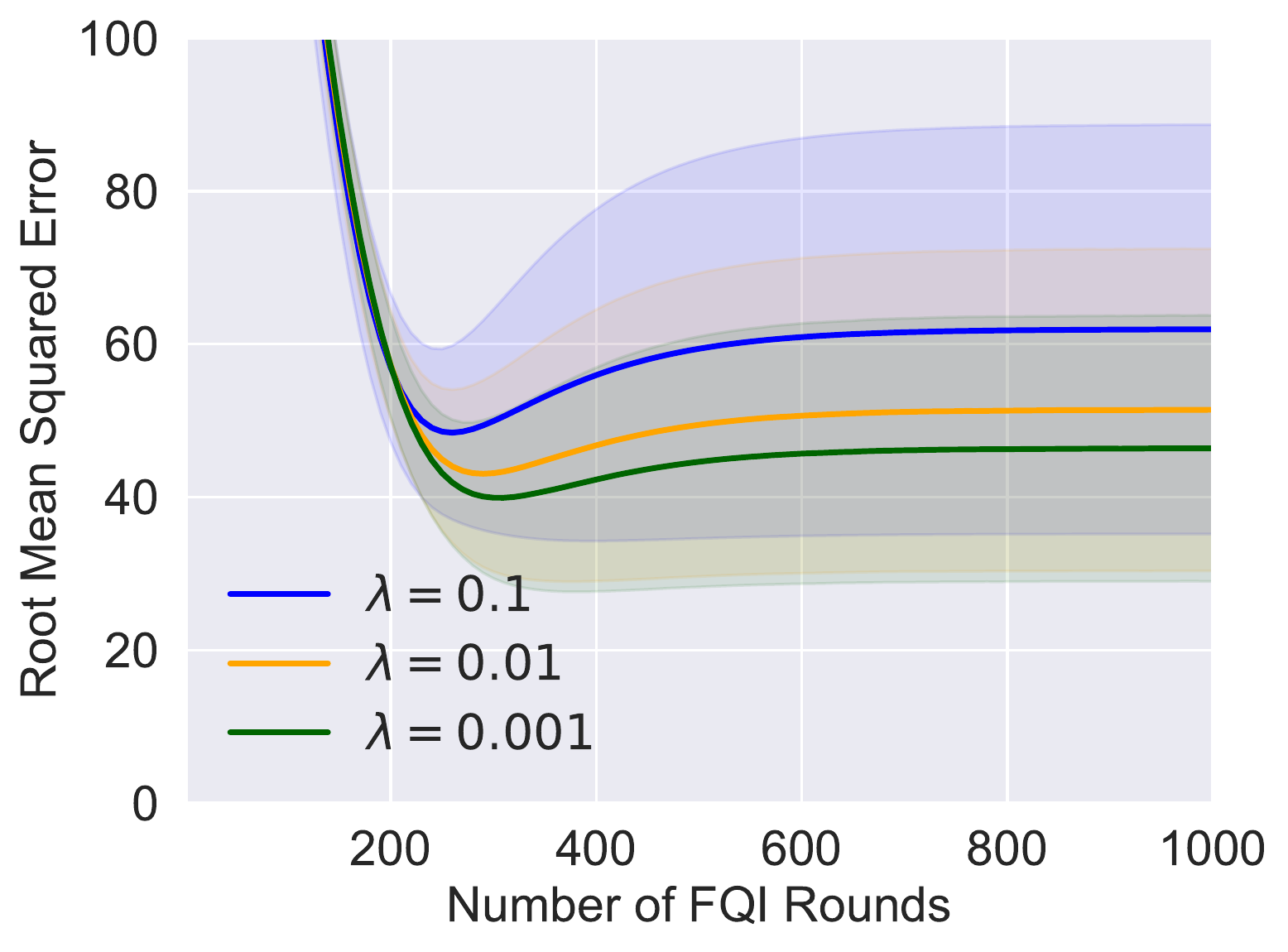}
}
\subfigure[$D^{\star}$+ 2x random]{
\includegraphics[width=0.30\textwidth]{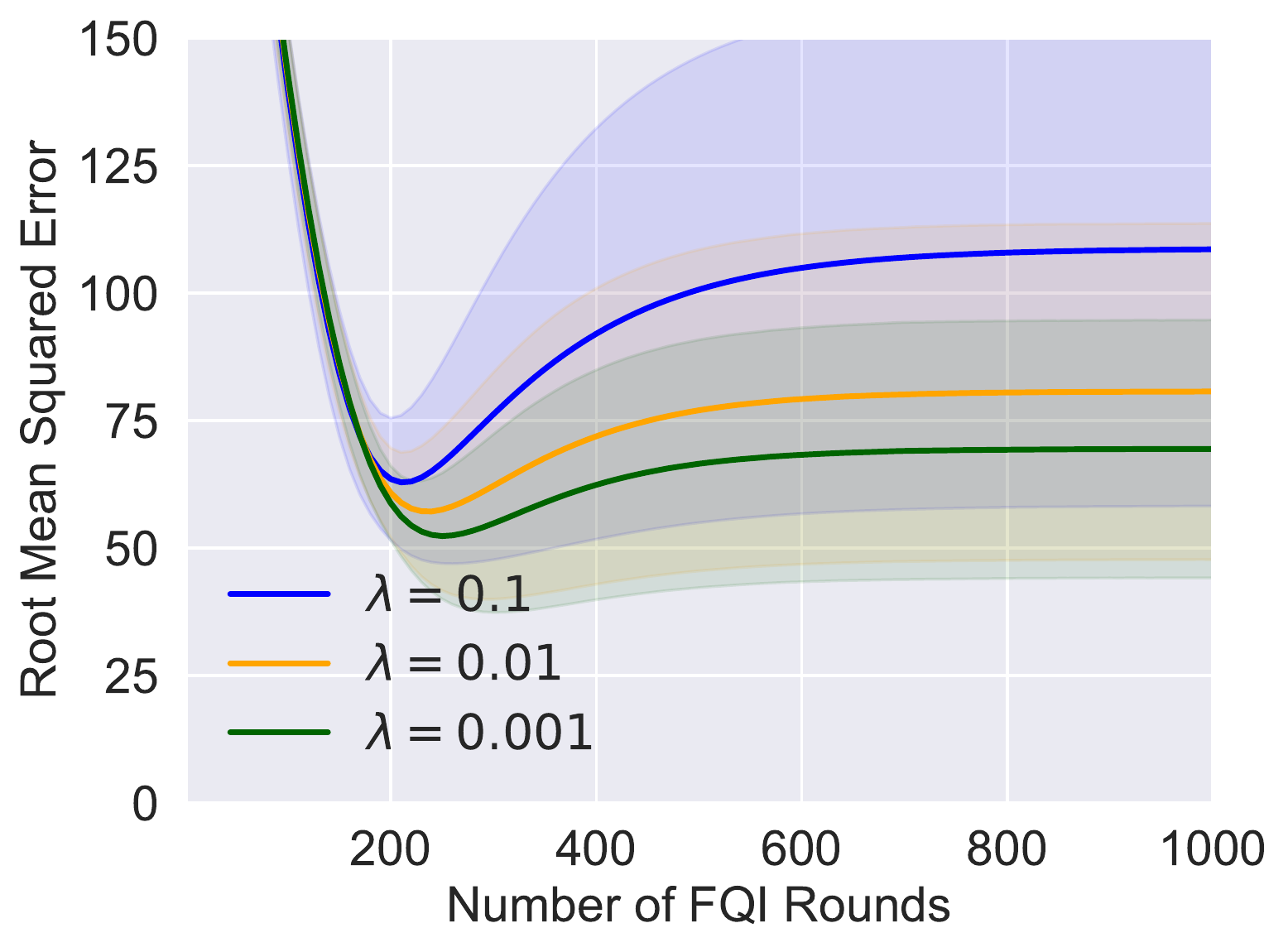}
}
\caption{Performance of FQI on Walker2d-v2, with features from pre-trained neural networks, datasets induced by random policies, and different regularization parameter $\lambda$.
See Figure~\ref{fig:ridge_ant} to Figure~\ref{fig:ridge_walker} in Section~\ref{sec:exp_results} for results on other environments. 
}
\label{fig:ridge}
\end{figure*}

\FloatBarrier
\section{Further Methodological Discussion}\label{sec:further}

We now expand on a few methodological motivations over the previous section, because these points merit further discussion.

\paragraph{Why did our methodology not compare to method $X$?}
As mentioned in Section~\ref{sec:related}, there are a number of algorithms for offline RL~\citep{fujimoto2019off, kumar2019stabilizing, wu2020behavior, jaques2020way, nachum2019algaedice, peng2019advantage,siegel2020keep, kumar2020conservative, agarwal2019striving, yu2020mopo, kidambi2020morel, rafailov2020offline}. In this work, our focus is on the more basic \emph{policy evaluation} problem, rather than \emph{policy improvement}, and because of this, we focus on FQI and LSTD due to that these other methodologies are designed for the further challenges associated with policy improvement.
Specifically, let us discuss two specific reasons why our current methodology is well motivated, in light of the broader set of neural approaches for offline policy improvement.
First, the main techniques in these prior empirical approaches largely focus on 
constraining the learned policy to be close to the behavioral policy, which is achieved by adding a penalty (by uncertainty quantification); in our setting, the target policy is given and such constraints are not needed due to that the target policy is not being updated. Furthermore, since we mainly focus on policy evaluation in this paper, it is not evident how to even impose such penalties (or constraints) for a fixed policy. Second, in this work, in order to better understand offline RL methods when combined with function approximation schemes, we decouple practical representational learning questions from the offline RL methods because this allows for us to directly examine if the given neural representation, which is sufficient to represent the value of the target policy, is sufficient for effective offline RL. By doing so, we can better understand if the error amplification effects (suggested by worst-case theoretical analysis) occur when combining pre-trained features with offline RL methods.  An interesting further question is how to couple the representational learning with offline RL (for the simpler question of policy evaluation) --- see the discussion in Section~\ref{sec:discussion}.

\paragraph{What about importance sampling approaches?}
One other approach we did not consider for policy evaluation is importance sampling~\citep{dudik2011doubly, mandel2014offline, thomas2015high, li2015toward, jiang2016doubly, thomas2016data, guo2017using, wang2017optimal, liu2018breaking, farajtabar2018more, xie2019towards, kallus2019efficiently, liu2019understanding, uehara2019minimax, kallus2020double, jiang2020minimax, feng2020accountable, yang2020off, nachum2019dualdice, zhang2020gendice, zhang2020gradientdice}.  
There is a precise sense in which importance sampling would in fact be
successful in our setting, and we view this as a point in favor of our
approach, which we now explain. First, to see that importance sampling
will be successful, note that due to the simplicity of our data
collection procedure, we have that all our experiments contain at
least $30\%$ of the trajectories collected from the target policy
itself, so if we have access to the correct importance weights, then
importance sampling can be easily seen to have low variance. However,
here, importance sampling has low variance only due to our highly
favorable (and unrealistic) scenario where our data collection has
such a high fraction of trajectories from the target policy itself. If
we consider a setting where the distribution shift is low in a
spectral sense, as in Section~\ref{sec:fqi},  but where the the data
collection does not have such a high fraction of trajectories
collected from the target policy, it is not evident how to effectively
implement the importance
sampling approach because there is no demonstrably (or provably) robust
method for combining function
approximation with importance sampling. In fact, methods which
combine importance sampling and function approximation are an
active and important area of research.

\section{Discussion and Implications}
\label{sec:discussion}

The main conclusion of this work, through extensive experiments on a
number of tasks, is that we observe substantial error
amplification, even when using pre-trained representations, even we
(unrealistically) tune hyper-parameters, regardless of what the distribution was
shifted to. Furthermore, this amplification even occurs under relatively mild distribution shift.
Our experiments complement the recent hardness results
in~\citet{wang2020statistical} showing the issue of error
amplification is a real practical concern. 

The implications of these results, both from a theoretical and an
empirical perspective, are that successful offline RL (where we seek
to go beyond the constraining, low distribution shift regime) requires substantially
stronger conditions beyond those which suffice for successful
supervised learning. These results also
raise a number of concerns about empirical practices employed in a
number of benchmarks. We now discuss these two points further.

\paragraph{Representation Learning in Offline RL.}
Our experiments demonstrate that the definition of a good
representation in offline RL is more subtle than in supervised
learning, since features extracted from pre-trained neural networks
are usually extremely effective in supervised learning.  Certainly,
features extracted from pre-trained neural networks and random
features satisfy the realizability assumption
(Assumption~\ref{assmp:realizability}) approximately.  However, from
our empirical findings, these features do not seem to satisfy strong
representation conditions
(e.g. Assumption~\ref{assumption:completeness}) that permits
sample-efficient offline RL.  This suggests that better representation
learning process (feature learning methods that differs from those
used in supervised learning) could be a route for achieving better
performance in offline RL.

\paragraph{Implications for Empirical Practices and Benchmarks.}
Our empirical findings suggests a closer inspection of certain 
empirical practices used in the evaluation of offline RL
algorithms.

\begin{itemize}
\item \textbf{Offline Data Collection.} Many empirical settings
  create an offline dataset under a distribution
  which contains a large fraction from the target
  policy itself (e.g. creating the dataset using an online RL
  algorithm). This may substantially limit the methodology to only
  testing algorithms in a low distribution shift regime; our
  results suggests this may not be reflective
  of what would occur with more realistic and diverse datasets.
\item \textbf{Hyperparameter Tuning in Offline RL.}
A number of methodologies tune hyperparameters using interactive
access to the environment, a practice that is clearly not possible with the
given offline dataset (e.g. see ~\citep{paine2020hyperparameter} for
further discussion). The instability of hyperparameter tuning, as
observed in our experiments, suggests that hyperparameter tuning in
offline RL may be a substantial hurdle.
\end{itemize}

Finally, we should remark that the broader motivation of our results
(and this discussion)
is to help with advancing
the field of offline RL through better linking our theoretical
understanding with the empirical practices. It is also worth noting
that there are notable empirical
successes in more realistic settings, e.g.~\citep{mandel2014offline,
  youtube}.

\section*{Acknowledgements}
The authors would like to thank Alekh Agarwal, Akshay Krishnamurthy,
Aditya Kusupati, and Nan Jiang for helpful discussions.
Sham M. Kakade acknowledges funding from the ONR award N00014-18-1-2247. 
Ruosong Wang and Ruslan Salakhutdinov are supported in part by NSF IIS1763562, AFRL CogDeCON FA875018C0014, and DARPA SAGAMORE HR00111990016.
Part of this work was done while Ruosong Wang was visiting the Simons Institute for the Theory of Computing.

\bibliography{bib}
\bibliographystyle{abbrvnat}
\appendix
\section{Omitted Proofs}\label{sec:proof}
\subsection{Proof of Lemma~\ref{lem:upper_main}}
Clearly, 
\begin{align*}
\hat{\theta}_t
 &= \hat{\Lambda}^{-1} \left(\frac{1}{N}\sum_{i = 1}^{N} \phi(s_i, a_i) \cdot(r_i + \gamma \hat{V}_{t - 1}(\overline{s}_i))\right)\\
 &= \hat{\Lambda}^{-1} \left(\frac{1}{N}\sum_{i = 1}^{N} \phi(s_i, a_i) \cdot(r_i + \gamma \hat{Q}_{t - 1}(\overline{s}_i, \pi(\overline{s}_i)))\right)\\ 
&= \hat{\Lambda}^{-1} \left(\frac{1}{N}\sum_{i = 1}^{N} \phi(s_i, a_i) \cdot(r_i + \gamma \phi(\overline{s}_i, \pi(\overline{s}_i))^{\top} \hat{\theta}_{t - 1})\right)\\
&= \hat{\Lambda}^{-1} \left(
\frac{1}{N}\sum_{i = 1}^{N} \phi(s_i, a_i) \cdot(r_i + \gamma \phi(\overline{s}_i, \pi(\overline{s}_i))^{\top} \theta^*))
+ 
\frac{1}{N} \sum_{i = 1}^{N} \phi(s_i, a_i) \cdot \gamma\phi(\overline{s}_i, \pi(\overline{s}_i))^{\top} (\hat{\theta}_{t - 1} - \theta^*)
)\right)\\
&= \hat{\Lambda}^{-1} \left(
\frac{1}{N}\sum_{i = 1}^{N} \phi(s_i, a_i) \cdot(r_i + \gamma \phi(\overline{s}_i, \pi(\overline{s}_i))^{\top} \theta^*))\right)
+ 
\gamma \hat{\Lambda}^{-1} \left(
\frac{1}{N}\sum_{i = 1}^{N} \phi(s_i, a_i) \cdot \phi(\overline{s}_i, \pi(\overline{s}_i))^{\top} (\hat{\theta}_{t - 1} - \theta^*)
)\right).
 \end{align*}
 For the first term, we have
 \begin{align*}
 &  \hat{\Lambda}^{-1} \left(\frac{1}{N}\sum_{i = 1}^{N} \phi(s_i, a_i) \cdot(r_i + \gamma \phi(\overline{s}_i, \pi(\overline{s}_i))^{\top} \theta^*))\right) \\
  = &\hat{\Lambda}^{-1}  \left(\frac{1}{N}\sum_{i = 1}^{N} \phi(s_i, a_i) \cdot(r_i + \gamma Q^{\pi}(\overline{s}_i, \pi(\overline{s}_i)))\right) \\
 = &\hat{\Lambda}^{-1}  \left(\frac{1}{N}\sum_{i = 1}^{N} \phi(s_i, a_i) \cdot(r_i + \gamma V^{\pi}(\overline{s}_i))\right) \\
 =& \hat{\Lambda}^{-1}  \left(\frac{1}{N}\sum_{i = 1}^{N} \phi(s_i, a_i) \cdot(Q^{\pi}(s_i, a_i) + \gamma \zeta_i)\right) \\
 =&\gamma \hat{\Lambda}^{-1} \cdot \frac{1}{N}  \sum_{i = 1}^{N} \phi(s_i, a_i) \cdot \zeta_i +  \hat{\Lambda}^{-1} \cdot \frac{1}{N}\sum_{i = 1}^{N} \phi(s_i, a_i) \cdot  \phi(s_i, a_i)^{\top} \cdot \theta^*\\
 =&\gamma \hat{\Lambda}^{-1} \cdot \frac{1}{N} \sum_{i = 1}^{N} \phi(s_i, a_i) \cdot \zeta_i +  \hat{\Lambda}^{-1} \cdot \frac{1}{N}(\hat{\Phi}^{\top}\hat{\Phi}) \cdot \theta^*\\
 = &\frac{\gamma}{N}   \hat{\Lambda}^{-1} \hat{\Phi}^{\top} \zeta+ \theta^* - \lambda \hat{\Lambda}^{-1}\theta^*.
 \end{align*}
 Therefore,
\begin{align*}
\hat{\theta}_T - \theta^*
&= \left(\frac{\gamma}{N} \hat{\Lambda}^{-1} \hat{\Phi}^{\top} \zeta - \lambda \hat{\Lambda}^{-1}\theta^*\right)+ \gamma \hat{\Lambda}^{-1} \frac{\Phi^{\top}\overline{\Phi}}{N} \cdot (\hat{\theta}_{T - 1} - \theta^*)\\
&=\left(\frac{\gamma}{N} \hat{\Lambda}^{-1} \hat{\Phi}^{\top} \zeta - \lambda \hat{\Lambda}^{-1}\theta^*\right)
+ \gamma \hat{\Lambda}^{-1}\frac{\Phi^{\top}\overline{\Phi}}{N} \left(\frac{\gamma}{N} \hat{\Lambda}^{-1} \hat{\Phi}^{\top} \zeta - \lambda \hat{\Lambda}^{-1}\theta^*\right)\\
&+ \left( \gamma \hat{\Lambda}^{-1}\frac{\Phi^{\top}\overline{\Phi}}{N} \right)^2(\hat{\theta}_{T - 2} - \theta^*)\\
& = \ldots\\
& = \sum_{t = 1}^T \left(\gamma \hat{\Lambda}^{-1}\frac{\Phi^{\top}\overline{\Phi}}{N}\right)^{t - 1} \cdot  \left(\frac{\gamma}{N} \hat{\Lambda}^{-1}\Phi^{\top} \zeta - \lambda \hat{\Lambda}^{-1}\theta^*\right) + \left(\gamma \hat{\Lambda}^{-1}\frac{\Phi^{\top}\overline{\Phi}}{N}\right)^T  \theta^*.
\end{align*}

\subsection{Proof of Lemma~\ref{lem:fqi_completeness}}
\newcommand{\cross}{\Lambda_{\mathrm{cross}}}
In this proof, we set the number of rounds $T = C_T \log(d / (\varepsilon(1 - \gamma))) / (1 - \gamma)$ for a sufficiently large constant $C_T$.
We set $\lambda$ so that $\lambda \le \varepsilon(1 - \gamma)\sigma_{\min}(\Lambda)/(8T^2\sqrt{d})$ and $\lambda \le \sigma_{\min}^3(\Lambda) / (20T)$.
We also set $N$ so that $N \ge C_NT^4d \log(1 / \delta)/(\varepsilon^2 \sigma_{\min}(\Lambda) (1 - \gamma)^2)$ and $N \ge C_N T^2d^2 \log(d / \delta) /  \sigma_{\min}^6(\Lambda)$ for a sufficiently large constant $C_N$. 
We use $L$ to denote $\hat{\Lambda}^{-1}\Phi^{\top}\overline{\Phi} / N$.
We use $\cross$ to denote \[\expect_{(s, a) \sim \mu, s' \sim P(\cdot \mid s, a)}[\phi(s, a)\phi(s', \pi(s'))^{\top}].\]
We use $\Phi_{\mathrm{all}} \in \mathbb{R}^{|\states| |\actions| \times d}$ to denote a matrix whose row indexed by $(s, a) \in \states \times \actions$ is $\phi(s, a) \in \mathbb{R}^d$.
We use $D^{\mu} \in \mathbb{R}^{|\states| |\actions| \times |\states| |\actions|}$ to denote a diagonal matrix whose diagonal entry indexed by $(s, a)$ is $\mu(s, a)$. 
We use $P^{\pi} \in \mathbb{R}^{|\states| |\actions| \times |\states| |\actions|}$ to denote a matrix where
\[
P^{\pi}((s, a), (s', a')) = \begin{cases}
P(s' \mid s, a) & a' = \pi(s)\\
0 & \text{otherwise} 
\end{cases}.
\]

Our first lemma shows that under Assumption~\ref{assumption:completeness}, $\Lambda^{-1} \cross$ satisfies certain ``non-expansiveness'' property. 

\begin{lemma}\label{lem:non_expansive}
Under Assumption~\ref{assumption:completeness}, for any integer $t \ge 0$, $x \in \mathbb{R}^d$ and $(s, a) \in \states \times \actions$, \[|\phi(s, a)(\Lambda^{-1} \cross)^t x| \le \|x\|_2.\] Moreover, 
\[\|(\Lambda^{-1} \cross)^t\|_2 \le \sigma_{\min}^{-1/2}(\Lambda).\]
\end{lemma}
\begin{proof}
Clearly, 
\[
\Lambda^{-1} \cross = (\Phi_{\mathrm{all}}^{\top} D^{\mu} \Phi_{\mathrm{all}}^{\top}) \Phi_{\mathrm{all}}^{\top} D^{\mu} P^{\pi} \Phi_{\mathrm{all}}.
\]
For any $x \in \mathbb{R^d}$ with $\|x\|_2 = 1$, $(s, a) \in \states \times \actions$, and integer $t \ge 0$, we have
\[
\Phi_{\mathrm{all}} (\Lambda^{-1} \cross)^t x = \Phi_{\mathrm{all}}((\Phi_{\mathrm{all}}^{\top} D^{\mu} \Phi_{\mathrm{all}}^{\top})^{-1} \Phi_{\mathrm{all}}^{\top} D^{\mu} P^{\pi} \Phi_{\mathrm{all}})^t x.
\]
By Assumption~\ref{assumption:completeness}, there exists $x^{(1)}, x^{(2)}, \ldots, x^{(t)} \in \mathbb{R}^d$ such that 
\[
P^{\pi} \Phi_{\mathrm{all}} x = \Phi_{\mathrm{all}} x^{(1)}, P^{\pi} \Phi_{\mathrm{all}} x^{(1)} = \Phi_{\mathrm{all}} x^{(2)}, \ldots, P^{\pi} \Phi_{\mathrm{all}} x^{(t - 1)} = \Phi_{\mathrm{all}} x^{(t)}.
\]
Therefore,
\begin{align*}
\Phi_{\mathrm{all}} (\Lambda^{-1} \cross)^t x =& \Phi_{\mathrm{all}}((\Phi_{\mathrm{all}}^{\top} D^{\mu} \Phi_{\mathrm{all}}^{\top})^{-1} \Phi_{\mathrm{all}}^{\top} D^{\mu} P^{\pi} \Phi_{\mathrm{all}})^{t-1} (\Phi_{\mathrm{all}}^{\top} D^{\mu} \Phi_{\mathrm{all}}^{\top})^{-1} \Phi_{\mathrm{all}}^{\top} D^{\mu} P^{\pi} \Phi_{\mathrm{all}} x\\
= & \Phi_{\mathrm{all}}((\Phi_{\mathrm{all}}^{\top} D^{\mu} \Phi_{\mathrm{all}}^{\top})^{-1} \Phi_{\mathrm{all}}^{\top} D^{\mu} P^{\pi} \Phi_{\mathrm{all}})^{t-1}  x^{(1)} \\
= & \Phi_{\mathrm{all}}((\Phi_{\mathrm{all}}^{\top} D^{\mu} \Phi_{\mathrm{all}}^{\top})^{-1} \Phi_{\mathrm{all}}^{\top} D^{\mu} P^{\pi} \Phi_{\mathrm{all}})^{t-2}  x^{(2)} \\
= & \ldots\\
= & \Phi_{\mathrm{all}} x^{(t)}.
\end{align*}
On the other hand, for any $u \in \mathbb{R}^{|\states||\actions|}$, $\|P^{\pi}u\|_{\infty} \le \|u\|_{\infty}$,
which implies
\begin{align*}
&\|\Phi_{\mathrm{all}} (\Lambda^{-1} \cross)^t x\|_{\infty} = \|\Phi_{\mathrm{all}} x^{(t)}\|_{\infty} = \|P^{\pi} \Phi_{\mathrm{all}} x^{(t - 1)}\|_{\infty} \le \|\Phi_{\mathrm{all}} x^{(t - 1)}\|_{\infty}  \\
= &\|P^{\pi} \Phi_{\mathrm{all}} x^{(t - 2)}\|_{\infty} 
\le \ldots \le \|\Phi_{\mathrm{all}} x\|_{\infty}. 
\end{align*}
Therefore, for any $(s, a) \in \states \times \actions$,
\begin{equation}\label{eqn:norm_feature}
|\phi(s, a) (\Lambda^{-1} \cross)^t x| \le \|\Phi_{\mathrm{all}} x\|_{\infty} \le \|x\|_2 = 1
\end{equation}
since $\|\phi(s, a)\|_2 \le 1$ for all $(s, a) \in \states \times \actions$. 

Now we show that for any $y \in \mathbb{R}^d$ with $\|y\|_2 = 1$, we have
\[
|y^{\top}(\Lambda^{-1} \cross)^t x|  \le \sigma_{\min}^{-1/2}(\Lambda).
\]
Note that this implies $\|(\Lambda^{-1} \cross)^t x \|_2 \le \sigma_{\min}^{-1/2}(\Lambda)$ for all $x \in \mathbb{R}^d$ with $\|x\|_2 = 1$,
and therefore $\|(\Lambda^{-1} \cross)^t\|_2 \le \sigma_{\min}^{-1/2}(\Lambda)$.

Suppose for the sake of contradiction that there exists $y \in \mathbb{R}^d$ with $\|y\|_2 = 1$ such that
\[
|y^{\top}(\Lambda^{-1} \cross)^t x|   > \sigma_{\min}^{-1/2}(\Lambda).
\]
By Cauchy–Schwarz inequality, this implies
\[
\|(\Lambda^{-1} \cross)^t x\|_2 > \sigma_{\min}^{-1/2}(\Lambda).
\]
Moreover, by Equation~\eqref{eqn:norm_feature} and H\"older's inequality
\[
\|(\Lambda^{-1} \cross)^t x\|_{\Lambda}^2 = \sum_{(s, a) \in \states \times \actions} \mu(s, a) (\phi(s, a)^{\top} (\Lambda^{-1} \cross)^t x)^2 \le 1,
\]
which implies
\[
\sigma_{\min}(\Lambda) \le 1 / \|(\Lambda^{-1} \cross)^t x\|_2^2.
\]
The fact that $\|(\Lambda^{-1} \cross)^t x\|_2^2 > \sigma_{\min}^{-1}(\Lambda)$ leads to a contradiction. 
\end{proof}

Now we show that by taking the number of samples $N$ polynomially large, the matrix $L = \hat{\Lambda}^{-1}\Phi^{\top}\overline{\Phi} / N$ concentrates around $\Lambda^{-1} \cross$.

\begin{lemma}\label{lem:perturbation}
When $N \ge C_1 T^2d^2 \log(d / \delta) /  \sigma_{\min}^6(\Lambda) $ for some constant $C_1 > 0$, with probability $1 - \delta / 2$, $\|L - \Lambda^{-1} \cross\|_2 \le \sigma_{\min}(\Lambda) / T$.
\end{lemma}
\begin{proof}
We show that with probability $1 - \delta / 2$, $\|\hat{\Lambda} - \Lambda\|_2  \le \sigma_{\min}^3(\Lambda) / (20T)$ and $\|\Phi^{\top}\overline{\Phi} / N - \cross\|_2  \le \sigma_{\min}^3(\Lambda) / (20T)$.
For each index $(i, j) \in \{1, 2, \ldots, d\} \times \{1, 2, \ldots, d\}$, by Chernoff bound, with probability at least $1 - \delta / (2d^2)$, we have
\[
|(\Phi^{\top}\Phi/N)_{i, j} - \Lambda_{i, j}| \le \sigma_{\min}^3(\Lambda) / (20Td)
\]
and 
\[
|(\Phi^{\top}\overline{\Phi}/N)_{i, j} - (\cross)_{i, j}| \le \sigma_{\min}^3(\Lambda) / (20Td).
\]
By union bound and the fact that the operator norm of a matrix is upper bounded its Frobenius norm, with probability $1 - \delta / 2$, 
we have
\[
\|\Phi^{\top}\Phi/N- \Lambda\|_2  \le \sigma_{\min}^3(\Lambda) / (20T)
\]
and
\[
\|\Phi^{\top}\overline{\Phi} / N - \cross\|_2  \le \sigma_{\min}^3(\Lambda) / (20T).
\]
Since $\lambda \le \sigma_{\min}^3(\Lambda) / (20T)$,
\[
\|\hat{\Lambda}- \Lambda\|_2 \le \sigma_{\min}^3(\Lambda) / (10T).
\]
Note that
\[
\hat{\Lambda}^{-1} 
= (\Lambda + (\hat{\Lambda} - \Lambda))^{-1} 
= (\Lambda(I + \Lambda^{-1}(\hat{\Lambda} - \Lambda)))^{-1} 
=  (I + \Lambda^{-1} (\hat{\Lambda} - \Lambda))^{-1}\Lambda^{-1} .
\]
By Neumann series,
\[
(I + \Lambda^{-1} (\hat{\Lambda} - \Lambda))^{-1} = I - \Lambda^{-1} (\hat{\Lambda} - \Lambda) + (\Lambda^{-1} (\hat{\Lambda} - \Lambda))^2 + \ldots,
\]
and therefore
\[
\|I - (I + \Lambda^{-1} (\hat{\Lambda} - \Lambda))^{-1}\|_2 \le \sigma_{\min}^{-1} \cdot \sigma_{\min}^3(\Lambda) / (5T) =\sigma_{\min}^2(\Lambda)/ (5T)  .
\]
Hence,
\[
\|\hat{\Lambda}^{-1} - \Lambda^{-1}\|_2 = \|  ((I + \Lambda^{-1} (\hat{\Lambda} - \Lambda))^{-1} - I) \Lambda^{-1}\| \le  \sigma_{\min}^2(\Lambda) / (5T) \cdot 1 / \sigma_{\min}  =  \sigma_{\min} (\Lambda) / (5T).
\]
Finally,
\begin{align*}
\|L - \Lambda^{-1} \cross\|_2  &= \| \hat{\Lambda}^{-1}\Phi^{\top}\overline{\Phi} / N - \Lambda^{-1} \cross\|_2 \\
& \le \|(\hat{\Lambda}^{-1} - \Lambda^{-1})\cross\|_2 + \|\Lambda^{-1} (\Phi^{\top}\overline{\Phi} / N - \cross)\|_2 + \|(\hat{\Lambda}^{-1} - \Lambda^{-1}) (\Phi^{\top}\overline{\Phi} / N - \cross)\|_2\\
& \le \sigma_{\min} (\Lambda) / T.
\end{align*}
\end{proof}
Now we use Lemma~\ref{lem:non_expansive} and Lemma~\ref{lem:perturbation} to prove the following lemma.
\begin{lemma}\label{lem:non_expansive_final}
Conditioned on the event in Lemma~\ref{lem:perturbation}, for any integer $0 \le t \le T$, $x \in \mathbb{R}^d$ and $(s, a) \in \states \times \actions$, $|\phi(s, a)^{\top} L^t x| \le (t + 1) \|x\|_2$. 
\end{lemma}
\begin{proof}
Clearly, $L^t = (\Lambda^{-1} \cross + (L - \Lambda^{-1} \cross))^t$, and we apply binomial expansion on \[\phi(s, a)^{\top}L^t x = \phi(s, a)^{\top}(\Lambda^{-1} \cross + (L - \Lambda^{-1} \cross))^t x.\]
As a result, for each $i \in \{0, 1, 2, \ldots, t\}$, there are $\binom{t}{i}$ terms in \[\phi(s, a)^{\top}(\Lambda^{-1} \cross + (L - \Lambda^{-1} \cross))^t x\] where $\Lambda^{-1} \cross$ appears for $i$ times and $L - \Lambda^{-1} \cross$ appears for $t - i$ times.

When $i = 0$, by Lemma~\ref{lem:non_expansive}, \[|\phi(s, a) (\Lambda^{-1} \cross)^t x| \le \|x\|_2.\]
When $0 < i \le t$, by Lemma~\ref{lem:perturbation} and Lemma~\ref{lem:non_expansive}, each of the $\binom{t}{i}$ terms (where $\Lambda^{-1} \cross$ appears for $i$ times and $L - \Lambda^{-1} \cross$ appears for $t - i$ times) can be bounded by $T^{-i}\|x\|_2$, and therefore their summation can be bounded by $\binom{t}{i} \cdot T^{-i}\|x\|_2 \le \|x\|_2$. 
Summing up over all $i \in \{0, 1, 2, \ldots, t\}$ finishes the proof. 
\end{proof}
\begin{lemma}\label{lem:ridge}
By taking $N \ge C_1 T^2d^2 \log(d / \delta) /  \sigma_{\min}^6(\Lambda)$ and $N \ge C_2T^4d \log(1 / \delta)/(\varepsilon^2 \sigma_{\min}(\Lambda) (1 - \gamma)^2)$ for some constants $C_1>0$ and $C_2 > 0$, with probability $1 - \delta / 2$, 
\[
\|\gamma \hat{\Lambda}^{-1}\Phi^{\top} \zeta / N- \lambda \hat{\Lambda}^{-1}\theta^*\|_2 \le  \varepsilon / (2T^2). 
\]
\end{lemma}
\begin{proof}
According to the argument in the proof of Lemma~\ref{lem:perturbation}, with probability $1 - \delta / 4$, $\|\hat{\Lambda}^{-1}\|_2 \le 2 / \sigma_{\min}(\Lambda)$.
Clearly,
\[
\|\gamma \hat{\Lambda}^{-1}\Phi^{\top} \zeta / N\|_2^2 \le \| \hat{\Lambda}^{-1}\Phi^{\top} \zeta / N\|_2^2 \le \frac{1}{N} \|\hat{\Lambda}^{-1}\|_2 \cdot \|\zeta\|_{\Phi \hat{\Lambda}^{-1}\Phi^{\top} /N}^2.
\]
According to~\citep{hsu2012tail}, with probability $1 - \delta / 4$, there exists a constant $C_3 > 0$ such that 
\[
\|\zeta\|_{\Phi \hat{\Lambda}^{-1}\Phi^{\top} /N}^2 \le C_3 d \log(1 / \delta) / (1 - \gamma)^2. 
\]
Therefore, conditioned on the two events defined above, 
\[
\|\gamma \hat{\Lambda}^{-1}\Phi^{\top} \zeta / N\|_2 \le  \left( \frac{1}{N} \|\hat{\Lambda}^{-1}\|_2 \cdot \|\zeta\|_{\Phi \hat{\Lambda}^{-1}\Phi^{\top} /N}^2\right)^{1/2} \le \varepsilon / (4T^2).
\]
Moreover, conditioned on the event above, since $\|\theta^*\|_2 \le \sqrt{d} / (1 - \gamma)$ and $\lambda \le \varepsilon(1 - \gamma)\sigma_{\min}(\Lambda)/(8T^2\sqrt{d})$, we also have
\[
\|\lambda \hat{\Lambda}^{-1}\theta^*\|_2 \le  \varepsilon / 4T^2. 
\]
We finish the proof by applying the triangle inequality. 
\end{proof}

Now we are ready to prove Lemma~\ref{lem:fqi_completeness}. 
\begin{proof}[Proof of Lemma~\ref{lem:fqi_completeness}]
According to Lemma~\ref{lem:upper_main}, we only need to prove that for any $(s, a) \in \states \times \actions$, 
\[
\left|\phi(s,a)^{\top} \left(\sum_{t = 1}^T \left(\gamma \hat{\Lambda}^{-1}\frac{\Phi^{\top}\overline{\Phi}}{N}\right)^{t - 1} \cdot  \left(\frac{\gamma}{N} \hat{\Lambda}^{-1}\Phi^{\top} \zeta - \lambda \hat{\Lambda}^{-1}\theta^*\right) + \left(\gamma \hat{\Lambda}^{-1}\frac{\Phi^{\top}\overline{\Phi}}{N}\right)^T  \theta^* \right) \right|\le \varepsilon. 
\]
Conditioned on the events in Lemma~\ref{lem:perturbation} and Lemma~\ref{lem:ridge}, by Lemma~\ref{lem:non_expansive_final} and Lemma~\ref{lem:ridge}, 
for any $(s, a) \in \states \times \actions$, 
\[
\left|\phi(s,a)^{\top} \left(\sum_{t = 1}^T \left(\gamma \hat{\Lambda}^{-1}\frac{\Phi^{\top}\overline{\Phi}}{N}\right)^{t - 1} \cdot  \left(\frac{\gamma}{N} \hat{\Lambda}^{-1}\Phi^{\top} \zeta - \lambda \hat{\Lambda}^{-1}\theta^*\right) \right) \right|\le \sum_{t = 1}^T t \cdot \|\gamma \hat{\Lambda}^{-1}\Phi^{\top} \zeta / N- \lambda \hat{\Lambda}^{-1}\theta^*\|_2 \le \varepsilon / 2. 
\]
Moreover, by taking $T = C_T \log(d / (\varepsilon(1 - \gamma))) / (1 - \gamma)$ for some constant $C_T > 0$, for any $(s, a) \in \states \times \actions$, we have
\[
\left|\phi(s,a)^{\top}  \left(\gamma \hat{\Lambda}^{-1}\frac{\Phi^{\top}\overline{\Phi}}{N}\right)^T  \theta^*  \right| \le \gamma^T (T + 1)\|\theta^*\|_2 \le \gamma^T (T + 1) \sqrt{d} / (1 - \gamma)  \le \varepsilon / 2.
\]
Therefore, for any $(s, a) \in \states \times \actions$, 
\[
\left|\phi(s,a)^{\top} \left(\sum_{t = 1}^T \left(\gamma \hat{\Lambda}^{-1}\frac{\Phi^{\top}\overline{\Phi}}{N}\right)^{t - 1} \cdot  \left(\frac{\gamma}{N} \hat{\Lambda}^{-1}\Phi^{\top} \zeta - \lambda \hat{\Lambda}^{-1}\theta^*\right) + \left(\gamma \hat{\Lambda}^{-1}\frac{\Phi^{\top}\overline{\Phi}}{N}\right)^T  \theta^* \right) \right|\le \varepsilon. 
\]
\end{proof}

\subsection{Proof of Lemma~\ref{lem:fqi_distribution_shift}}

In this proof, we set $T \ge C_T \log (d \cdot C_{\mathrm{init}} / (\varepsilon(1 -C_{\mathrm{policy}}^{1/2} \gamma))) / (1 - C_{\mathrm{policy}}^{1/2}\gamma)$ for sufficiently large constant $C_T > 0$. 
We fix $\lambda = C_{\lambda} T \sqrt{d \log (d / \delta) / N} / \lambda$ for some sufficiently large constant $C_{\lambda}$. 
We set $N = C_N T^6 d^3 C^2_{\mathrm{init}} \log(d / \delta) / (\varepsilon^2 (1 - \gamma)^4)$ for sufficiently large $C_N > 0$.
We use $L$ to denote $\hat{\Lambda}^{-1}\Phi^{\top}\overline{\Phi} / N$.

By standard matrix concentration inequalities~\citep{tropp2015introduction}, we can show that $\Phi^{\top}\Phi/ N$ (and $\overline{\Phi}^{\top}\overline{\Phi}/N$) concentrates around $\Lambda$ (and $\overline{\Lambda}$). 

\begin{lemma}\label{lem:concentrate}
With probability $1 - \delta / 2$, for some constant $C_4 > 0$, we have
\[
\|\Phi^{\top}\Phi/ N - \Lambda\|_2 \le C_4 \sqrt{d \log (d / \delta) / N}
\]
and
\[
\|\overline{\Phi}^{\top}\overline{\Phi}/ N - \overline{\Lambda}\|_2 \le C_4 \sqrt{d \log (d / \delta) / N}
\]
Conditioned on the event above, since $\lambda = C_{\lambda} T \sqrt{d \log (d / \delta) / N} / \lambda$, we have
\[
\hat{\Lambda} = \Phi^{\top}\Phi / N + \lambda I \succeq \Lambda.
\]
\end{lemma}

Now we are ready to prove Lemma~\ref{lem:fqi_distribution_shift}. By Lemma~\ref{lem:upper_main}, we have
\[
\expect_{s \sim \mu_{\mathrm{init}}}[(V^{\pi}(s) - \hat{V}_T(s))^2]
=\|\theta_T -  \theta^*\|_{\Lambda_\mathrm{init}}^2
= \left \| \left(\sum_{t = 1}^T \left(\gamma L\right)^{t - 1}  (\frac{\gamma}{N} \hat{\Lambda}^{-1}\Phi^{\top} \zeta - \lambda \hat{\Lambda}^{-1}\theta^*)  \right)
+ \left(\gamma L\right)^T  \theta^* \right \| _{\Lambda_\mathrm{init}}^2. 
\]
Therefore,
\begin{align*}
&\expect_{s \sim \mu_{\mathrm{init}}}[(V^{\pi}(s) - \hat{V}_T(s))^2] \\
\le& (2T + 1) \left( 
\sum_{t = 1}^T \left\| (\gamma L)^{t - 1}  \frac{\gamma}{N} \hat{\Lambda}^{-1}\Phi^{\top} \zeta\right\|_{\Lambda_\mathrm{init}}^2 
+ \sum_{t = 1}^T \left\| (\gamma L)^{t - 1}  \lambda \hat{\Lambda}^{-1} \theta^* \right\|_{\Lambda_\mathrm{init}}^2 
+ \|\left(\gamma L\right)^T  \theta^*\|_{\Lambda_\mathrm{init}}^2 
\right).
\end{align*}

Note that
\[
\left\|\left(\gamma \hat{\Lambda}^{-1}\Phi^{\top}\overline{\Phi} / N \right)^T  \theta^*\right\|_{\Lambda_{\mathrm{init}}}^2
 \le \gamma^2 \|\hat{\Lambda}^{-1/2} \Lambda_{\mathrm{init}} \hat{\Lambda}^{-1/2} \|_2  \cdot \|\Phi (N\hat{\Lambda})^{-1} \Phi^{\top}\|_2 \cdot \left\|\left(\gamma \Lambda^{-1} \Phi^{\top}\overline{\Phi} / N\right)^{T - 1}  \theta^*\right\|_{\overline{\Phi}^{\top}\overline{\Phi}/N}^2
\]
and
\begin{align*}
& \left\|\left(\gamma \Lambda^{-1} \Phi^{\top}\overline{\Phi} / N\right)^{T - 1}  \theta^*\right\|_{\overline{\Phi}^{\top}\overline{\Phi}/N}^2\\
\le &\gamma^2 \|\hat{\Lambda}^{-1/2} (\overline{\Phi}^{\top}\overline{\Phi}/N) \hat{\Lambda}^{-1/2} \|_2  \cdot \|\Phi (N\hat{\Lambda})^{-1} \Phi^{\top}\|_2 \cdot \left\|\left(\gamma \Lambda^{-1}\Phi^{\top}\overline{\Phi} / N \right)^{T - 2}  \theta^*\right\|_{\overline{\Phi}^{\top}\overline{\Phi}/N}^2.
\end{align*}
Conditioned on the event in Lemma~\ref{lem:concentrate}, we have $\hat{\Lambda} \succeq \Lambda$ and $N\hat{\Lambda} \succeq \Phi^{\top}\Phi$.
This implies
\[
\|\Phi (N\hat{\Lambda})^{-1} \Phi^{\top}\|_2 \le 1.
\]
By Assumption~\ref{assumption:low_distribution_shift}, we also have
\[
\|\hat{\Lambda}^{-1/2} \Lambda_{\mathrm{init}} \hat{\Lambda}^{-1/2} \|_2 \le C_{\mathrm{init}}
\]
and
\begin{align*}
 &\|\hat{\Lambda}^{-1/2} (\overline{\Phi}^{\top}\overline{\Phi}/N) \hat{\Lambda}^{-1/2} \|_2  \le  \|\hat{\Lambda}^{-1/2} \overline{\Lambda} \hat{\Lambda}^{-1/2} \|_2  +  \|\hat{\Lambda}^{-1/2} (\overline{\Phi}^{\top}\overline{\Phi}/N - \overline{\Lambda}) \hat{\Lambda}^{-1/2} \|_2 \\\
\le &C_{\mathrm{policy}} + C_4 \sqrt{d \log (d / \delta) / N} / \lambda \le  C_{\mathrm{policy}} + 1 / T
\end{align*}
since $\lambda = C_{\lambda} T \sqrt{d \log (d / \delta) / N} / \lambda$ for sufficiently large constant $C_{\lambda}$. 
Therefore,
\[
\left\|\left(\gamma \hat{\Lambda}^{-1}\Phi^{\top}\overline{\Phi} / N \right)^T  \theta^*\right\|_{\Lambda_{\mathrm{init}}}^2 \le \gamma^{2T}  C_{\mathrm{init}} (C_{\mathrm{policy}} + 1 / T)^{T - 1} \|\theta^*\|_{\overline{\Phi}^{\top}\overline{\Phi}/N}^2 \le e \gamma^{2T} C_{\mathrm{init}} C_{\mathrm{policy}}^{T - 1} d / (1 - \gamma)^2. 
\]
Therefore. by taking $T \ge C_5 \log (d \cdot C_{\mathrm{init}} / (\varepsilon(1 -C_{\mathrm{policy}}^{1/2} \gamma))) / (1 - C_{\mathrm{policy}}^{1/2}\gamma)$ for some constant $C_5 > 0$, we have \[\left\|\left(\gamma \hat{\Lambda}^{-1}\Phi^{\top}\overline{\Phi} / N \right)^T  \theta^*\right\|_{\Lambda_{\mathrm{init}}}^2 \le \varepsilon / (2T + 1)^2.\]

Furthermore,
\[
\left\|\frac{\gamma}{N} \hat{\Lambda}^{-1}\Phi^{\top} \zeta\right\|_{\overline{\Phi}^{\top}\overline{\Phi}/N}^2
\le \frac{\gamma^2}{N} \|\hat{\Lambda}^{-1/2} ( \overline{\Phi}^{\top}\overline{\Phi}/N  )\hat{\Lambda}^{-1/2}\|_2 \cdot \|\zeta\|_{\Phi (N \hat{\Lambda})^{-1} \Phi^{\top}}^2 \le \frac{\gamma^2}{N} (C_{\mathrm{policy}} + 1 / T)  \cdot \|\zeta\|_{\Phi (N \hat{\Lambda})^{-1} \Phi^{\top}}^2.
\] 

According to~\citep{hsu2012tail}, with probability $1 - \delta / 2$, there exists a constant $C_5 > 0$ such that 
\[
\|\zeta\|_{\Phi (N\hat{\Lambda})^{-1}\Phi^{\top}}^2 \le C_5d  \log(1 / \delta) / (1 - \gamma)^2. 
\]
By Assumption~\ref{assumption:low_distribution_shift}, we have $C_{\mathrm{policy}}\gamma^2 < 1$, which implies
\[
\left\|\frac{\gamma}{N} \hat{\Lambda}^{-1}\Phi^{\top} \zeta\right\|_{\overline{\Phi}^{\top}\overline{\Phi}/N}^2 \le C_5 \frac{1 + 1/T}{N} d  \log(d / \delta) / (1 - \gamma)^2. 
\]
Hence, there exists a constant $C_6 > 0$, such that for each $t \in \{1, 2, \ldots, T\}$,
\[
 \left\| (\gamma L)^{t - 1}  \frac{\gamma}{N} \hat{\Lambda}^{-1}\Phi^{\top} \zeta\right\|_{\Lambda_\mathrm{init}}^2 
 \le C_6  C_{\mathrm{init}}   d  \log(d / \delta) / (1 - \gamma)^2 / N.
\]
Similarly, for each $t \in \{1, 2, \ldots, T\}$,
\[
 \left\| (\gamma L)^{t - 1}  \lambda \hat{\Lambda}^{-1} \theta^* \right\|_{\Lambda_\mathrm{init}}^2  \le e C_{\mathrm{init}} \lambda^2 \|\theta^*\|_{\hat{\Lambda}^{-1}}^2 \le e C_{\mathrm{init}}\lambda d / (1 - \gamma)^2 = C_{7} C_{\mathrm{init}} T d \sqrt{d \log (d / \delta) / N} / (1 - \gamma)^2 
  \]
  where $C_7 > 0$ is a constant. 
 Therefore, since $N = C_N T^6 d^3 C^2_{\mathrm{init}} \log(d / \delta) / (\varepsilon^2 (1 - \gamma)^4)$ for sufficiently large $C_N > 0$, we have \[\left\| (\gamma L)^{t - 1}  \frac{\gamma}{N} \hat{\Lambda}^{-1}\Phi^{\top} \zeta\right\|_{\Lambda_\mathrm{init}}^2  \le \varepsilon/ (2T + 1)^2\]and
 \[
 \left\| (\gamma L)^{t - 1}  \lambda \hat{\Lambda}^{-1} \theta^* \right\|_{\Lambda_\mathrm{init}}^2 \le \varepsilon/ (2T + 1)^2,
 \]
 which implies
 \[
 \expect_{s \sim \mu_{\mathrm{init}}}[(V^{\pi}(s) - \hat{V}_T(s))^2]  \le \varepsilon. 
 \]

\clearpage

\section{Additional Experiment Details}\label{sec:exp_details}
In this section, we provide more details about our experiments.

\subsection{Details in Step 1}
In this section, we provide details of the first step of our experiments (i.e., the step for deciding on a target policy to be evaluated, along with a good feature mapping for this policy). 

When running DQN and TD3, the number of hidden layers is always set to be $3$. 
The activation function is always set to be leaky ReLU with slop $0.1$, i.e., 
\[
\sigma(x) = \begin{cases}
x & x \ge 0\\
-0.1 x & x < 0
\end{cases}.
\]
When running DQN and TD3, there is no bias term in the output layer. 

For TD3, we use the official implementation released by the authors~\citep{fujimoto2018addressing}\footnote{\url{https://github.com/sfujim/TD3}}.
Except for hyperparameters explicitly mentioned above, we use the default hyperparameters in the implementation released by the authors.
For DQN, we write our own implementation (using the PyTorch package), and the choices of hyperparameters are reported in Table~\ref{table:hyper_dqn}. 
\begin{table}
\centering
\begin{tabular}{l l}
\hline
\textbf{Hyperparameter} & \textbf{Choice}\\
\hline
Number of Hidden Units in Q-network & 100\\
Number of Hidden Layers in Q-network & 3\\
Activation Function in Q-network & Leaky ReLU with slope $0.1$\\
Bias Term in the Output Layer & None\\
Batch Size & 128\\
Optimizer & Adam\\
Target Update Rate ($\tau$) & 0.1 \\
Total Number of Time Steps & 200000\\
Replay Buffer Size & 1000000\\
Gradient Clipping & False\\
Discount Factor & 0.99\\
Exploration Strategy & $\varepsilon$-Greedy\\
Initial Value of $\varepsilon$ & 1\\
Minimum Value of $\varepsilon$ & $0.01$\\
Decrease of $\varepsilon$ After Each Episode &$0.01$\\
Iterations Per Time Step & 1\\
\hline
\end{tabular}
\caption{\emph{Hyperparameters of DQN}.}\label{table:hyper_dqn}. 
\end{table}
We also report the learning curves of both algorithm in Figure~\ref{fig:learning_curves}.
\begin{figure}
\centering
\subfigure[Ant-v2 (TD3)]{
\includegraphics[width=0.30\textwidth]{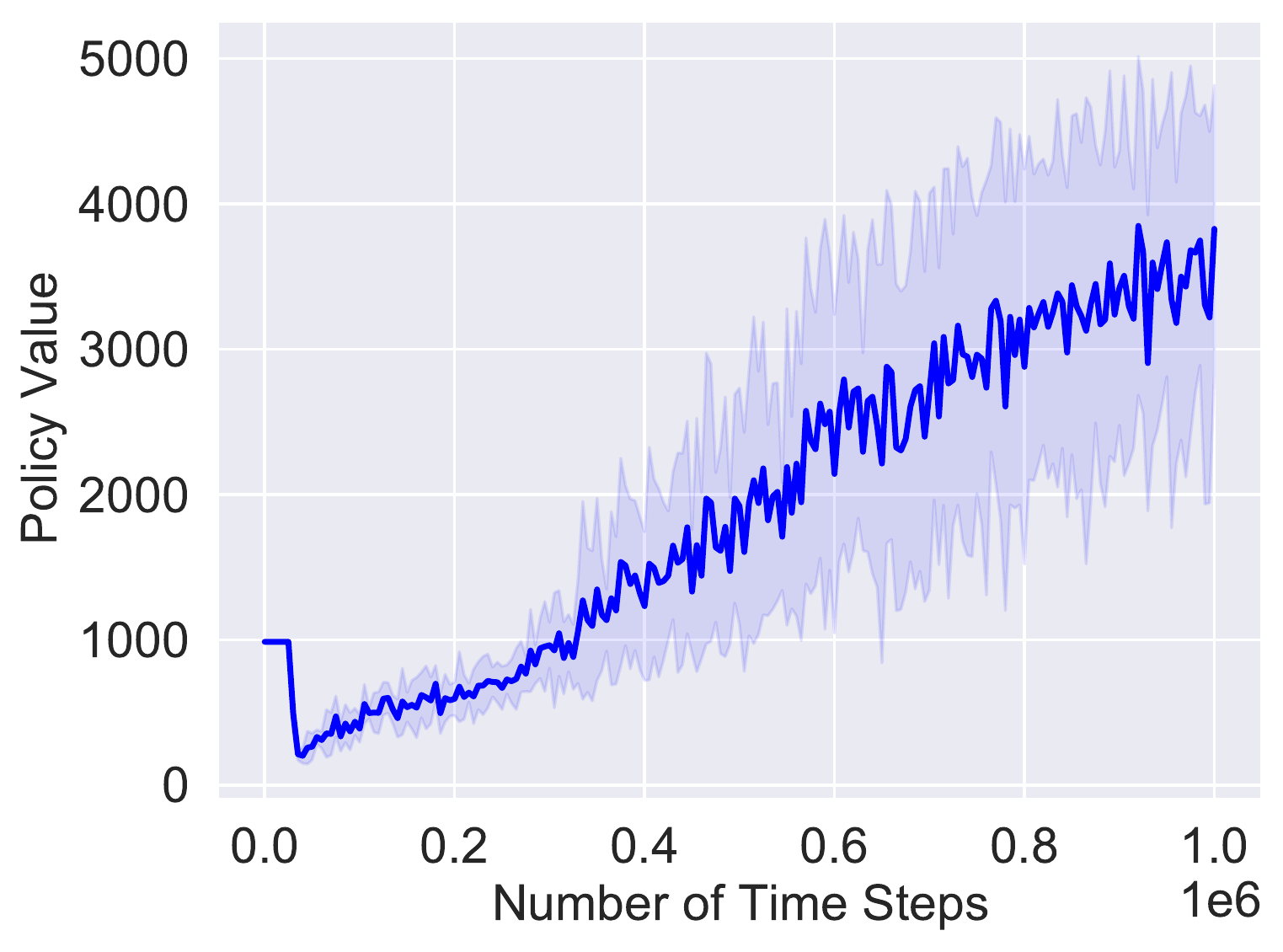}
}
\subfigure[CartPole-v0 (DQN)]{
\includegraphics[width=0.30\textwidth]{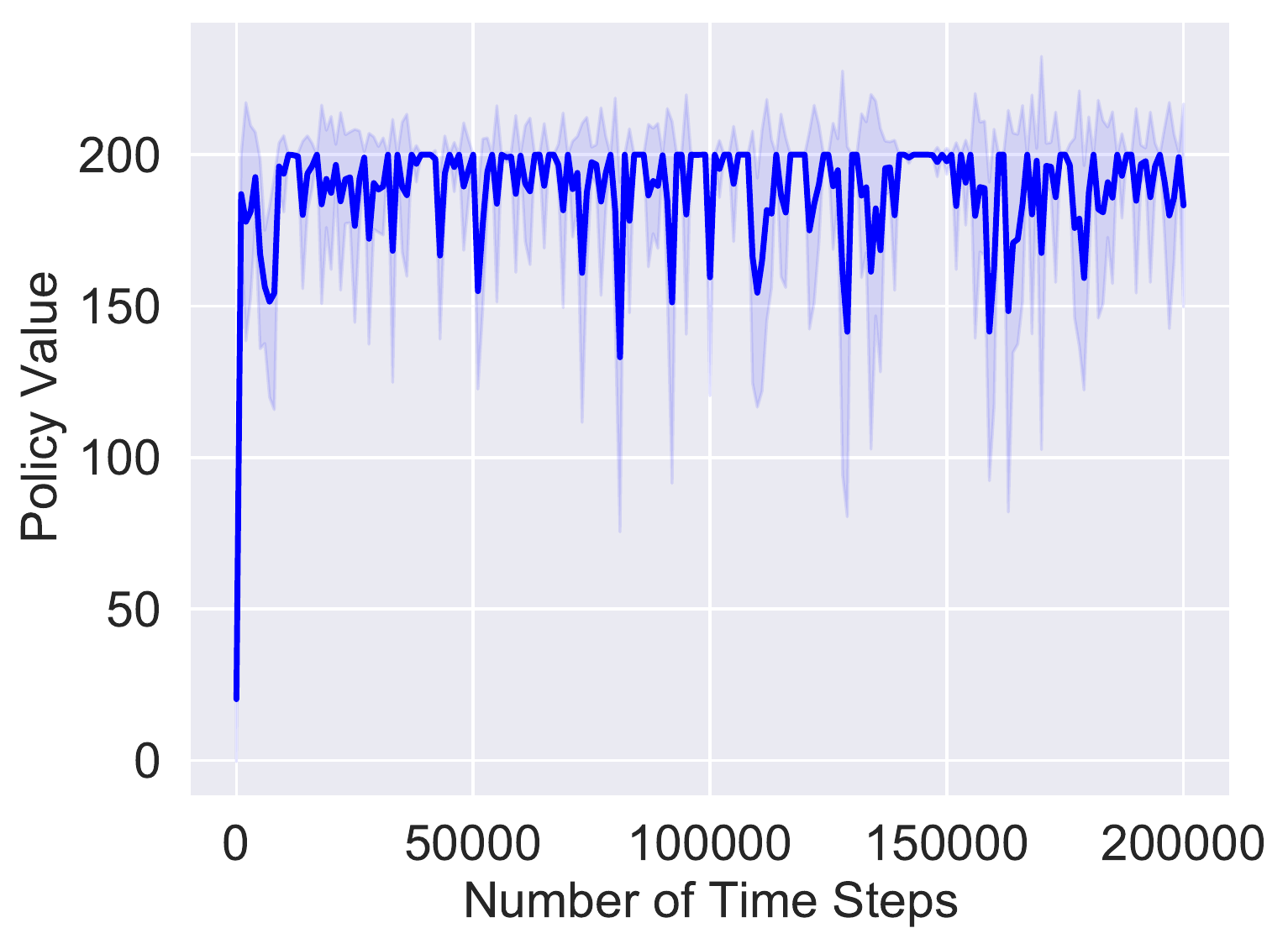}
}
\subfigure[HalfCheetah-v2 (DQN)]{
\includegraphics[width=0.30\textwidth]{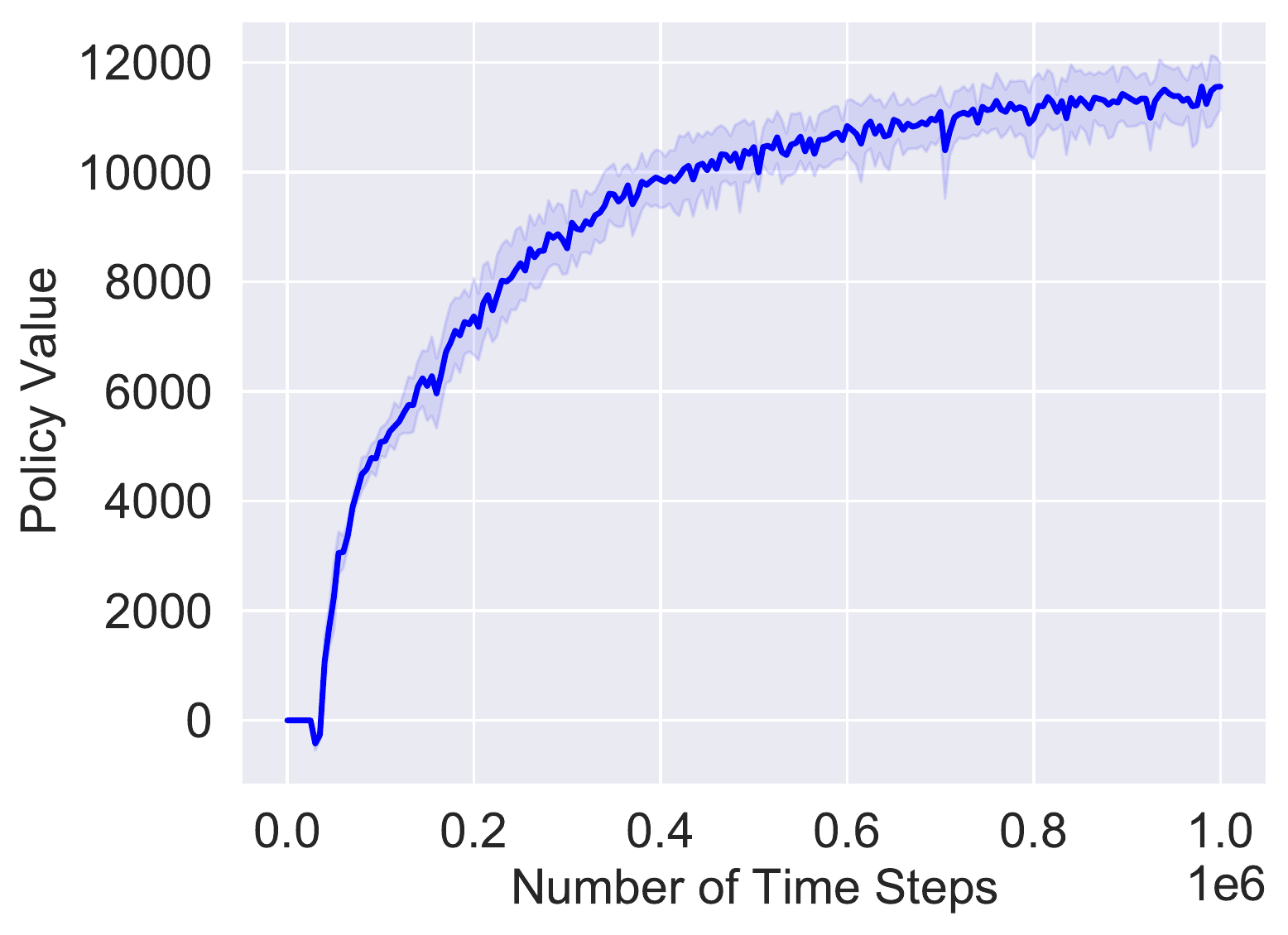}
}
\subfigure[Hopper-v2 (TD3)]{
\includegraphics[width=0.30\textwidth]{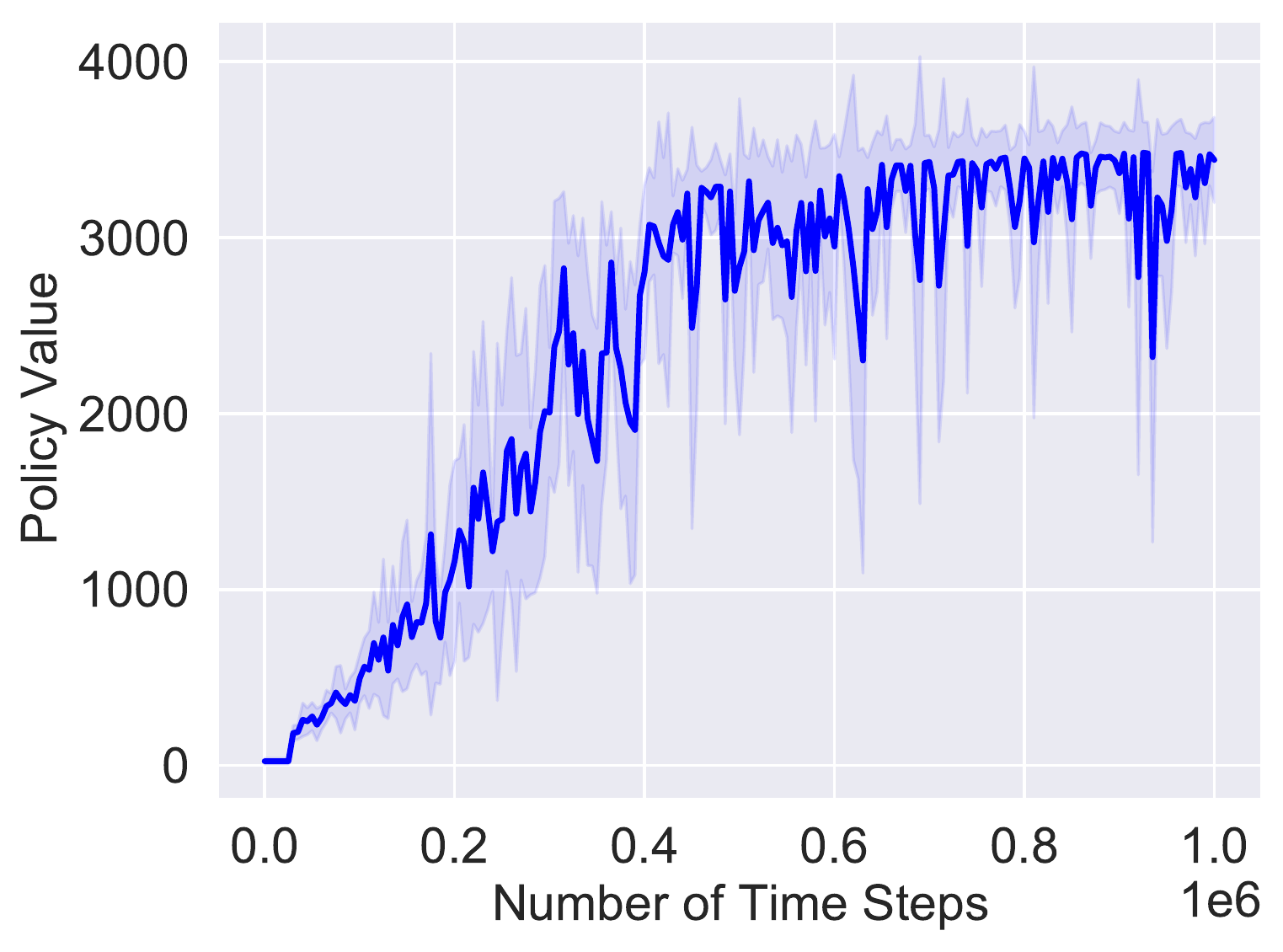}
}
\subfigure[MountainCar-v0 (TD3)]{
\includegraphics[width=0.30\textwidth]{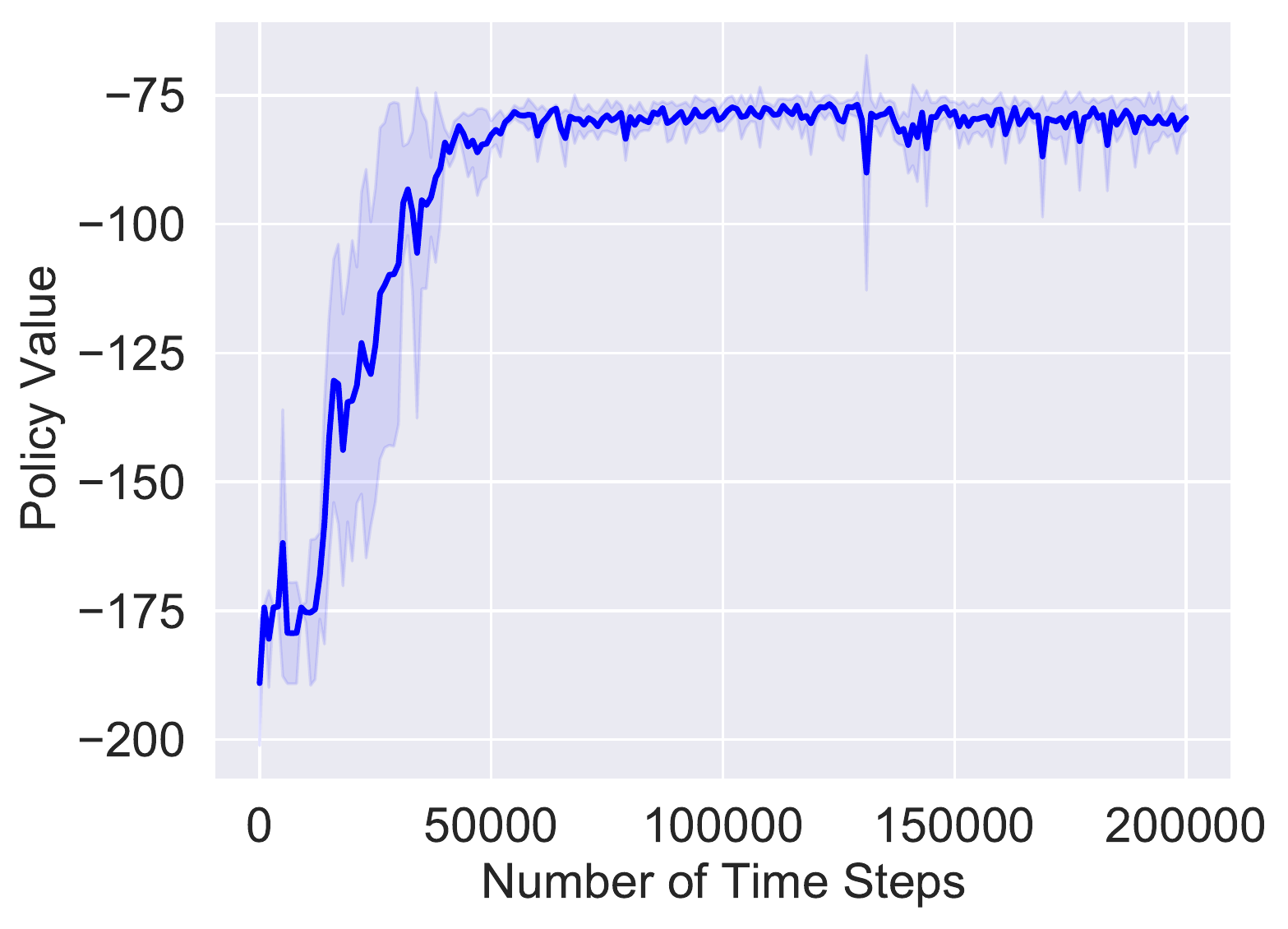}
}
\subfigure[Walker2d-v2 (DQN)]{
\includegraphics[width=0.30\textwidth]{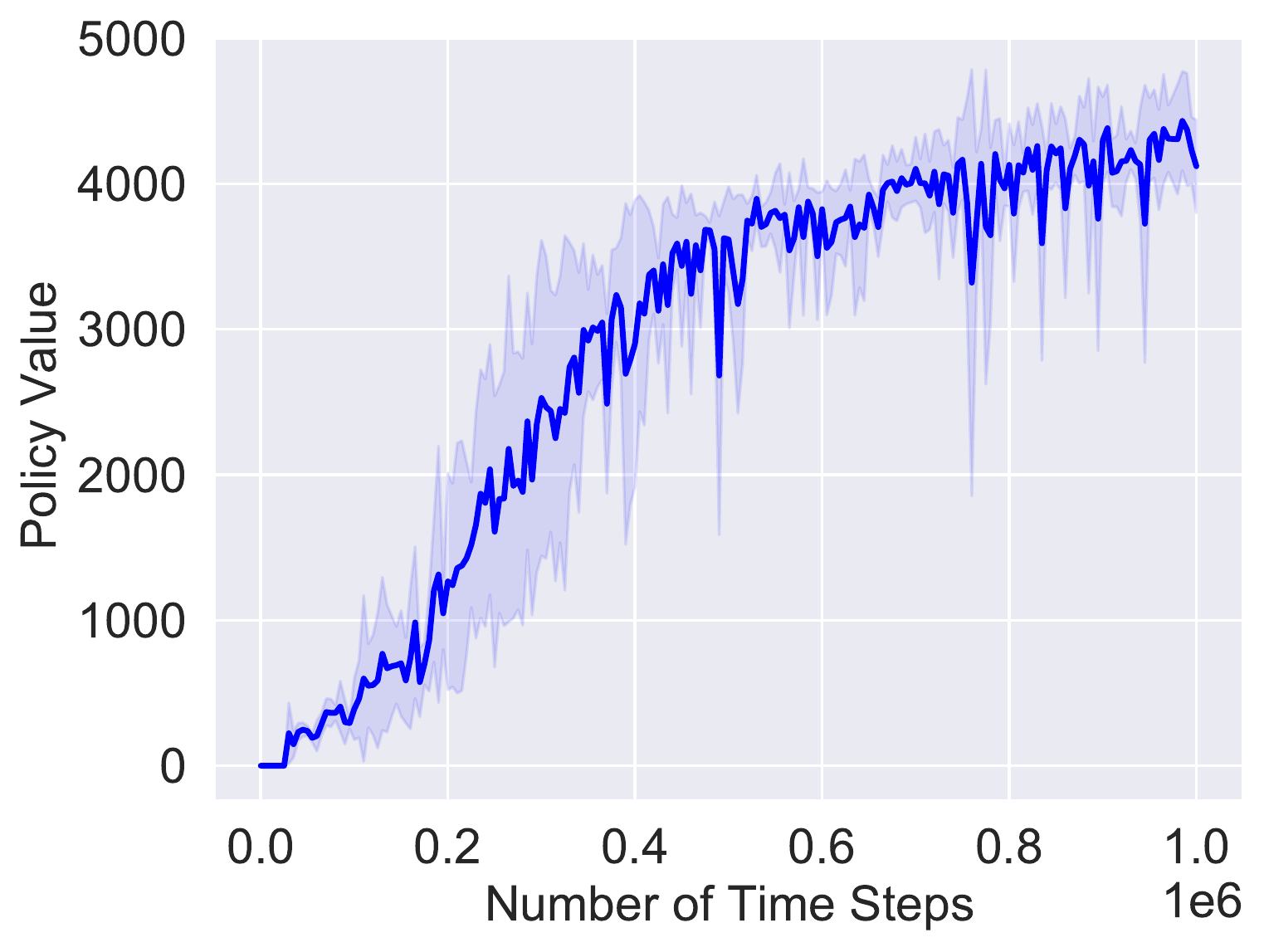}
}
\caption{Learning curves of DQN and TD3.}
\label{fig:learning_curves}
\end{figure}

\paragraph{Reward Shaping in MountainCar-v0.}
When running DQN on MountainCar-v0, we slightly modify the reward function to facilitate exploration. 
It is known that without exploration bonus, exploration in MountainCar-v0 is a hard problem (see e.g.~\citep{houthooft2016vime}), and using exploration bonus potentially leads to a representation that is incompatible with the original problem. 
To mitigate the issue of exploration, we slightly modify the reward function. 
Suppose the current position of the car is $x$.
In the original problem, the reward is set to be $-1$ if $x < 0.5$, and is set to be $0$ if $x \ge 0.5$. 
In our case, we set the reward to be $\max\{-1, x - 0.5\}$. 
Using such a modified reward function, the reward values are still in $[-1, 0]$, while being smoother (with respect to the current position of the car) and therefore facilitates exploration. 
All of our experiments are performed on this modified version of MountainCar-v0.

\paragraph{Details of Random Fourier Features. } Recall that for a vector $x \in \mathbb{R}^{D}$, its random Fourier feature is defined to be $\phi(x) = \cos(\sqrt{2 \gamma_{\mathrm{RFF}}}Wx + u)$. Here $W \in \mathbb{R}^{d \times D}$ and $u \in \mathbb{R}^D$, and each entry of $W$ is sampled from the standard Gaussian distribution and each entry of $u$ is sampled from $[0, 2\pi]$ uniformly at random. 
Here, both $\gamma_{\mathrm{RFF}}$ and the feature dimension $d$ are hyperparameter. 
Moreover, the $\cos$ function is understood in a pointwise sense. 
Also note that we use the same $W$ and $u$ for all $x \in \mathbb{R}^d$. 
In order to properly set the hyperparameter $\gamma_{\mathrm{RFF}}$, we first randomly sample $10^4$ pairs of data points, and use $D_{\mathrm{median}}$ to denote the median of the squared $\ell_2$ distances of the $10^4$ pairs of data points.
Then, for each environment, we manually pick a constant $C$ and set $\gamma_{\mathrm{RFF}} = C / D_{\mathrm{median}}$.
The choices of $d$ and $C$ for each environment are reported in Table~\ref{table:rff}.
\begin{table}
\parbox{.45\linewidth}{
\centering
\begin{tabular}{l l l}
\hline
Environment & $\gamma_{\mathrm{RFF}}$ & $d$\\
\hline
Ant-v2 & $0.1 / D_{\mathrm{median}}$ & $512$\\
CartPole-v0 & $1 / D_{\mathrm{median}}$ & $100$\\
HalfCheetah-v2 & $1 / D_{\mathrm{median}}$ & $256$\\
Hopper-v2 & $1 / D_{\mathrm{median}}$ & $256$\\
MountainCar-v0 & $64 / D_{\mathrm{median}} $& $100$\\
Walker2d-v2 & $4 / D_{\mathrm{median}}$ & $512$\\
\hline
\end{tabular}
\caption{\emph{Hyperparamters of random Fourier features}. }\label{table:rff}
}
\hfill
\parbox{.45\linewidth}{
\begin{tabular}{ll}
\hline
Environment & Discounted Value\\
\hline
Ant-v2 & $411.77 \pm 96.824$\\
CartPole-v0 & $90.17 \pm 20.61$\\
HalfCheetah-v2 & $1053.71 \pm 121.74$\\
Hopper-v2 & $321.42 \pm 30.26$\\
MountainCar-v0 & $-26.16 \pm 17.83$\\
Walker2d-v2 & $336.64 \pm 49.80$\\
\hline
\end{tabular}
\caption{\emph{Values of Randomly Chosen States}. Mean value of the $100$ randomly chosen states (used for evaluating the estimations), $\pm$ standard deviation. }
\label{table:value_supp}

}
\end{table}

\subsection{Details in Step 2}
Recall that in Step 2 of our experiment (i.e., the step for collecting offline data), we use four lower performing policies (for each environment) $\pi_{\mathrm{sub}}^1, \pi_{\mathrm{sub}}^2, \pi_{\mathrm{sub}}^3, \pi_{\mathrm{sub}}^4$ to collect offline data. 
In order to find these lower performing policies, when running DQN, we evaluate the current policy every $1000$ time steps, store the policy if the value of the current policy bypasses certain threshold, and then increase the threshold. 
When running TD3, we evaluate the current policy every $5000$ time steps, store the policy if the value of the current policy bypasses certain threshold, and then increase the threshold. 
We then manually pick four policies for each environment.
The values of these lower performing policies are reported in Table~\ref{table:lower_performing}. 
\begin{table}
\centering
\begin{tabular}{l c c c c}
\hline
Environment & $V^{\pi_{\mathrm{sub}}^1}$ & $V^{\pi_{\mathrm{sub}}^2}$ & $V^{\pi_{\mathrm{sub}}^3}$ & $V^{\pi_{\mathrm{sub}}^4}$\\
\hline
Ant-v2 & 4157.40 / 359.65 & 3289.39 / 307.75 & 2131.66 / 227.15 & 1682.81 / 198.89\\
CartPole-v0 & 198.55 / 90.25 & 186.80 / 85.80 & 171.70 / 86.88 & 169.70 / 81.73\\
HalfCheetah-v2 & 10181.87 / 734.22 & 8210.71 / 602.07 & 6129.40 / 457.19 & 4091.48 / 316.07\\
Hopper-v2 & 3519.63 / 267.17 & 2840.49 / 265.56 & 2386.34 / 257.85 & 1124.92 / 245.80\\
MountainCar-v0 & -80.84 / -54.20 & -83.66 / -54.17 & -90.04 / -57.14 & -94.80 / -58.09\\
Walker2d-v2 & 4247.12 / 250.25 & 3340.85 / 224.25 & 2072.05 / 218.81 & 1321.86 / 156.27\\
\hline
\end{tabular}
\caption{\emph{Value of Lower Performing Policies}. For each entry, the first number is the undiscounted value, while the second number is the discounted value.}\label{table:lower_performing}
\end{table}

\subsection{Details in Step 3}\label{sec:details_3}

\paragraph{The LSTD Algorithm.}
Here we give a description of the LSTD algorithm (proposed in~\citep{bradtke1996linear}) in Algorithm~\ref{algo:lstd}.

\paragraph{Evaluating Offline RL Methods. }
Recall that when evaluating the performance of offline RL methods, we report the square root of the mean squared evaluation error, taking average over $100$ randomly chosen states.
In order to have a diverse set of states with a wide range of  values when evaluating the performance, we sample $100$ trajectories using the target policy, and randomly choose a state from the first $100$ time steps on each sampled trajectory.  
We also report the values ($V^{\pi}(s)$) of those randomly chosen states in Table~\ref{table:value_supp} for all the six environments. 
When evaluating the performance of FQI, we report the evaluation error after every $10$ rounds of FQI. 

\paragraph{Variance of Step 3.}
Here we stress that the offline RL step (Step 3) itself could also have high variance. 
Here we plot the performance of FQI on Ant-v2, HalfCheetah-v2, Hopper-v2 and Walker2d-v2 when using a fixed policy and a fixed representation. 
Here we repeat the setting in Figure~\ref{fig:noise}, but use a random $10\%$ subset of the original dataset and repeat for $5$ times. 
Therefore, in this setting, there is no randomness coming from the choice of the target policy or the representation, and instead all randomness comes from the offline methods themselves. 
The results are reported in Figure~\ref{fig:randomness}. 
Here, even if the policy and the representation are fixed, the variance of the estimation could still be high. 
Moreover, adding more data into the dataset generally results in higher variance. 
Note that this is consistent with our theory in Lemma~\ref{lem:upper_main}, which shows that the variance can also be exponentially amplified without strong representation conditions and low distribution shift conditions.

\begin{figure}[!h]
\centering
\subfigure[Ant-v2]{
\includegraphics[width=0.4\textwidth]{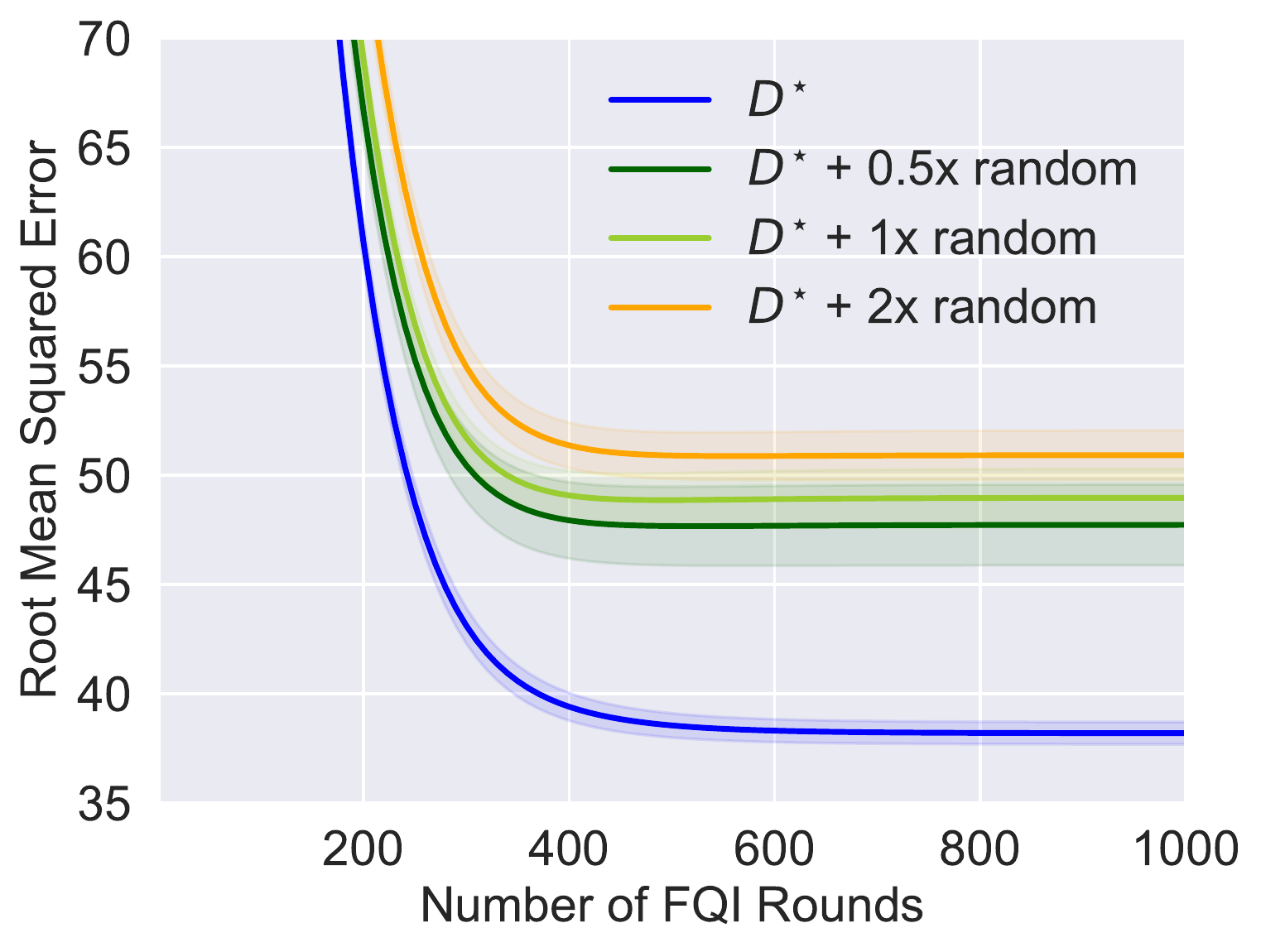}
}
\subfigure[Hopper-v2]{
\includegraphics[width=0.4\textwidth]{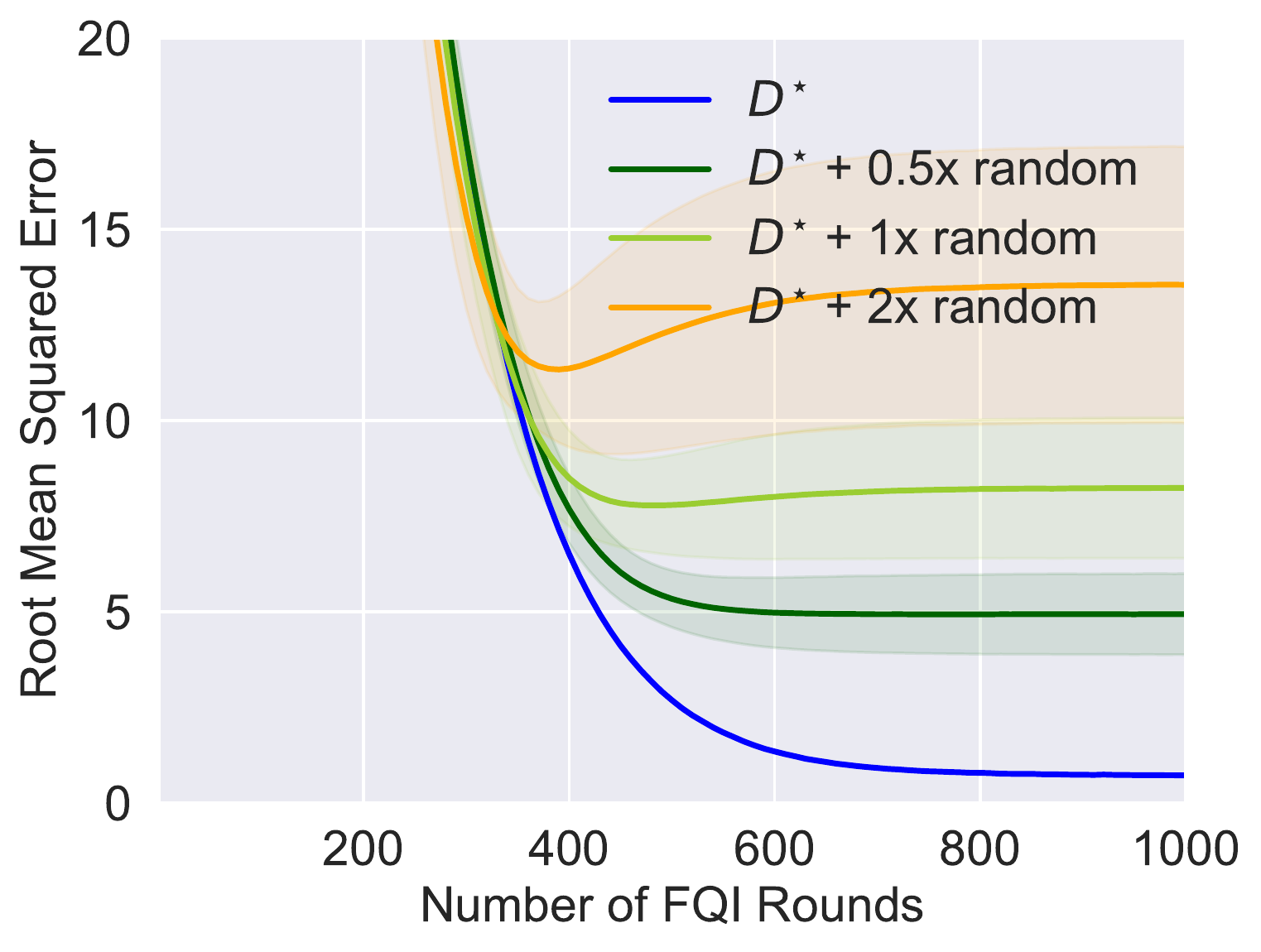}
}
\\
\subfigure[HalfCheetah-v2]{
\includegraphics[width=0.4\textwidth]{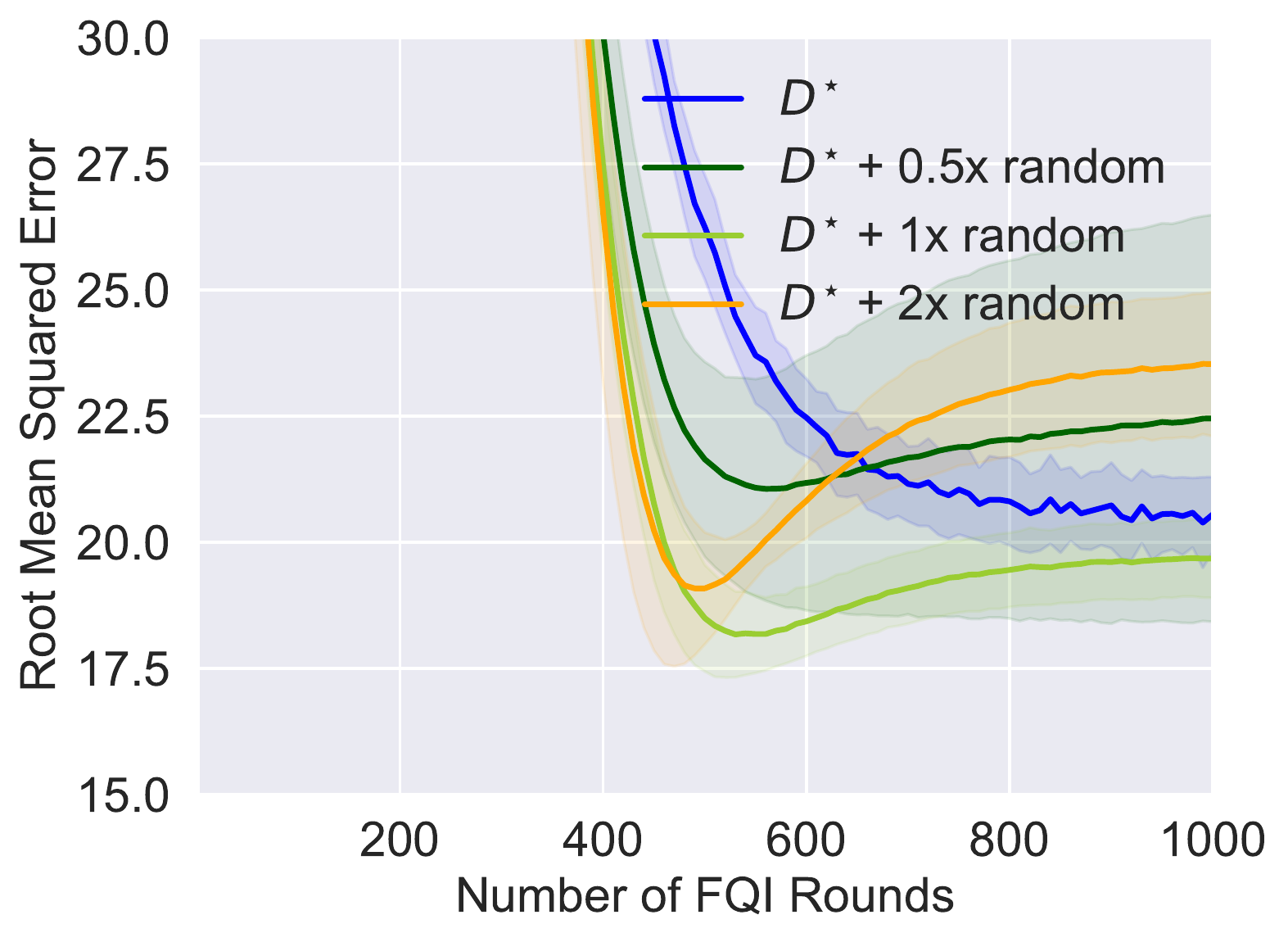}
}
\subfigure[Walker2d-v2]{
\includegraphics[width=0.4\textwidth]{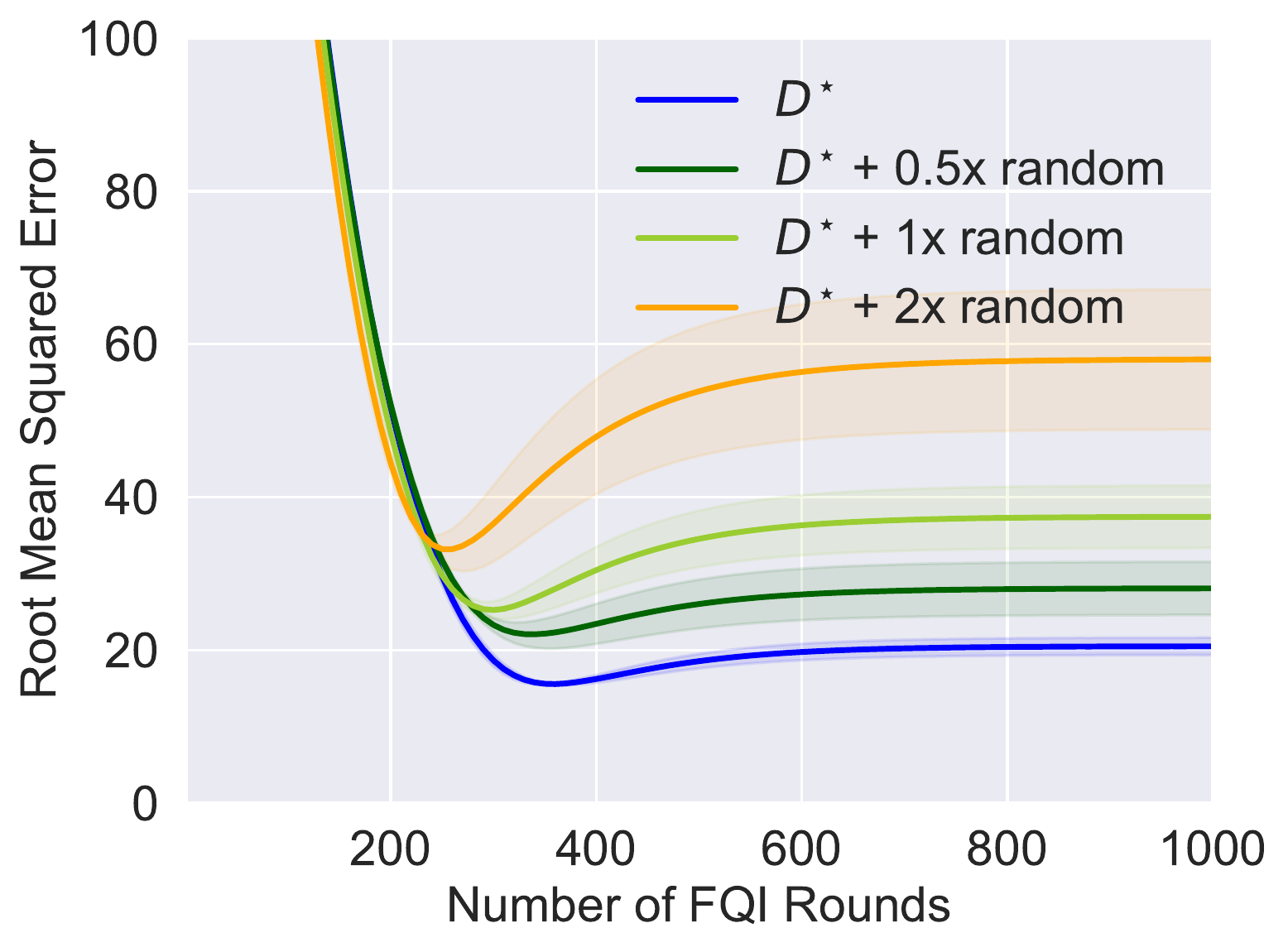}
}
\caption{Experiments for verifying the variance of Step 3.}
\label{fig:randomness}
\end{figure}

\begin{algorithm}[t]
	\centering 
	\caption{Least-Squares Temporal Difference (LSTD)}
	\label{algo:main}
	\begin{algorithmic}[1]
	\STATE \textbf{Input:} policy $\pi$ to be evaluated, number of samples $N$, regularization parameter $\lambda > 0$
	\label{line:sample}
	\STATE Take samples $(s_i, a_i) \sim \mu$, $r_i \sim r(s_i, a_i)$ and $\overline{s}_i \sim \trans(s_i, a_i)$ for each $i \in [N]$
	\STATE $\hat{\theta} =(\frac{1}{N}\sum_{i = 1}^N \phi(s_i, a_i) (\phi(s_i, a_i) - \gamma\phi(\overline{s}_i, \pi(\overline{s}_i))) + \lambda I)^{-1} (\frac{1}{N} \sum_{i = 1}^{N} \phi(s_i, a_i) \cdot r_i )) $
	\STATE $\hat{Q}(\cdot, \cdot) = \phi(\cdot, \cdot)^{\top} \hat{\theta}$ and $\hat{V}(\cdot) = \hat{Q}(\cdot, \pi(\cdot))$
	\STATE \textbf{return} $\hat{Q}(\cdot, \cdot)$
	\end{algorithmic}
	\label{algo:lstd}
\end{algorithm}

\clearpage
\section{Additional Experiment Results}\label{sec:exp_results}
\paragraph{Full Version of Figure~\ref{fig:noise}.}
In Figure~\ref{fig:noise_add}, we present the full version of Figure~\ref{fig:noise}, where we plot of the performance of FQI  with features from pre-trained neural networks and datasets induced by random policies, on all the six environments. 

\paragraph{Full Version of Figure~\ref{fig:suboptimal}.}
In Figure~\ref{fig:suboptimal_add}, we present the full version of Figure~\ref{fig:suboptimal}, where we plot of the performance of FQI  with features from pre-trained neural networks and datasets induced by lower performing policies, on all the six environments. 

\paragraph{Full Version of Figure~\ref{fig:rff}.}
In Figure~\ref{fig:rff_add}, we present the full version of Figure~\ref{fig:rff}, where we plot of the performance of FQI  with random Fourier features and datasets induced by random policies, on all the six environments. 

\paragraph{Full Version of Figure~\ref{fig:ridge}.}
In Figure~\ref{fig:ridge_ant} to Figure~\ref{fig:ridge_walker}, we present the full version of Figure~\ref{fig:rff}, where we plot of the performance of FQI with features from pre-trained neural networks, datasets induced by random policies, and different regularization parameter $\lambda$, on all the six environments. 
Here we vary the number of additional samples from random trajectories and the regularization parameter $\lambda$. 

\paragraph{Full Version of Table~\ref{table:lstd}}
In Table~\ref{table:lstd_add}, we present the full version of Table~\ref{table:lstd}, where we provide the performance of LSTD with features from pre-trained neural networks and distributions induced by random policies, on all the six environments. 

%

\begin{figure}[!h]
\thisfloatpagestyle{empty}
\centering
\subfigure[Ant-v2]{
\includegraphics[width=0.30\textwidth]{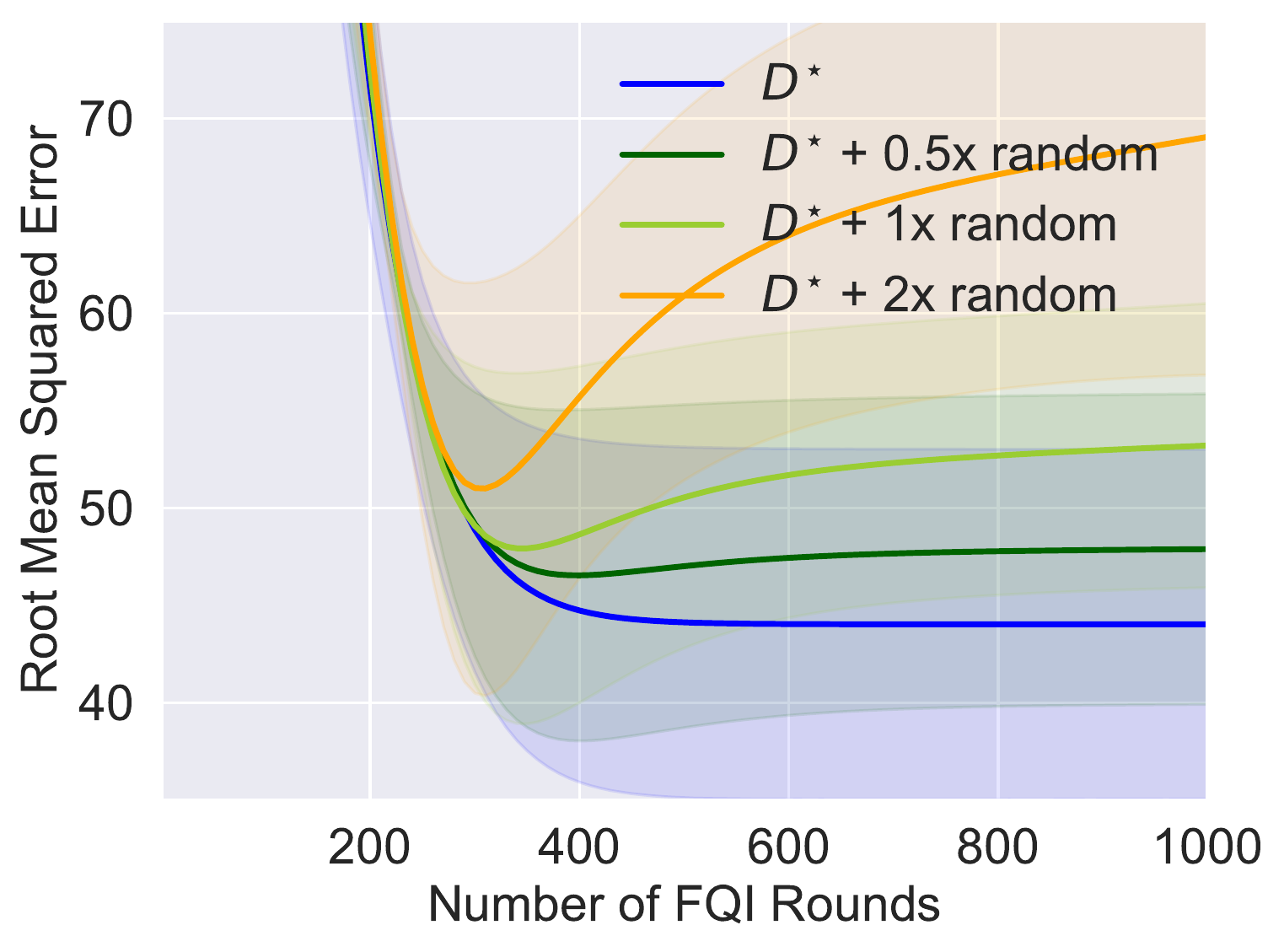}
}
\subfigure[CartPole-v0]{
\includegraphics[width=0.30\textwidth]{fig/fig_CartPole-v0_0}
}
\subfigure[Hopper-v2]{
\includegraphics[width=0.30\textwidth]{fig/fig_Hopper-v2_0}
}

\subfigure[HalfCheetah-v2]{
\includegraphics[width=0.30\textwidth]{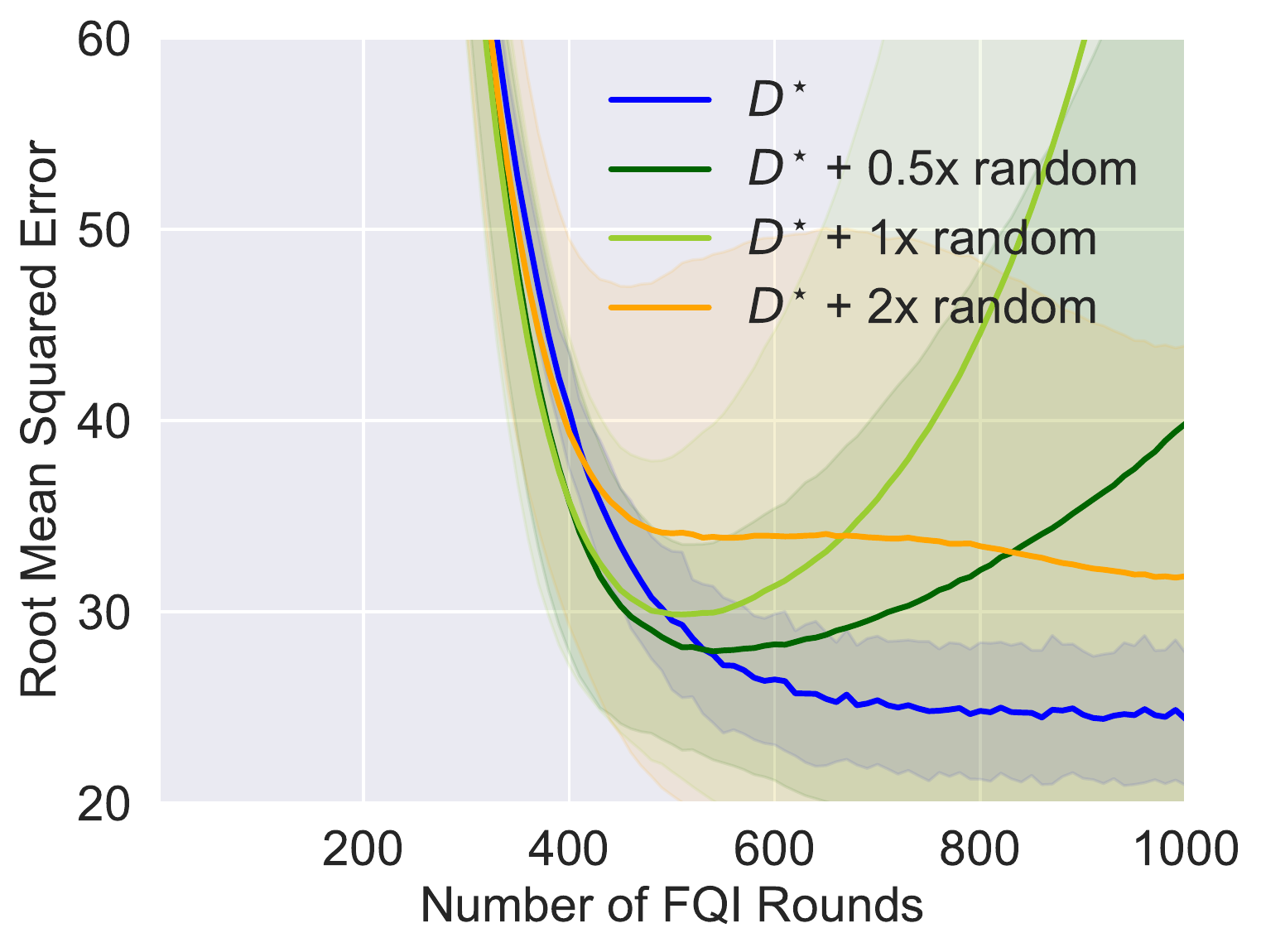}
}
\subfigure[MountainCar-v0]{
\includegraphics[width=0.30\textwidth]{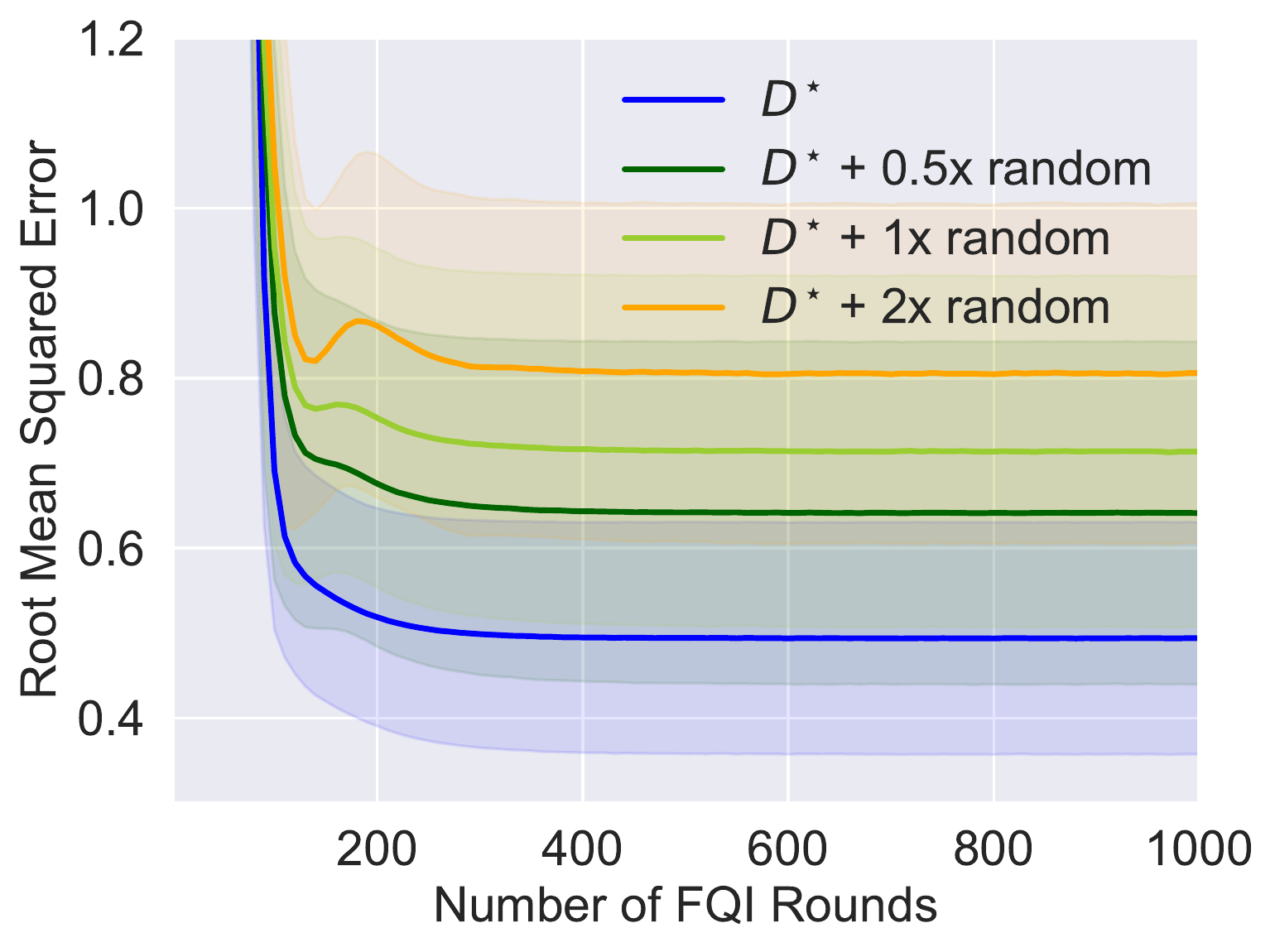}
}
\subfigure[Walker2d-v2]{
\includegraphics[width=0.30\textwidth]{fig/fig_Walker2d-v2_0}
}
\caption{Performance of FQI with features from pre-trained neural networks and datasets induced by random policies.}
\label{fig:noise_add}

\thisfloatpagestyle{empty}
\centering
\subfigure[Ant-v2]{
\includegraphics[width=0.30\textwidth]{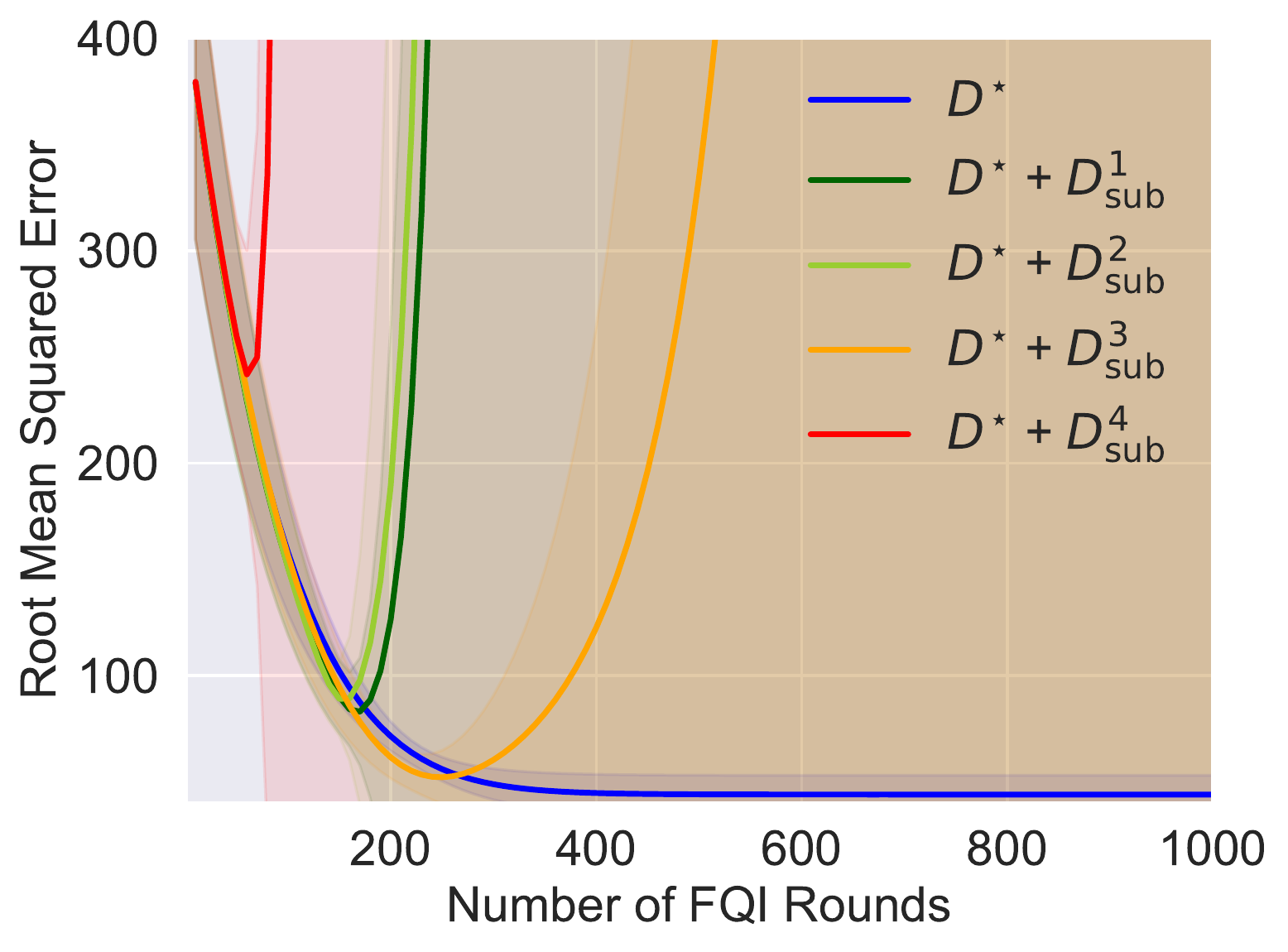}
}
\subfigure[CartPole-v0]{
\includegraphics[width=0.30\textwidth]{fig/fig_CartPole-v0_1}
}
\subfigure[Hopper-v2]{
\includegraphics[width=0.30\textwidth]{fig/fig_Hopper-v2_1}
}

\subfigure[HalfCheetah-v2]{
\includegraphics[width=0.30\textwidth]{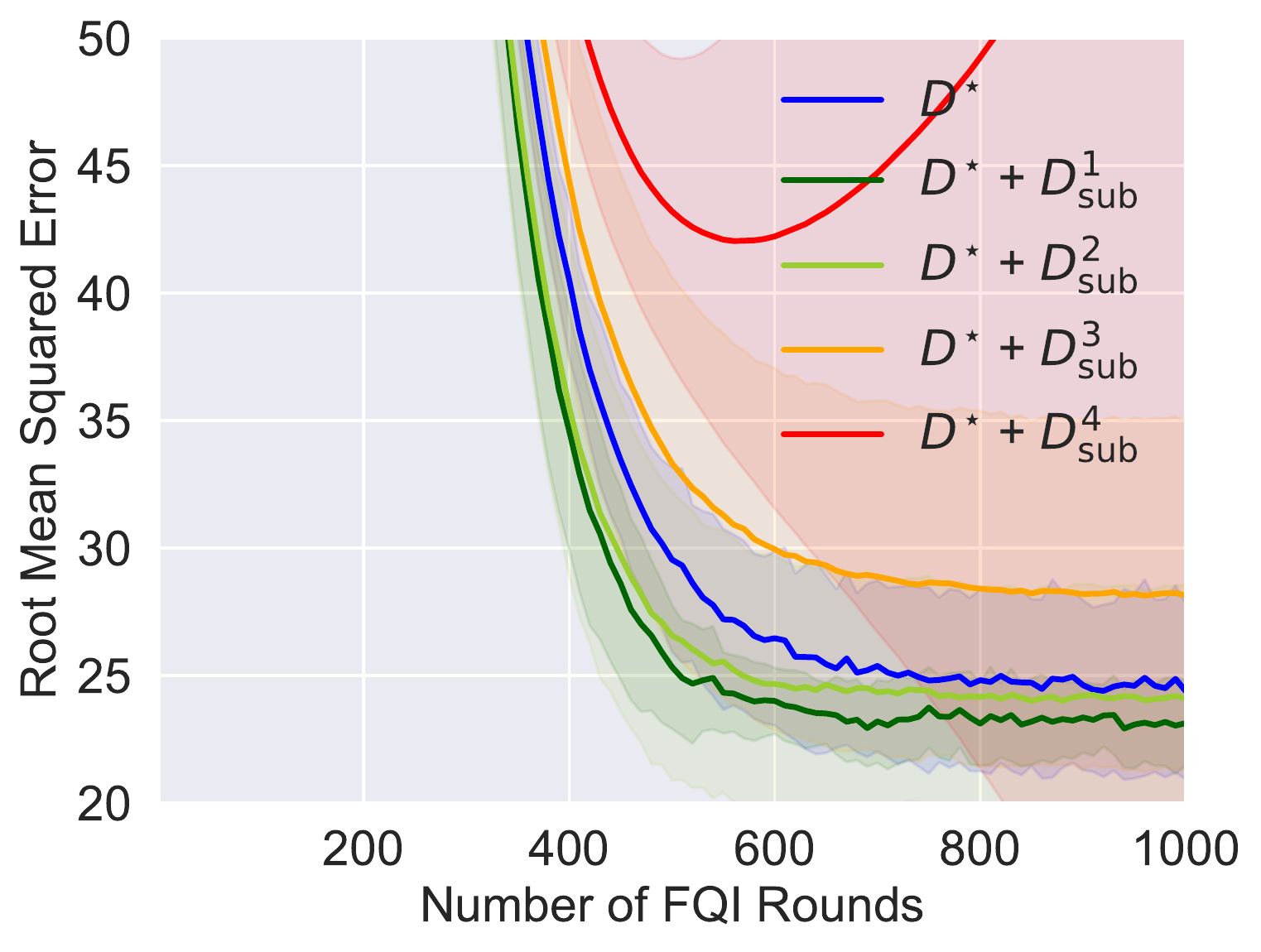}
}
\subfigure[MountainCar-v0]{
\includegraphics[width=0.30\textwidth]{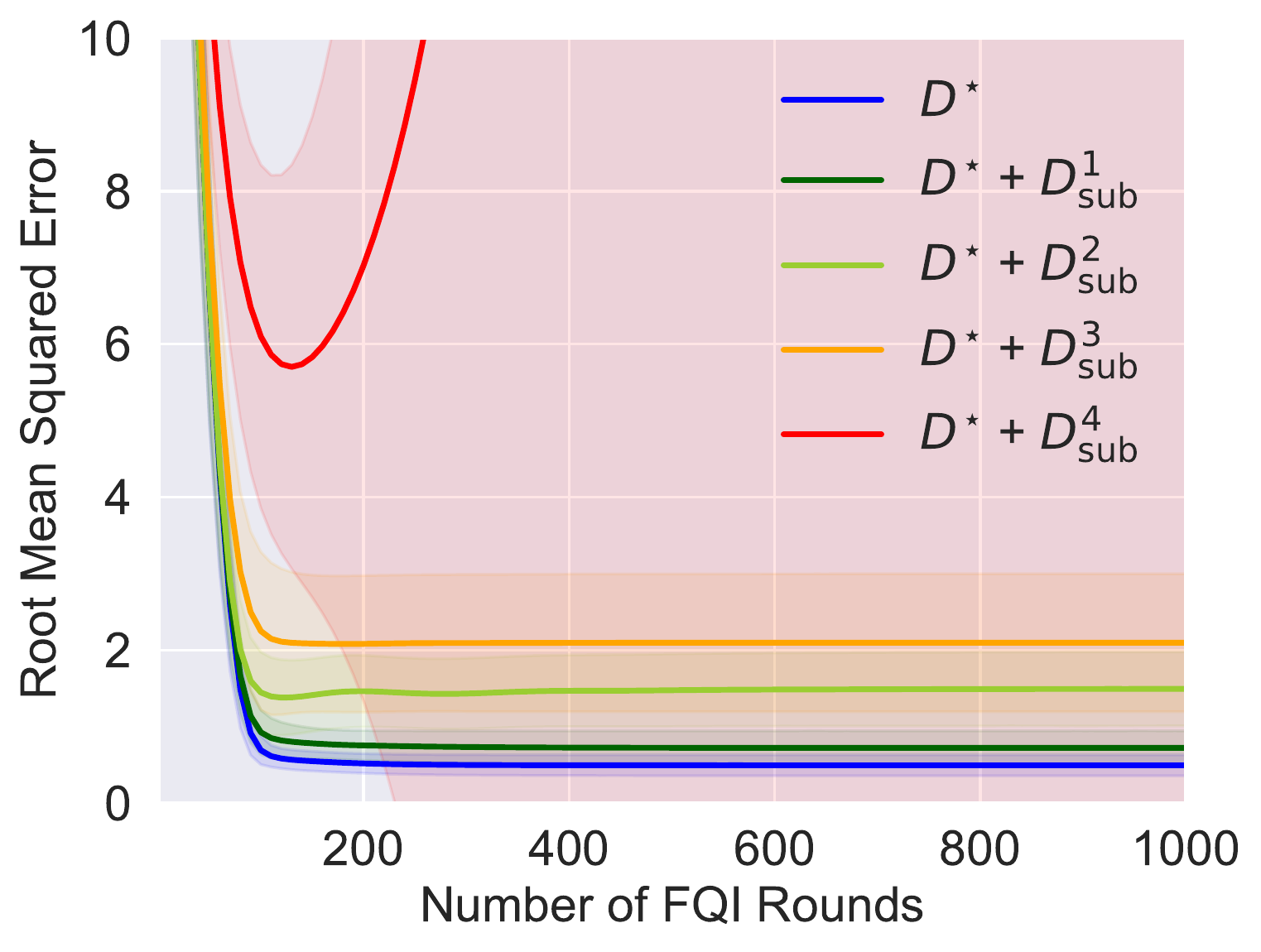}
}
\subfigure[Walker2d-v2]{
\includegraphics[width=0.30\textwidth]{fig/fig_Walker2d-v2_1}
}
\caption{Performance of FQI with features from pre-trained neural networks and datasets induced by lower performance policies.}
\label{fig:suboptimal_add}
\end{figure}

\begin{figure}[!h]
\thisfloatpagestyle{empty}
\centering
\subfigure[Ant-v2]{
\includegraphics[width=0.30\textwidth]{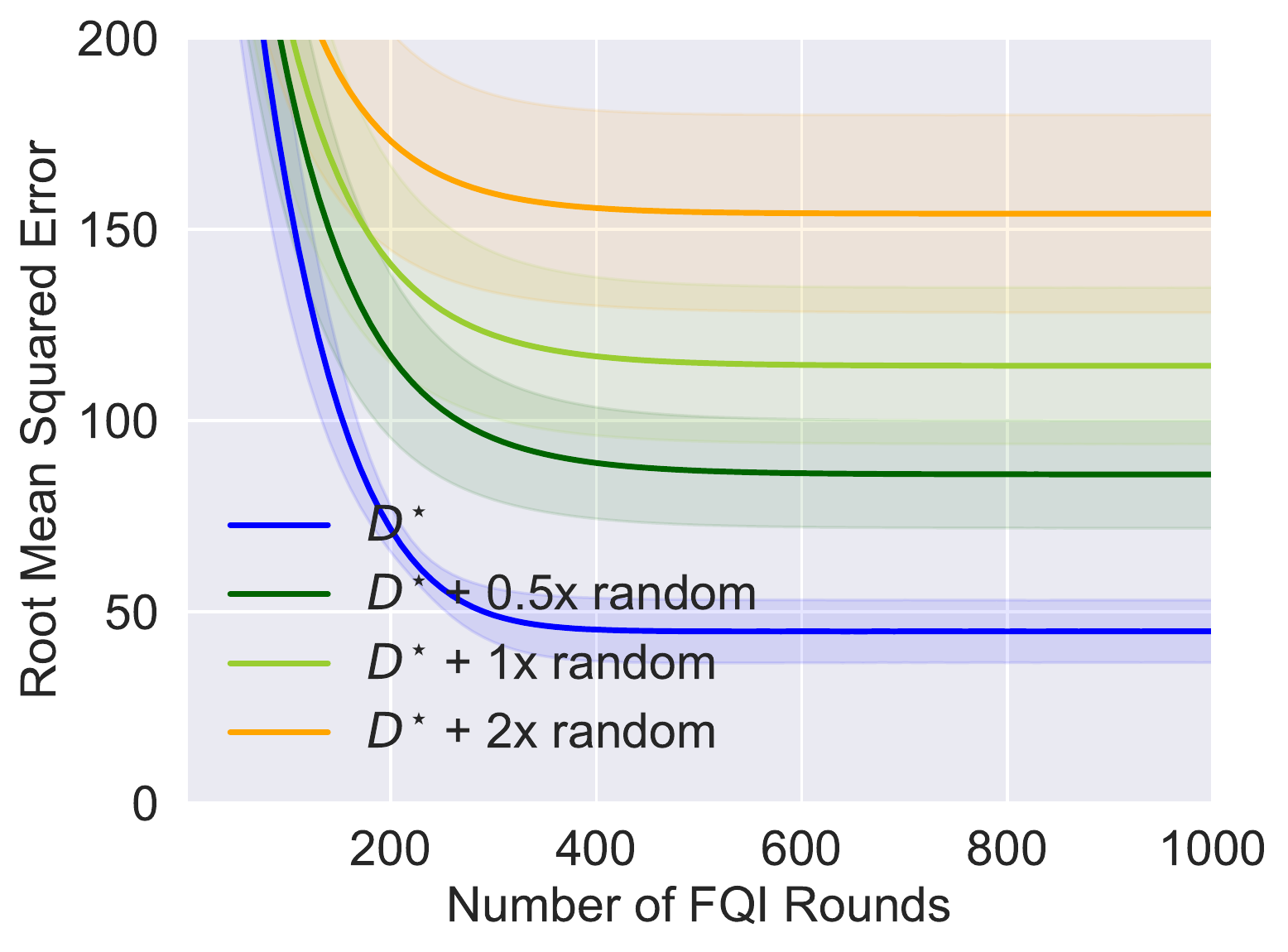}
}
\subfigure[CartPole-v0]{
\includegraphics[width=0.30\textwidth]{fig/fig_CartPole-v0_2}
}
\subfigure[Hopper-v2]{
\includegraphics[width=0.30\textwidth]{fig/fig_Hopper-v2_2}
}

\subfigure[HalfCheetah-v2]{
\includegraphics[width=0.30\textwidth]{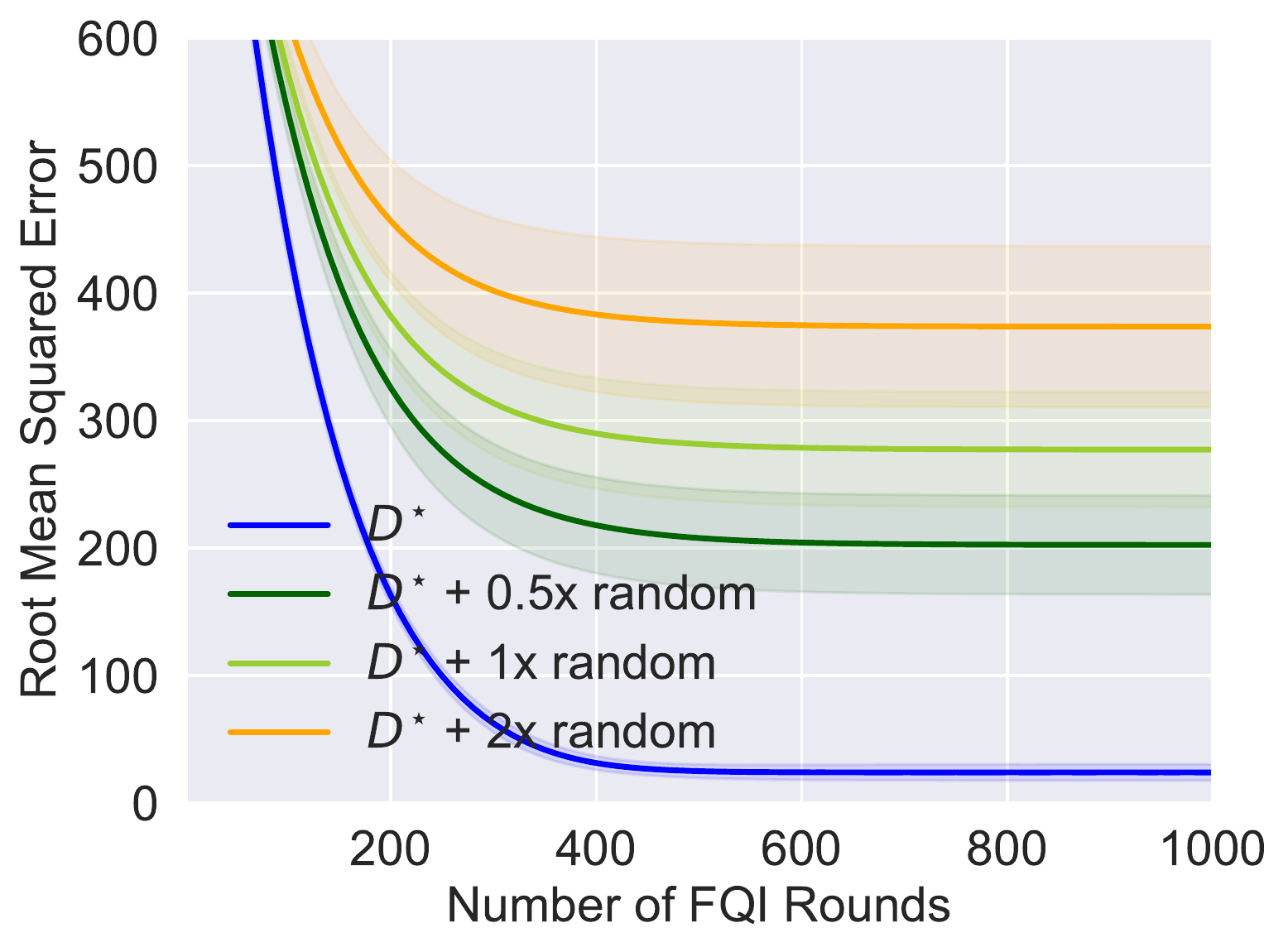}
}
\subfigure[MountainCar-v0]{
\includegraphics[width=0.30\textwidth]{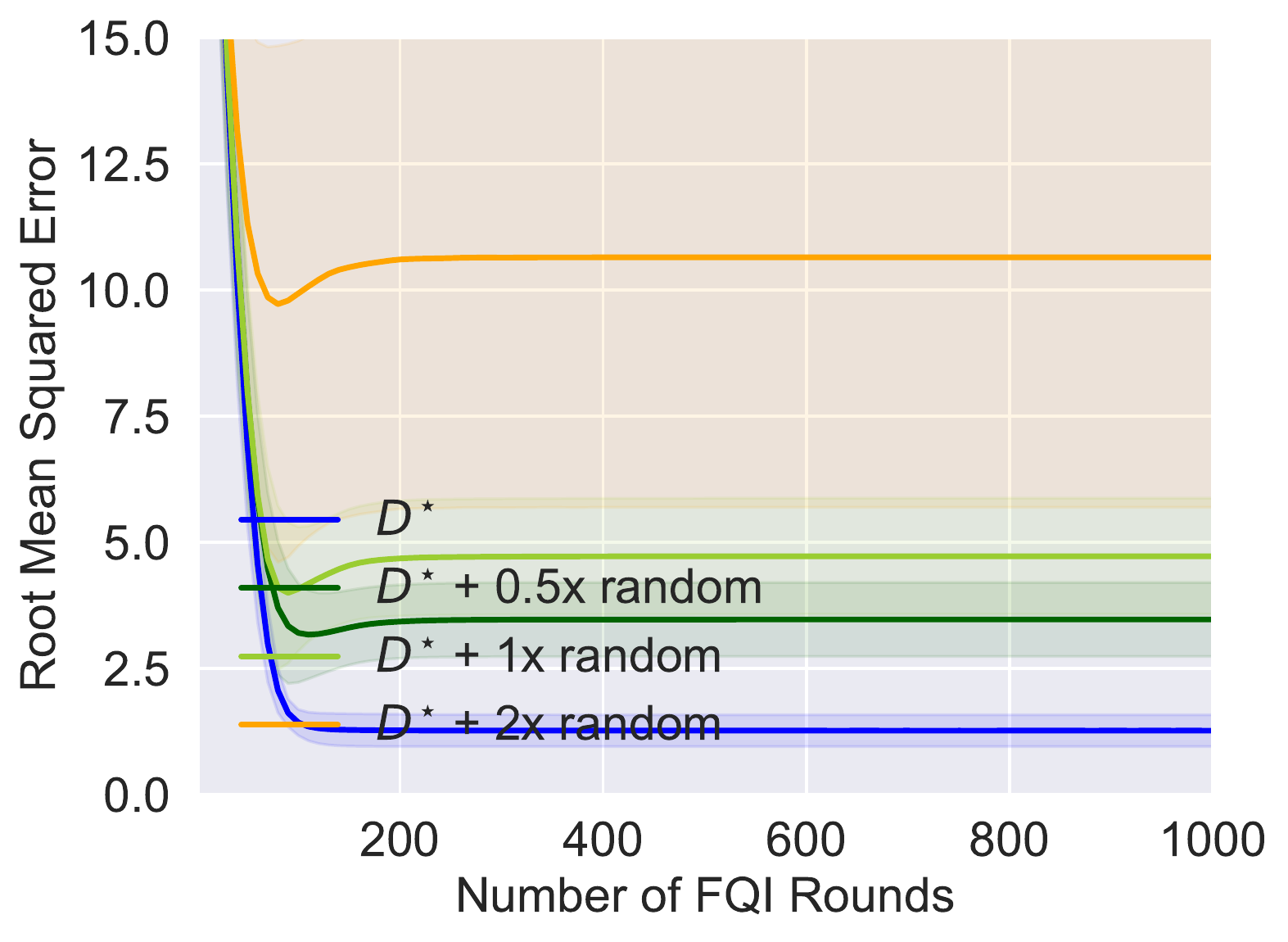}
}
\subfigure[Walker2d-v2]{
\includegraphics[width=0.30\textwidth]{fig/fig_Walker2d-v2_2}
}
\caption{Performance of FQI with random Fourier features and datasets induced by random policies.}
\label{fig:rff_add}

\thisfloatpagestyle{empty}
\subfigure[$D^{\star}$ ]{
\includegraphics[width=0.30\textwidth]{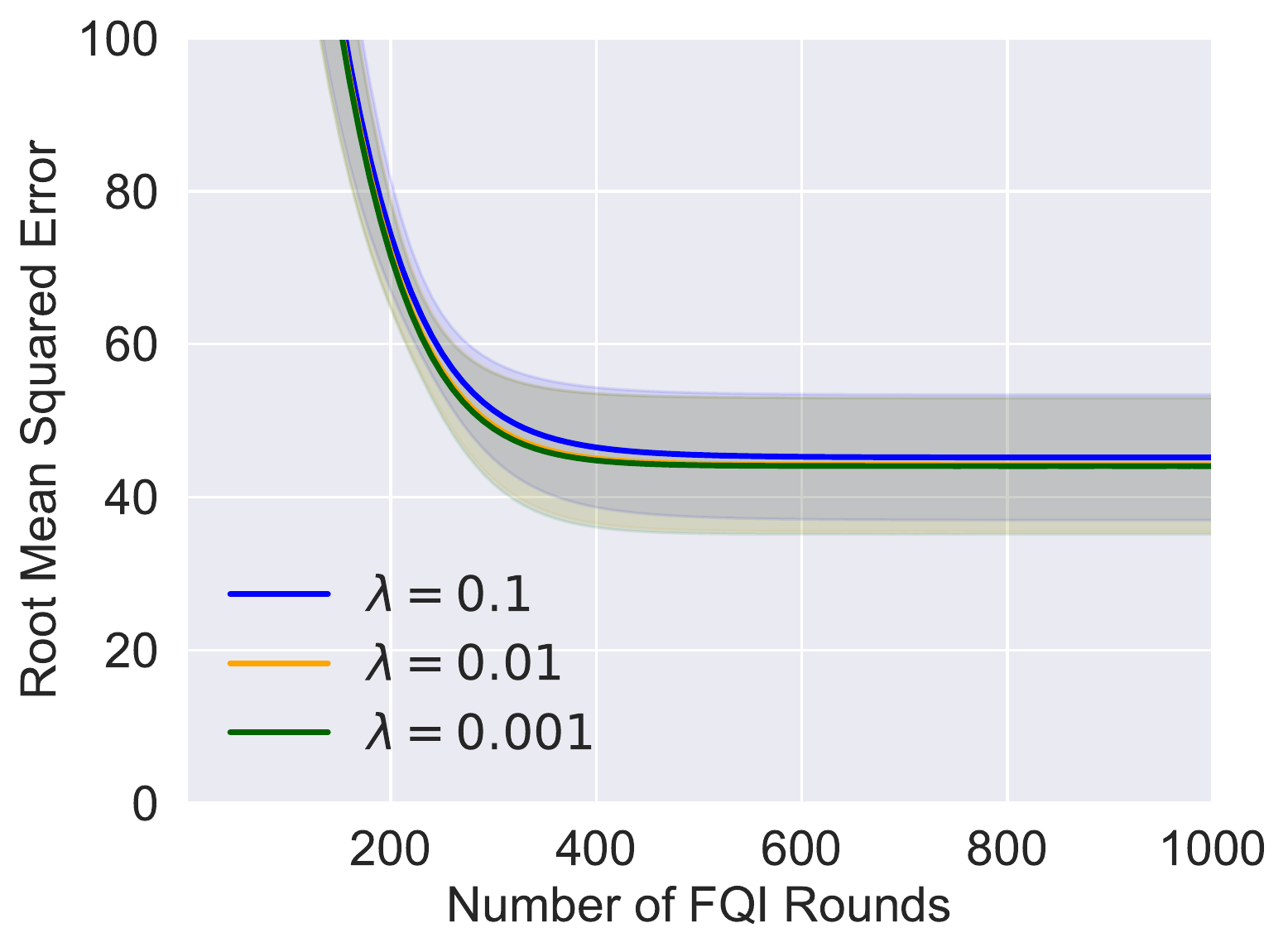}
}
\subfigure[$D^{\star}$ + 1x random]{
\includegraphics[width=0.30\textwidth]{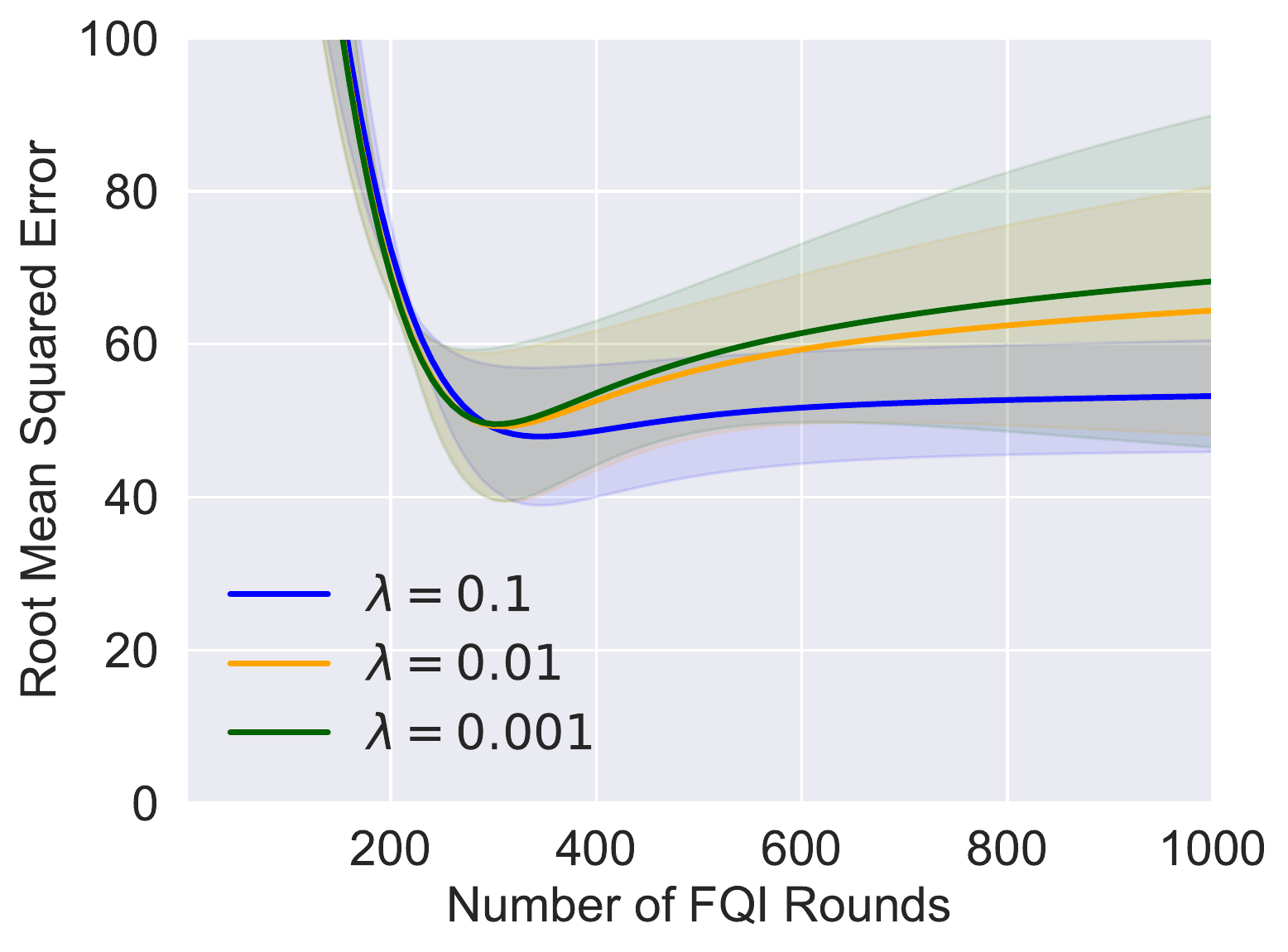}
}
\subfigure[$D^{\star}$+ 2x random]{
\includegraphics[width=0.30\textwidth]{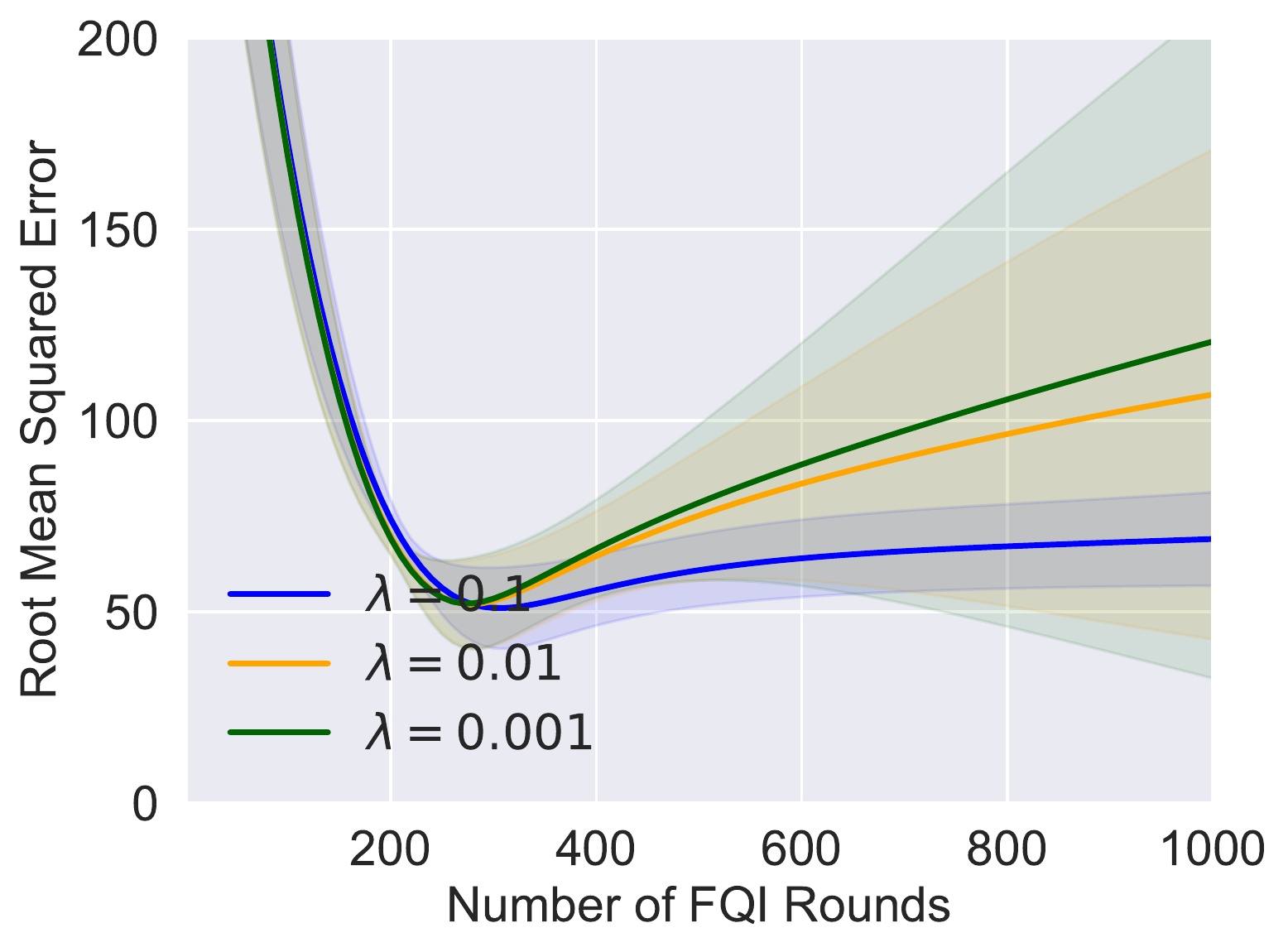}
}
\caption{Performance of FQI on Ant-v2, with features from pre-trained neural networks, datasets induced by random policies, and different regularization parameter $\lambda$.}
\label{fig:ridge_ant}

\subfigure[$D^{\star}$ ]{
\includegraphics[width=0.30\textwidth]{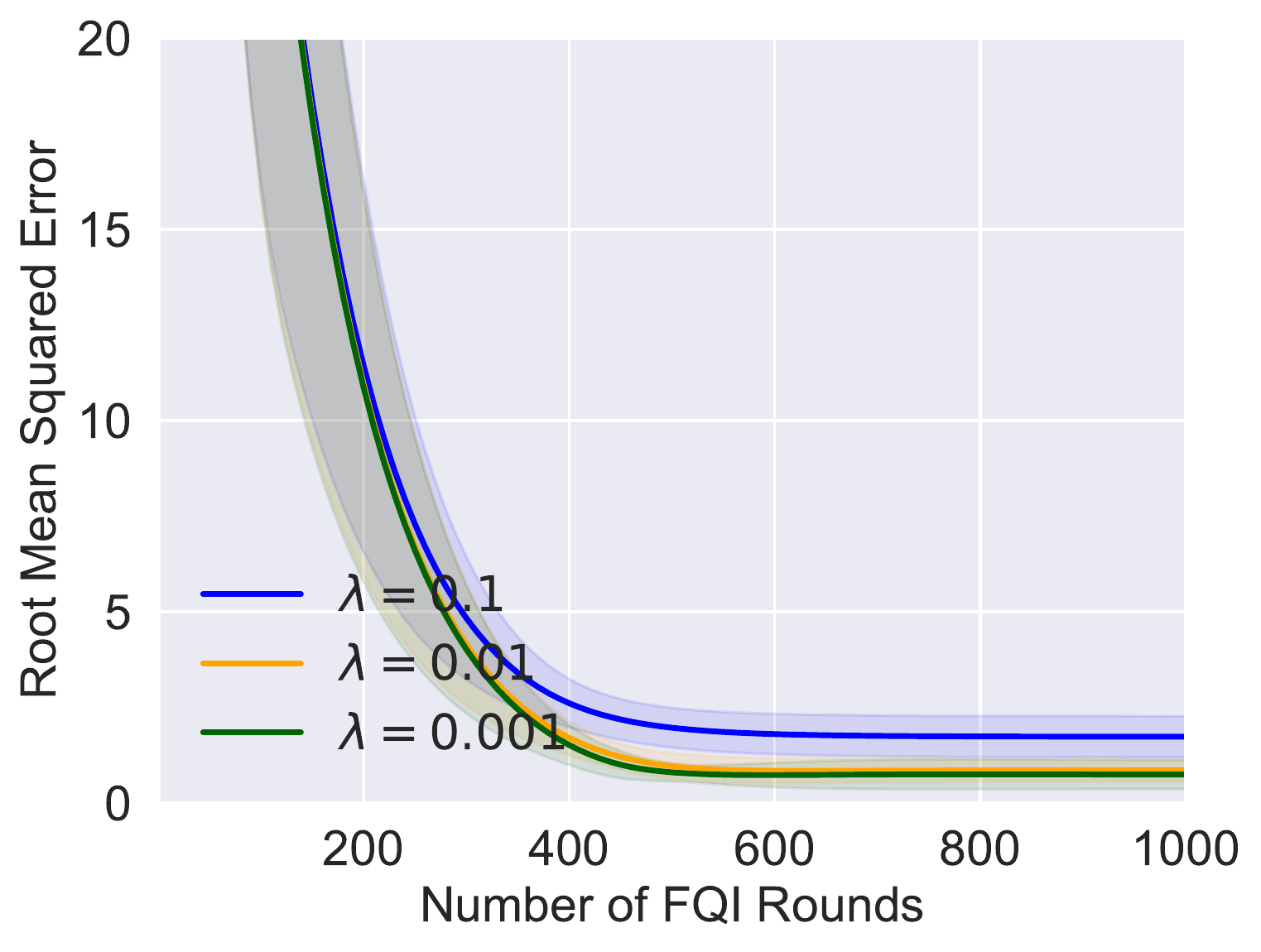}
}
\subfigure[$D^{\star}$ + 1x random]{
\includegraphics[width=0.30\textwidth]{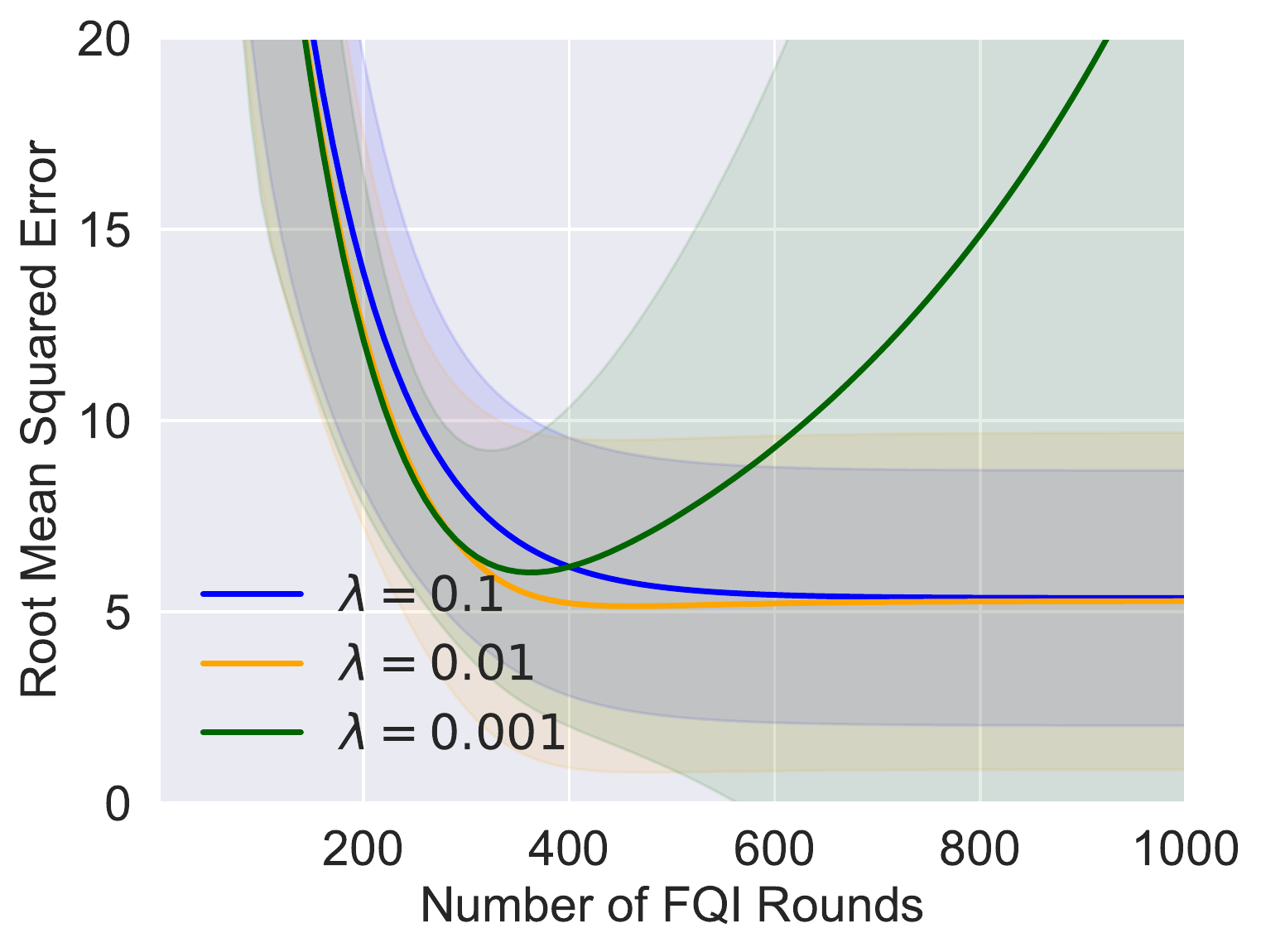}
}
\subfigure[$D^{\star}$+ 2x random]{
\includegraphics[width=0.30\textwidth]{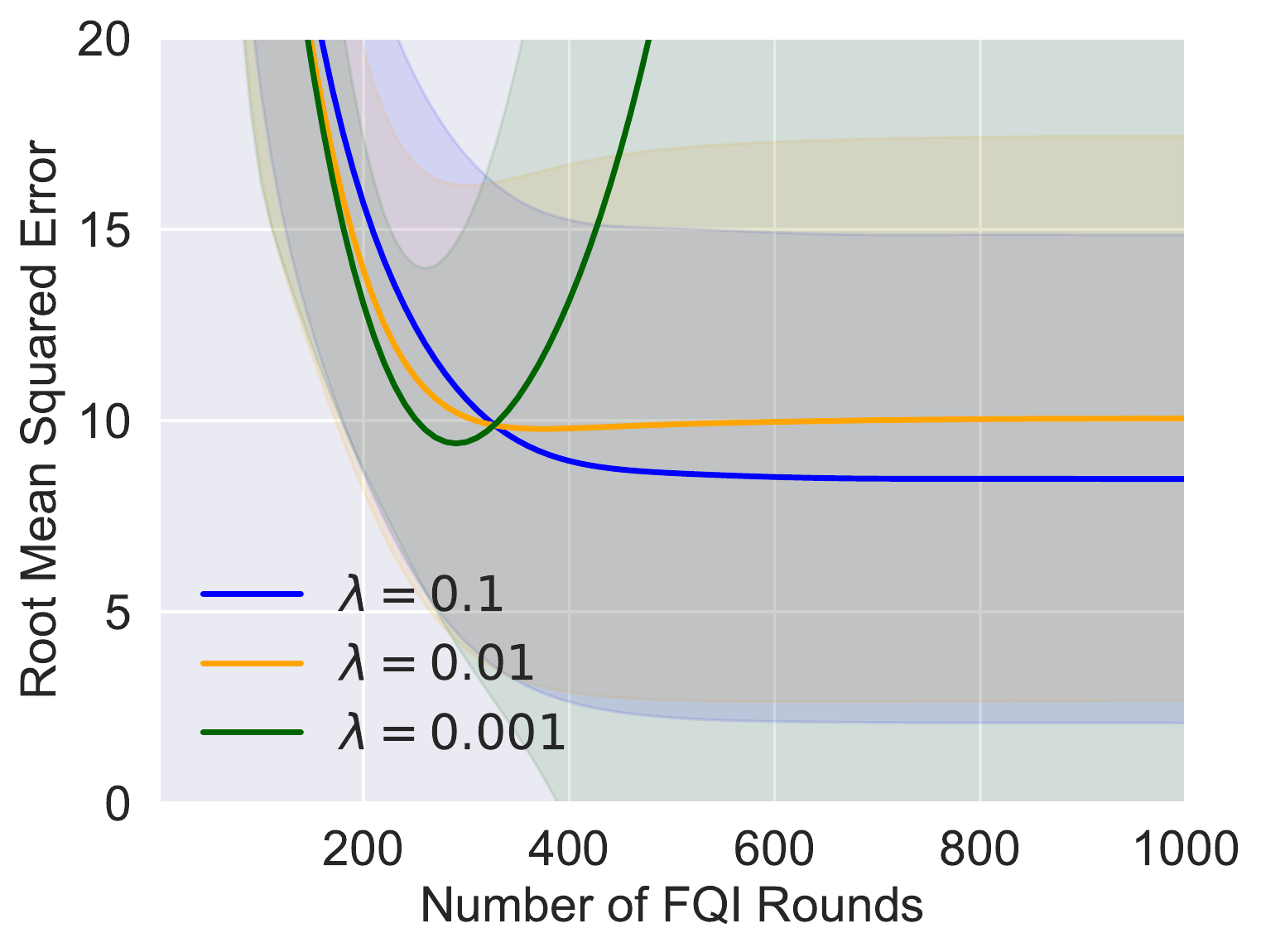}
}
\caption{Performance of FQI on CartPole-v0, with features from pre-trained neural networks, datasets induced by random policies, and different regularization parameter $\lambda$.}
\label{fig:ridge_cartpole}
\end{figure}
\begin{figure}

\thisfloatpagestyle{empty}
\subfigure[$D^{\star}$ ]{
\includegraphics[width=0.30\textwidth]{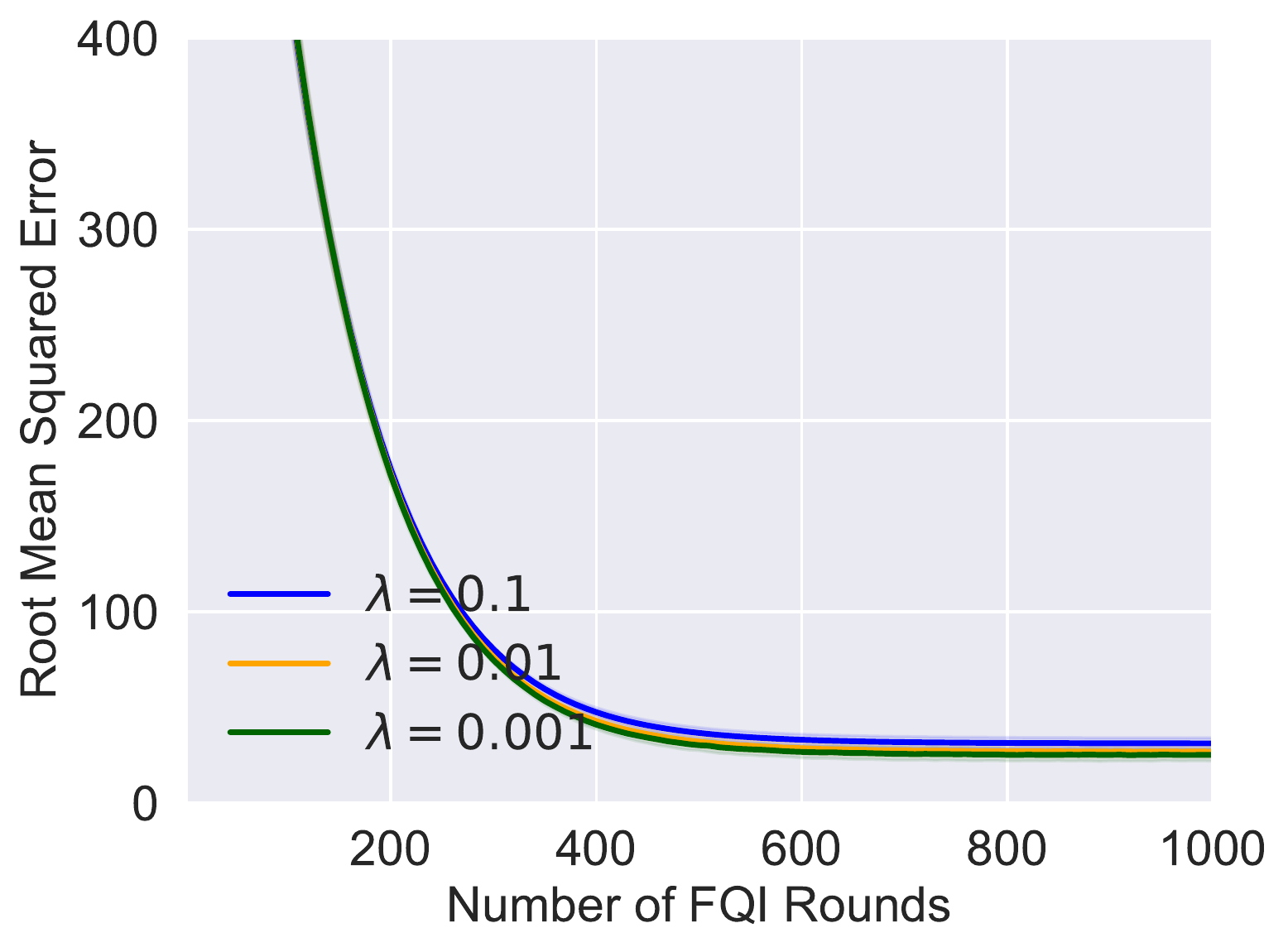}
}
\subfigure[$D^{\star}$ + 1x random]{
\includegraphics[width=0.30\textwidth]{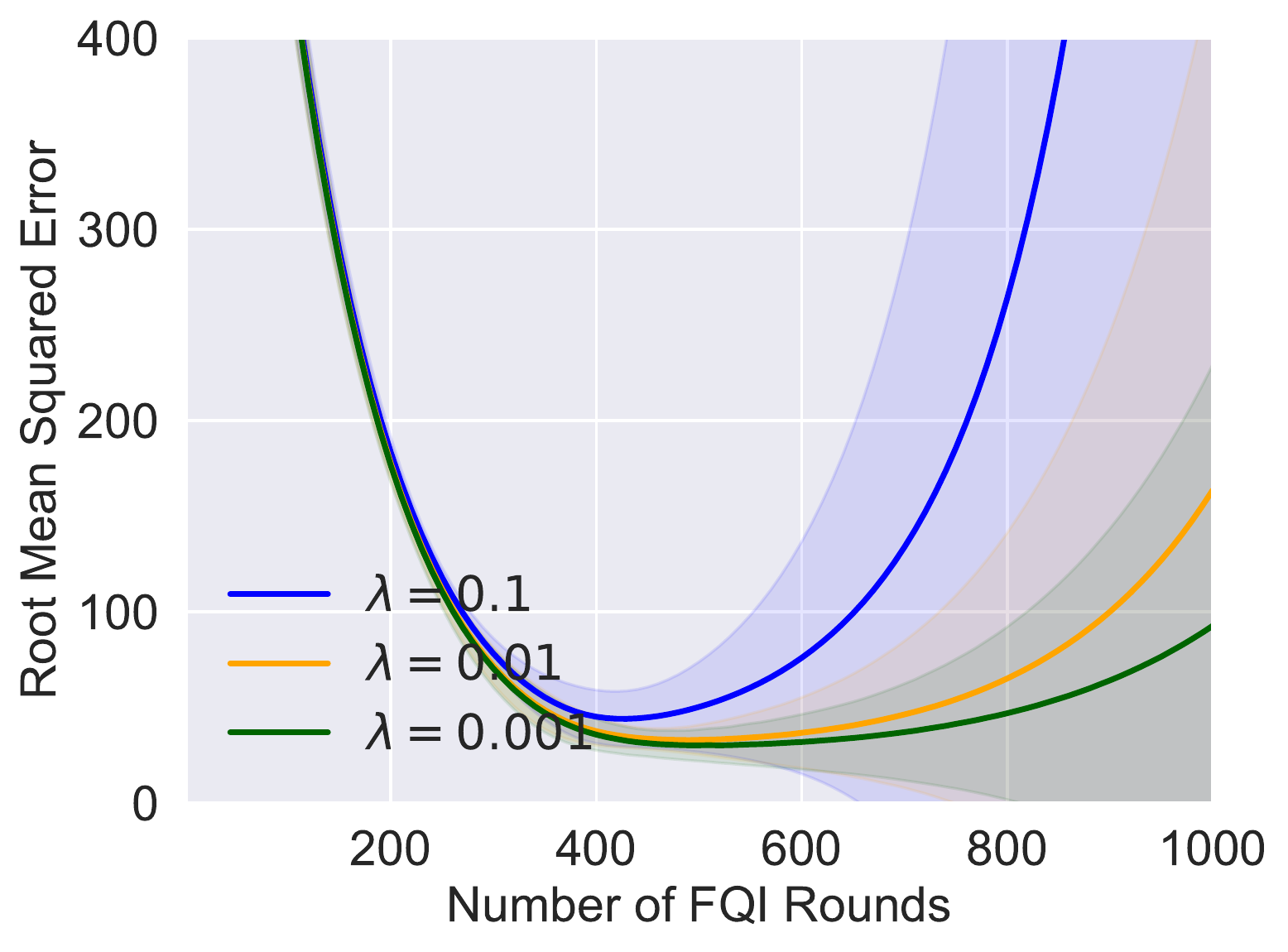}
}
\subfigure[$D^{\star}$+ 2x random]{
\includegraphics[width=0.30\textwidth]{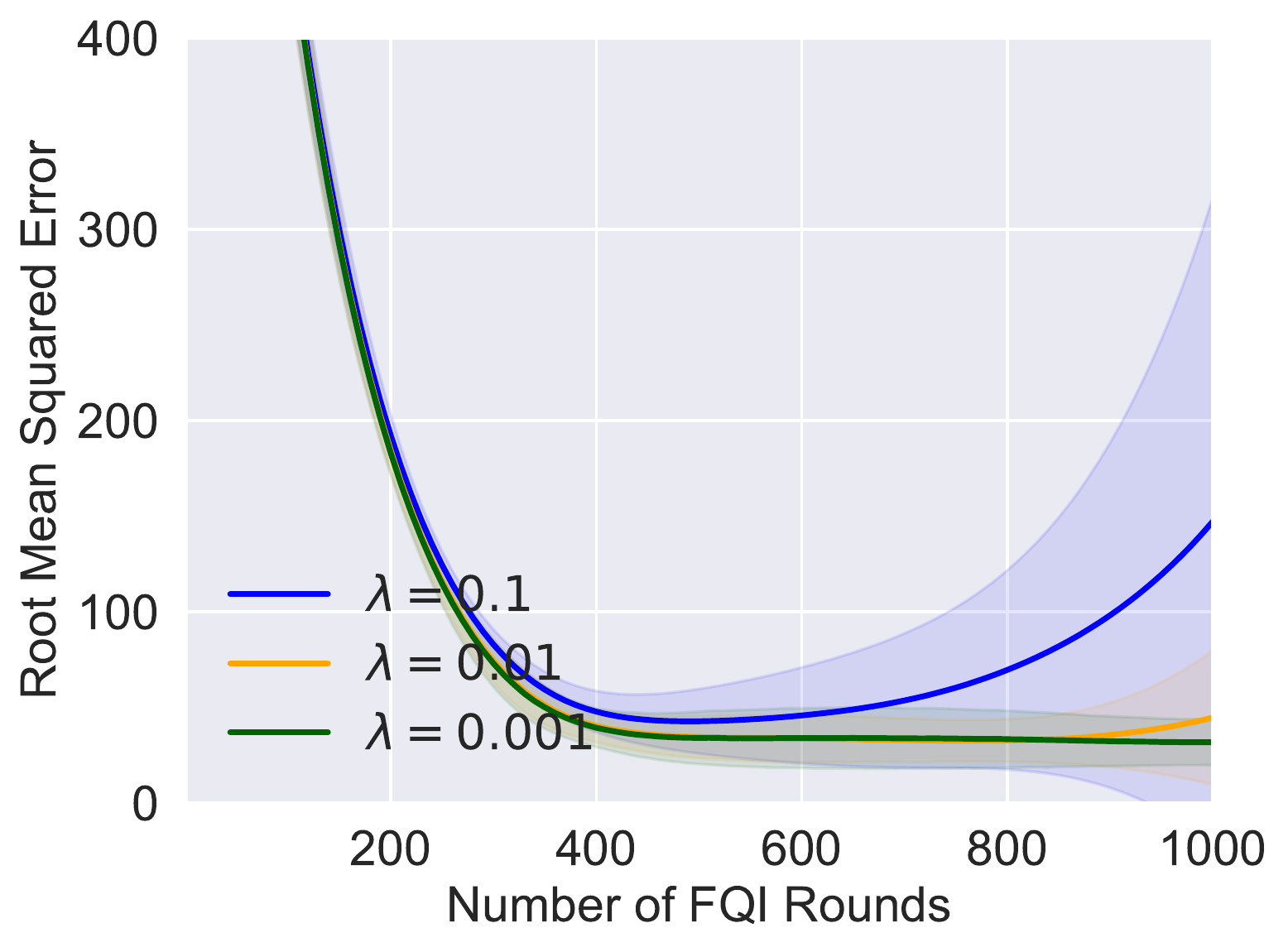}
}
\caption{Performance of FQI on HalfCheetah-v2, with features from pre-trained neural networks, datasets induced by random policies, and different regularization parameter $\lambda$.}
\label{fig:ridge_halfcheetah}

\subfigure[$D^{\star}$ ]{
\includegraphics[width=0.30\textwidth]{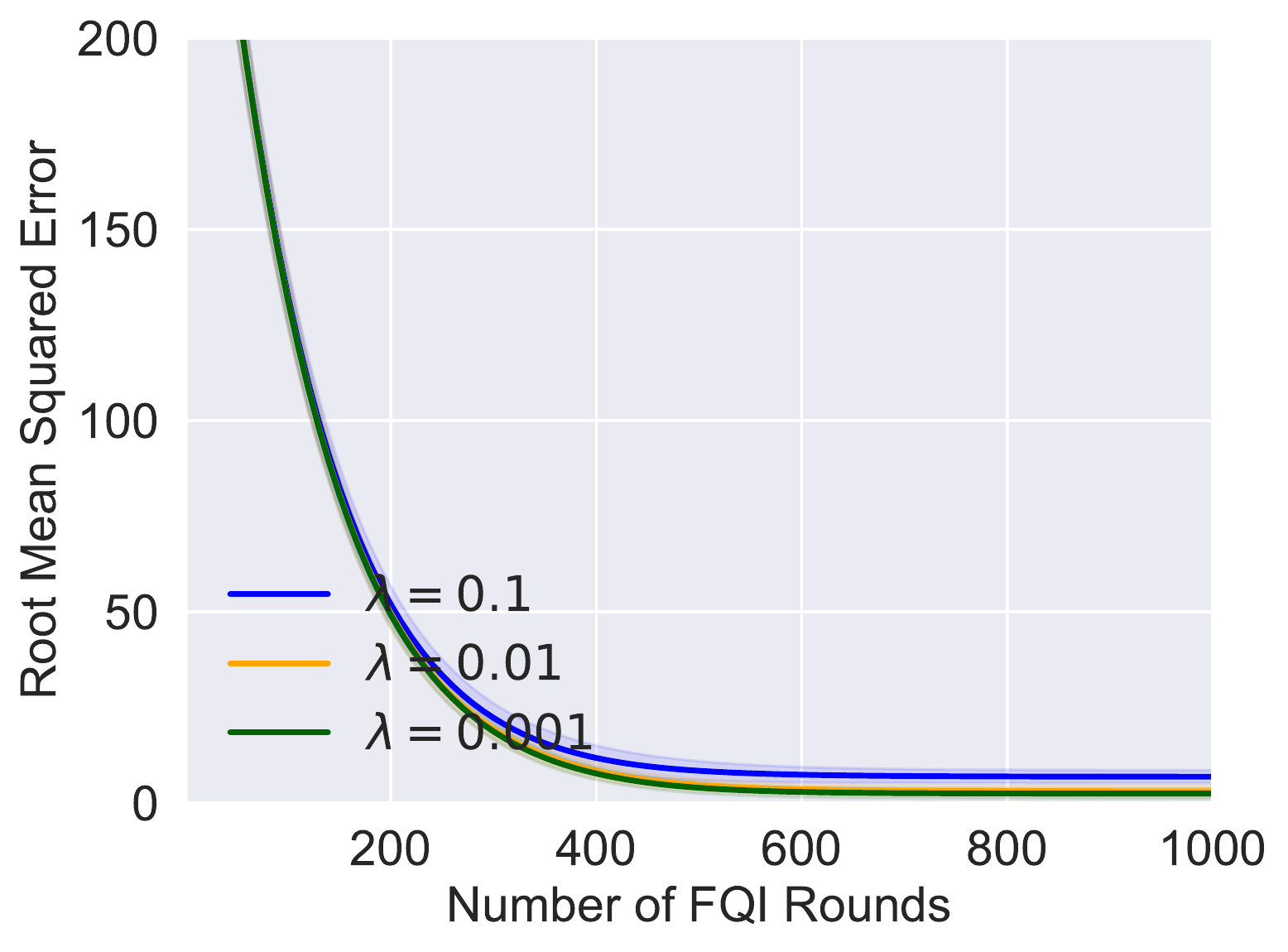}
}
\subfigure[$D^{\star}$ + 1x random]{
\includegraphics[width=0.30\textwidth]{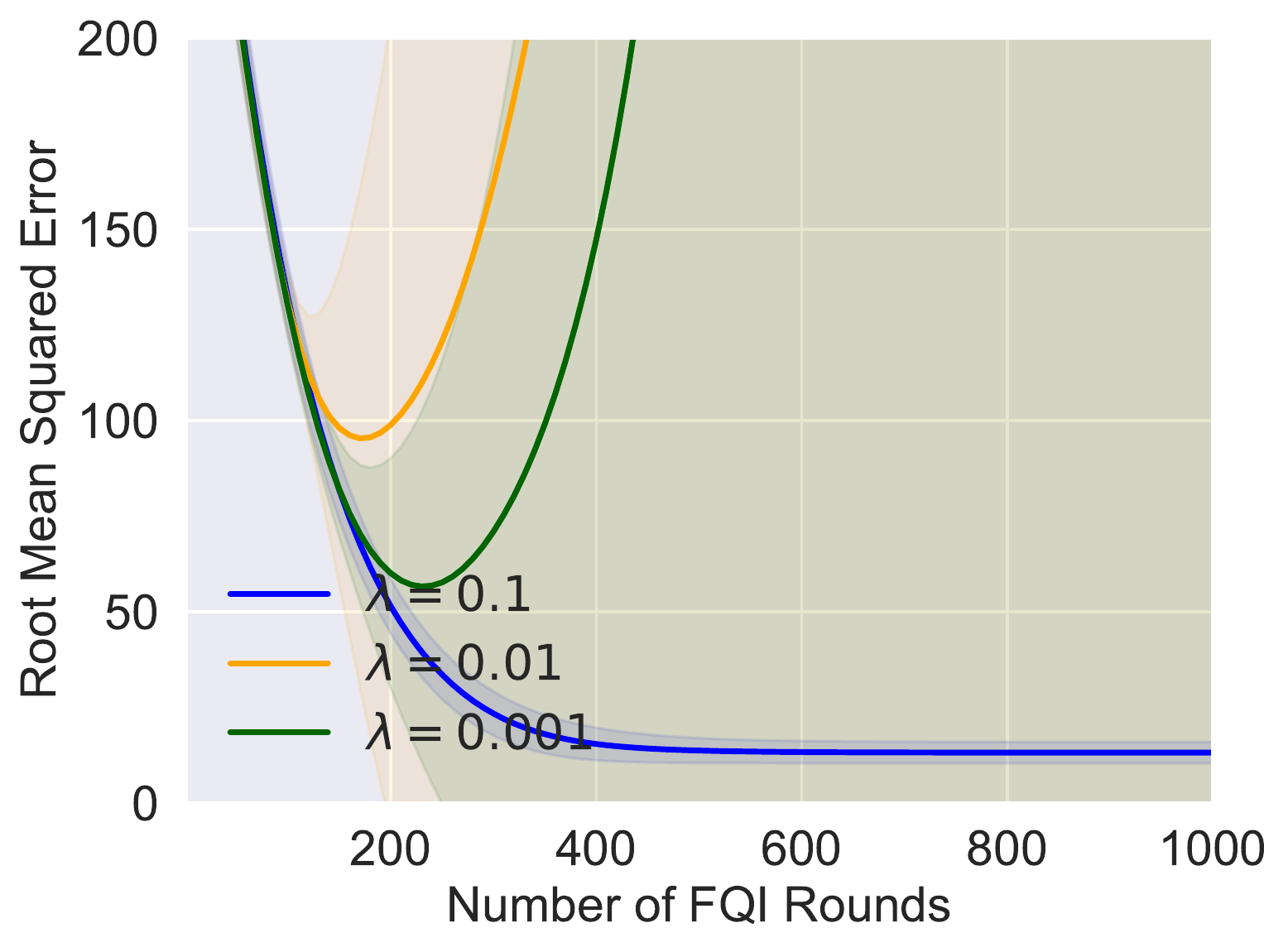}
}
\subfigure[$D^{\star}$+ 2x random]{
\includegraphics[width=0.30\textwidth]{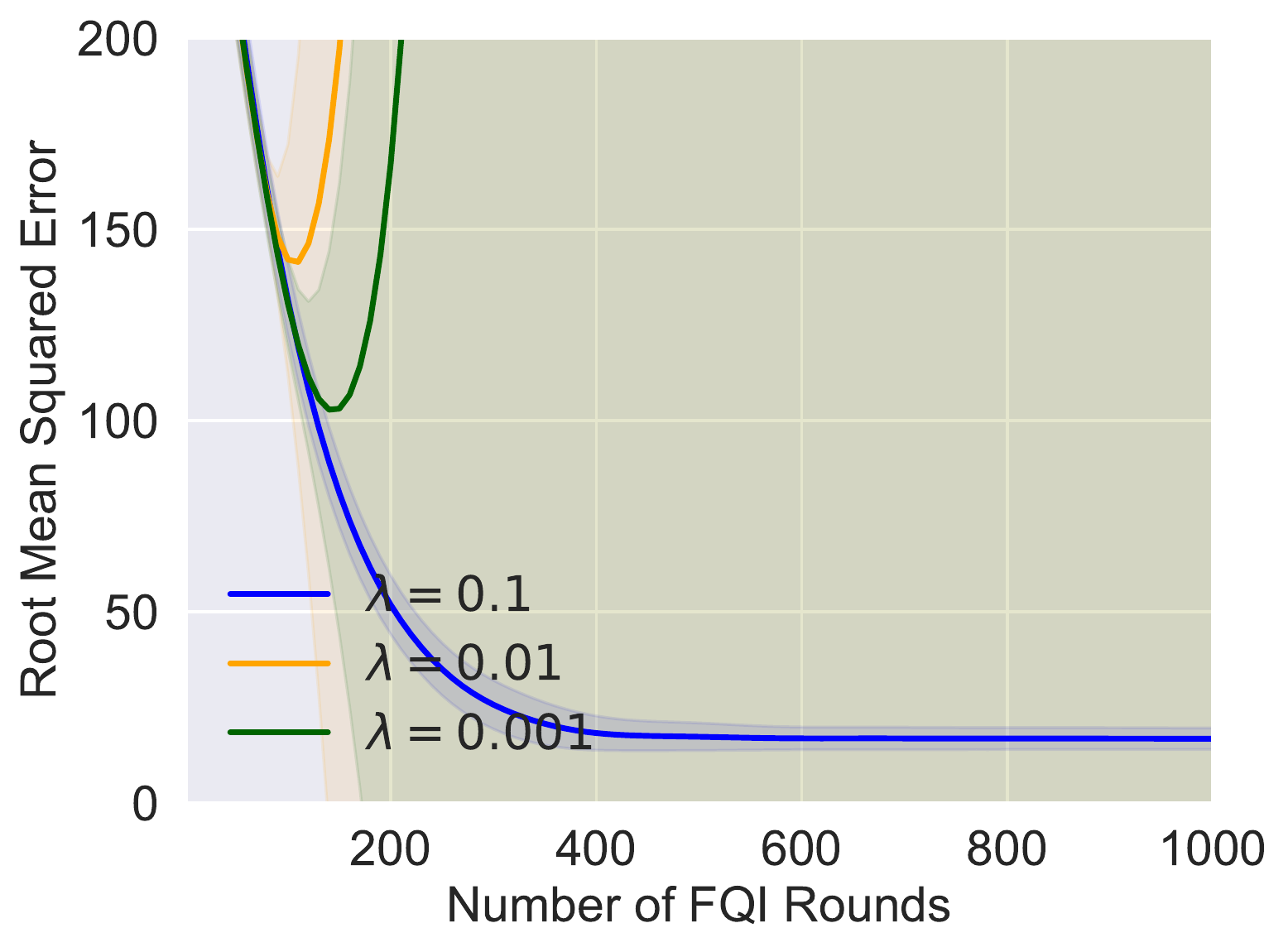}
}
\caption{Performance of FQI on Hopper-v2, with features from pre-trained neural networks, datasets induced by random policies, and different regularization parameter $\lambda$.}
\label{fig:ridge_hoppe}

\subfigure[$D^{\star}$ ]{
\includegraphics[width=0.30\textwidth]{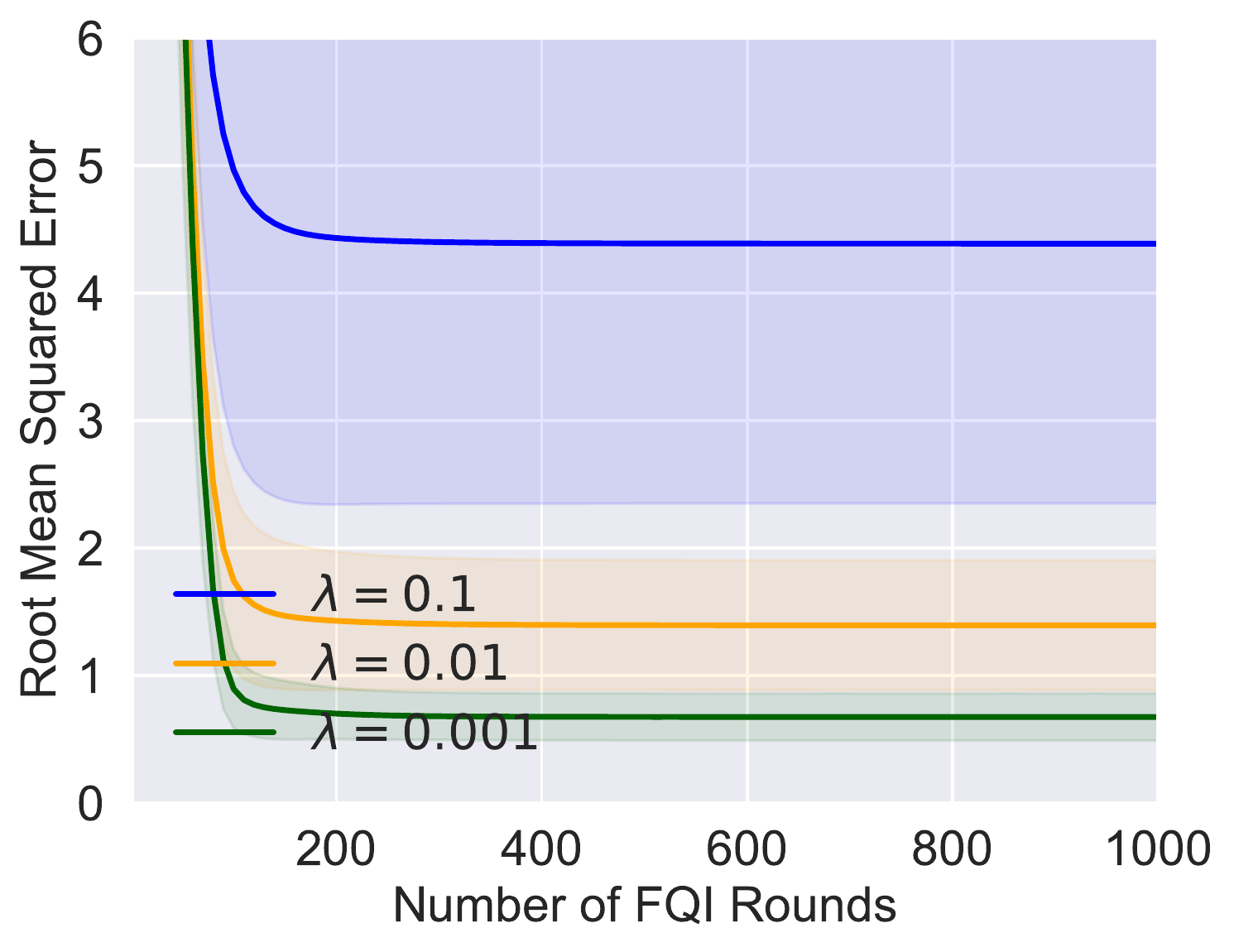}
}
\subfigure[$D^{\star}$ + 1x random]{
\includegraphics[width=0.30\textwidth]{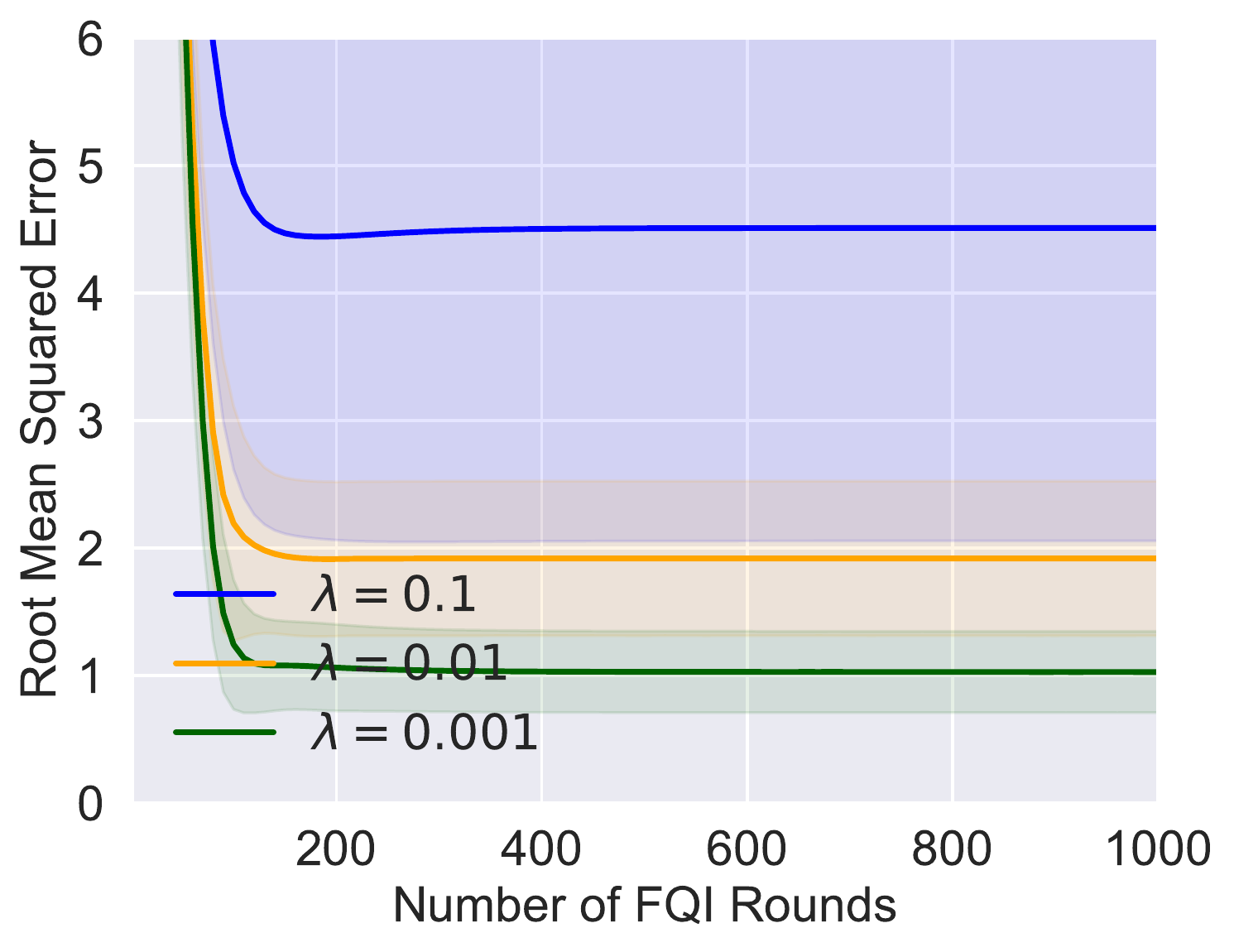}
}
\subfigure[$D^{\star}$+ 2x random]{
\includegraphics[width=0.30\textwidth]{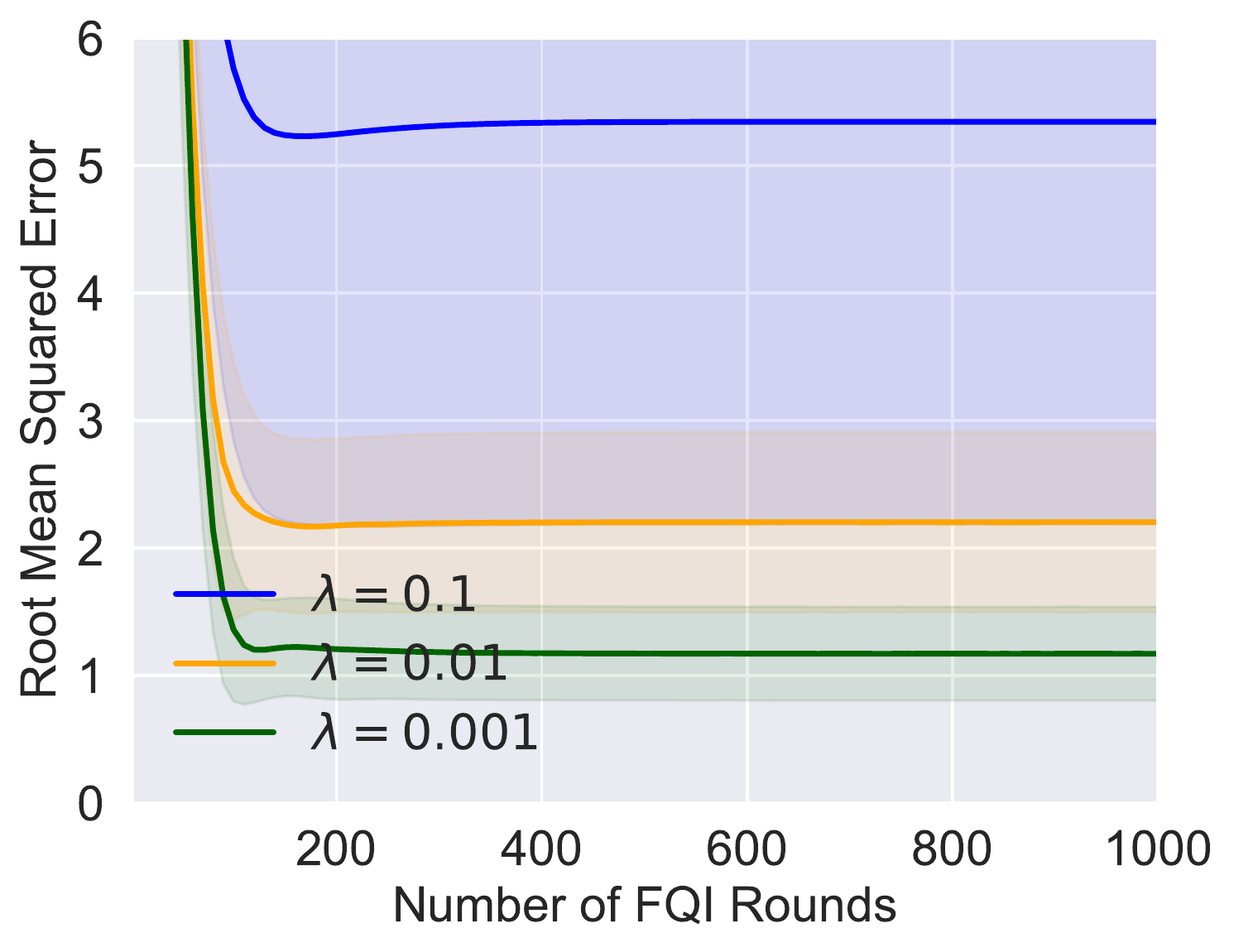}
}
\caption{Performance of FQI on MountainCar-v0, with features from pre-trained neural networks, datasets induced by random policies, and different regularization parameter $\lambda$.}
\label{fig:ridge_mountaincar}

\subfigure[$D^{\star}$ ]{
\includegraphics[width=0.30\textwidth]{fig/fig_Walker2d-v2_0_ridge}
}
\subfigure[$D^{\star}$ + 1x random]{
\includegraphics[width=0.30\textwidth]{fig/fig_Walker2d-v2_1_ridge}
}
\subfigure[$D^{\star}$+ 2x random]{
\includegraphics[width=0.30\textwidth]{fig/fig_Walker2d-v2_2_ridge}
}
\caption{Performance of FQI on Walker2d-v2, with features from pre-trained neural networks, datasets induced by random policies, and different regularization parameter $\lambda$.}
\label{fig:ridge_walker}
\end{figure}


\begin{table}
\centering
\begin{tabular}{lcccc}
\hline
Dataset & $D^{\star}$ & $D^{\star}$ + 0.5x random & $D^{\star}$ + 1x random & $D^{\star}$+ 2x random\\
\hline
Ant-v2 & $44.03 \pm 8.98$  &  $48.05 \pm 8.03$  &  $57.90 \pm 13.30$  &  $72.80 \pm 16.87$ \\

HalfCheetah-v2 & $24.86 \pm 3.39$  &  $27.54 \pm 6.63$  &  $30.14 \pm 11.60$  &  $36.66 \pm 21.32$ \\

Hopper-v2 & $2.18 \pm 1.14$  &  $9.38 \pm 3.84$  &  $13.18 \pm 2.77$  &  $16.86 \pm 2.84$ \\

Walker2d-v2 & $13.88 \pm 11.22$  &  $32.73 \pm 11.05$  &  $45.61 \pm 17.06$  &  $67.78 \pm 24.77$ \\
\hline
\end{tabular}
\caption{\emph{Performance of LSTD}. Performance of LSTD with features from pre-trained neural networks and distributions induced by random policies.
Each number of is the square root of the mean squared error of the
estimation, taking average over $5$ repetitions, $\pm$ standard
deviation. 
}
\label{table:lstd_add}
\end{table}

\end{document}